\documentclass[a4paper,11pt]{article}
\usepackage[top=4cm, bottom=4cm, left=3cm, right=3cm]{geometry}

\usepackage{microtype}
\usepackage{graphicx}
\usepackage{subfigure}
\usepackage{booktabs} 
\usepackage[sort&compress,square,comma,authoryear]{natbib}

\usepackage{hyperref}

\usepackage{url}            
\usepackage{amsfonts}       
\usepackage{nicefrac}       

\usepackage{amsmath, amssymb, graphicx, amsthm}
\usepackage{mathtools}
\usepackage{thm-restate}
\usepackage{algorithm}
\usepackage{algorithmic}
\usepackage{wrapfig}

\usepackage{setspace}
\usepackage{xspace}

\usepackage{wrapfig}
\usepackage{array} 
\usepackage{multirow}
\usepackage{enumitem}
\usepackage{mdframed}
\usepackage{relsize}
\usepackage{caption}

\usepackage{cleveref}
\creflabelformat{equation}{#2#1#3}
\usepackage{xfrac}

\usepackage[colorinlistoftodos, textwidth=20mm,disable]{todonotes}
\definecolor{blued}{rgb}{0.27, 0.77, 0.87}

\newcommand{\tododout}[1]{\todo[color=blued]{\tiny#1}}
\newcommand{\todoa}[1]{\todo[color=yellow, inline]{#1}}
\newcommand{\todoaout}[1]{\todo[color=yellow]{\tiny#1}}
\definecolor{citrine}{rgb}{0.89, 0.82, 0.04}
\newcommand{\todom}[1]{\todo[color=citrine]{\tiny#1}}

\newcommand{\citespecific}[2]{(\citeauthor{#2},~\citeyear{#2},~{#1})}

\renewcommand{\epsilon}{\varepsilon}

\newcommand*{\MyDef}{\mathrm{\tiny def}}
\newcommand*{\eqdefU}{\ensuremath{\mathop{\overset{\MyDef}{=}}}}
\newcommand*{\eqdef}{\eqdefU}
\renewcommand*{\triangleq}{\eqdefU}

\def\:#1{\protect \ifmmode {\mathbf{#1}} \else {\textbf{#1}} \fi}

\newcommand{\wt}[1]{\widetilde{#1}}

\newcommand{\wb}[1]{\overline{#1}}

\newcommand{\bk}{\mathbf{k}}

\newcommand{\bu}{\mathbf{u}}

\newcommand{\bx}{\mathbf{x}}
\newcommand{\by}{\mathbf{y}}
\newcommand{\bz}{\mathbf{z}}

\newcommand{\bA}{\mathbf{A}}
\newcommand{\bB}{\mathbf{B}}

\newcommand{\bI}{\mathbf{I}}

\newcommand{\bK}{\mathbf{K}}

\newcommand{\bP}{\mathbf{P}}

\newcommand{\bV}{\mathbf{V}}

\newcommand{\bX}{\mathbf{X}}

\newcommand{\bZ}{\mathbf{Z}}

\newcommand{\wrt}{w.r.t.\xspace}

\DeclareMathOperator*{\argmax}{arg\,max}
\DeclareMathOperator*{\logdet}{log\,det}

\renewcommand{\epsilon}{\varepsilon}
\newcommand{\bigotime}{\mathcal{O}}
\newcommand{\abigotime}{\wt{\mathcal{O}}}

\DeclareMathOperator*{\polylog}{polylog}

\newcommand{\normsmall}[1]{\Vert #1 \Vert}

\newcommand{\transp}{\mathsf{\scriptscriptstyle T}}

\DeclareMathOperator*{\Tr}{Tr}

\DeclareMathOperator*{\Ran}{Im}

\newcommand{\condbar}{\;\middle|\;}

\newcommand{\indfunc}{\mathbb{I}}
\newcommand{\Real}{\mathbb{R}}

\newcommand{\rkhs}{\mathcal{H}}

\newcommand{\kerfunc}{k}
\newcommand{\kermatrix}{{\bK}}

\newcommand{\deff}{d_{\text{eff}}}

\newcommand{\ie}{i.e.,\xspace}
\newcommand{\eg}{e.g.,\xspace}
\newcommand{\whp}{w.h.p.\xspace}

\newcommand{\amu}{\wt{\mu}}
\newcommand{\asigma}{\wt{\sigma}}

\newcommand{\abA}{\wt{\bA}}

\renewcommand{\Re}{\mathbb{R}}

\newcommand{\phimat}[1]{\boldsymbol{\Phi}(\bX_{#1})}
\newcommand{\aphimat}[1]{\wt{\boldsymbol{\Phi}}_{#1}(\bX_{#1})}
\newcommand{\phivec}[1]{\featmap(\bx_{#1})}
\newcommand{\aphivec}[1]{\wt{\boldsymbol{\phi}}(\bx_{#1})}

\newcommand{\nystrom}{Nystr\"{o}m\xspace}

\newcommand{\coldict}{\mathcal{S}}

\newcommand{\armset}{\mathcal{A}}
\newcommand{\featmap}{\boldsymbol{\phi}}
\newcommand{\embfunc}{\bz}
\newcommand{\embvec}{\bz}
\newcommand{\embmat}{\bZ}

\newcommand{\fnorm}{F}
\newcommand{\Narm}{A}

\newcommand{\qbar}{\wb{q}}

\newcommand{\fb}[1]{\normalfont \texttt{fb}(#1)}

\newcommand{\bkb}{\textsc{BKB}\xspace}
\newcommand{\bbkb}{\textsc{BBKB}\xspace}

\newcommand{\gpucb}{\textsc{GP-UCB}\xspace}
\newcommand{\bgpucb}{\textsc{GP-BUCB}\xspace}

\newcommand{\bts}{\textsc{async-TS}\xspace}
\newcommand{\regev}{\textsc{Reg-evolution}\xspace}

\newtheorem{theorem}{Theorem}

\newtheorem{lemma}{Lemma}
\newtheorem{proposition}{Proposition}

\newcommand{\Tinit}{T_{\text{init}}}
\newcommand{\numbatches}{B}
\newcommand{\minbatchsize}{P}

\newcommand{\epsfrac}{\left(\frac{1 + \varepsilon}{1 - \varepsilon}\right)}

\title{Near-linear Time Gaussian Process Optimization with\\*
 Adaptive Batching and Resparsification}

\date{}

\author{%
    \hspace{-4em}Daniele  Calandriello$^{*,1}$~ Luigi Carratino$^{*,2}$~ Alessandro Lazaric$^{3}$~ Michal Valko$^{4}$~ Lorenzo Rosasco$^{1,2,5}$\\
     \hspace{-4em}{\footnotesize $^*$ Equal contribution ~~ $^1$  Istituto Italiano di Tecnologia, Genova, Italy ~~ $^2$ University of Genova, Genova, Italy}\\
   \hspace{-4em} {\footnotesize $^3$ Facebook AI Research, Paris, France ~~ $^4$ Deepmind, Paris, France ~~ $^5$ Massachusetts Institute of Technology, MA, USA}\\
    \hspace{-4em}{\footnotesize Correspondence to: \texttt{daniele.calandriello@iit.it, luigi.carratino@dibris.unige.it}}
}

\begin{document}
\maketitle
\begin{abstract}
Gaussian processes (GP) are one of the most successful frameworks to model uncertainty. However, GP optimization (e.g., GP-UCB) suffers from major scalability issues. Experimental time grows linearly with the number of evaluations, unless candidates are selected in batches (e.g., using GP-BUCB) and evaluated in parallel. Furthermore, computational cost is often prohibitive since algorithms such as GP-BUCB require a time at least quadratic in the number of dimensions and iterations to select each batch.

In this paper, we introduce BBKB (Batch Budgeted Kernel Bandits), the first no-regret GP optimization algorithm that provably runs in near-linear time and selects candidates in batches. This is obtained with a new guarantee for the tracking of the posterior variances that allows BBKB to choose increasingly larger batches, improving over GP-BUCB. Moreover, we show that the same bound can be used to adaptively delay costly updates to the sparse GP approximation used by BBKB, achieving a near-constant per-step amortized cost. These findings are then confirmed in several experiments, where BBKB is much faster than state-of-the-art methods.
\end{abstract}

\section{Introduction}

Gaussian process (GP) optimization is a principled way to optimize a black-box function from noisy evaluations (i.e., sometimes referred to as \textit{bandit} feedback). Due to the presence of noise, the optimization process is modeled as a sequential \textit{learning} problem, where at each step $t$
\begin{enumerate}
\item the learner chooses candidate $\bx_t$ out of a decision set $\armset$;
\item the environment evaluates $f(\bx_t)$ and returns a noisy feedback $y_t$ to the learner;
\item the learner uses $y_t$ to guide its subsequent choices
\end{enumerate}

The goal of the learner is to converge over time to a global optimal candidate. This goal is often formalized as a \textit{regret minimization} problem, where the performance of the learner is evaluated by the cumulative value of the candidates chosen over time (i.e., $\sum_t f(\bx_t)$) compared to the optimum of the function $f^\star = \max_{\bx} f(\bx)$.
While many GP optimization algorithms come with strong theoretical guarantees and are empirically effective, most of them suffer from experimental and/or computational scalability issues.

\textbf{Experimental scalability.}
GP optimization algorithms usually follow a sequential interaction protocol, where at each step $t$, they wait for the feedback $y_t$ before proposing a new candidate $\bx_{t+1}$. As such, the \textit{experimentation time} grows linearly with $t$, which may be impractical in applications where each evaluation may take long time to complete (e.g., in lab experiments). This problem can be mitigated by switching to \emph{batched} algorithms, which at step $t$ propose a batch of candidates that are evaluated in parallel. After the batch is evaluated, the algorithm integrates the feedbacks and move on to select the next batch. This strategy reduces the experimentation time, but it may degrade the optimization performance, since candidates in the batch are with much less feedback.
Many approaches have been proposed for batched GP optimization.
Among those with theoretical guarantees, some are 
based on sequential greedy selection \cite{desautels_parallelizing_2014},
entropy search \cite{hennig_entropy_2012},
determinantal point process sampling \cite{kathuria_batched_2016}, or
multi-agent cooperation \cite{daxberger17a}, as well as many
heuristics for which no regret guarantees~\cite{chevalier2013fast,shah2015parallel}.
However, they all suffer from the same computational limitations of classical GP methods. 

\textbf{Computational scalability.}
The computational complexity of choosing a single candidate in classical GP methods grows quadratically with the number of evaluations. This makes it impractical to optimize complex functions, which require many steps before converging. Many approaches exist to improve scalability.
Some have been proposed in the context of sequential GP optimization,
such as those based on inducing points and sparse GP approximation
\cite{quinonero-candela_approximation_2007, calandriello_2019_coltgpucb}, variational inference
\cite{huggins2018scalable}, random fourier features \cite{mutny_efficient_2018},
and grid based methods \cite{wilson2015kernel}.
While some of these methods come with regret and computational guarantees,
they rely on a strict sequential protocol,
and therefore they are subject to the experimental bottleneck.
Other scalable approximations
are specific to batched methods, such as Markov approximation \cite{daxberger17a}
and Gaussian approximation \cite{shah2015parallel}.
However, these methods fail to guarantee either low regret or scalability.

\textbf{State of the art.} \todoaout{General comment: we should be sure that when talking about complexity, we make it clear whether its global or per-step.}
\gpucb~\citet{srinivas2010gaussian} is the most popular algorithm for GP optimization and it suffers a regret $\bigotime(\sqrt{T}\gamma_T)$ regret, where
$\gamma_T$ is the maximal mutual information gain of a GP after $T$ evaluations
\cite{srinivas2010gaussian,chowdhury2017kernelized}.\footnote{Recently, \citet{calandriello_2019_coltgpucb} connected this quantity
	to the so-called effective dimension $\deff$ of the GP.}
Among approximate GP optimization methods, budgeted kernelized bandits
(\bkb) \cite{calandriello_2019_coltgpucb} and Thompson sampling with quadrature
random features (\textsc{TS-QFF}) \cite{mutny_efficient_2018} are currently the only provably
scalable methods that achieve the $\bigotime(\sqrt{T}\gamma_T)$ rate with sub-cubic
computational complexity.
However they both fail to achieve a fully satisfactory runtime.
\textsc{TS-QFF}'s complexity\footnote{
The $\abigotime(\cdot)$ notation ignores logarithmic dependencies.} $\abigotime(T2^d\deff^2)$ scales exponentially
in $d$, and therefore can only be applied to low-dimensional input spaces,
while \bkb's complexity is still quadratic $\abigotime(T^2\deff^2)$.
Furthermore, both \textsc{TS-QFF} and \bkb are constrained to a sequential protocol
and therefore suffer from poor experimental scalability.

In batch GP optimization, \citet{desautels_parallelizing_2014}
introduced a batched version of \gpucb (\bgpucb) that can effectively deal
with delayed feedback, essentially matching the rate of \gpucb. 
Successive methods improved on this approach \citespecific{{App.~G}}{daxberger17a}
but are too expensive
to scale
and/or require strong assumptions on the function $f$ \cite{contal_parallel_2013}.
\citet{kathuria_batched_2016} uses determinantal point process (DPP) sampling
to globally select the batch of points, but DPP sampling is an expensive
process in itself, requiring cubic time in the number of alternatives.
Although some MCMC-based approximate DPP sampler are scalable, they do not provide
sufficiently strong guarantees to prove low regret.
Similarly \citet{daxberger17a} use a Markov-based approximation to select queries, but lose all regret guarantees in the process.
In general, the time and space complexity of selecting each candidate 
 $\bx_t$ in the batch remains at least $\bigotime(t^2)$, resulting 
in an overall $\bigotime(T^3)$ time and $\bigotime(T^2)$ space complexity, which is prohibitive beyond a few thousands evaluations.

\textbf{Contributions.}
In this paper we introduce a novel sparse approximation
for batched \gpucb, batch budgeted kernelized bandits (\bbkb)
with a \emph{constant} $\bigotime(\deff^2)$
amortized per-step complexity that can easily scale to
tens of thousands of iterations.\todoaout{We should decide how to sell this ``constant'' claim.}
If $\armset$ is finite with $A$ candidates, we prove that \bbkb runs in near-linear $\abigotime(TA\deff^2 + T\deff^3)$ time and $\abigotime(A\deff^2)$ space and it achieves a regret of order $\bigotime(\sqrt{T}\gamma_T)$, thus matching both \gpucb and \bgpucb at a fraction of the computational costs (i.e., their complexity scales as $O(T^3)$).
This is achieved with two new results of independent interest.
First we introduce a new adaptive schedule to select the sizes of the batches of candidates,
where the batches selected are larger than the ones used by \bgpucb
\cite{desautels_parallelizing_2014} while providing the same regret guarantees.
Second we prove that the same
adaptive schedule can be used to delay updates to \bkb's sparse GP approximation~\cite{calandriello_2019_coltgpucb},
also without compromising regret.
This results in large computational savings (\ie from $\abigotime(t\deff^2)$ to $\abigotime(\deff^2)$ per-step complexity) even in the sequential setting,
since updates to the sparse GP approximation, \ie resparsifications,
are the most expensive operation performed by \bkb. Delayed resparsifications also allow us to exploit important
implementation optimizations in \bbkb, such as rank-one updates and lazy updates
of GP posterior.
We also show that our approach can be combined with existing initialization
procedures, both to guarantee a desired minimum level of experimental parallelism
and to include pre-existing feedback to bootstrap the optimization problem.
We validate our approach on several datasets, showing that
\bbkb matches or outperforms existing methods in both regret and runtime.

\section{Preliminaries}\label{sec:preliminaries}

\textbf{Setting.} A learner is provided with a decision set $\armset$ (e.g., a compact set in $\Re^d$) and it sequentially selects candidates $\bx_1, \ldots, \bx_T$ from $\armset$. At each step $t$, the learner receives a feedback $y_t \triangleq f(\bx_t) + \eta_t$, where $f$ is an unknown function, and $\epsilon_t$ is an additive noise drawn i.i.d.\ from $\mathcal{N}(0,\xi^2)$.\footnote{Candidates are sometime referred to as actions, arms, or queries, and feedback is sometimes called reward or observation.} We denote by 
$\bX_t  \triangleq [\bx_1, \dots, \bx_t]^\transp \in \Real^{t \times d}$ the matrix of the candidates selected so far, and with $\by_t \triangleq [y_1, \dots, y_t]^\transp$ the corresponding feedback.
We evaluate the performance of the learner by its regret, i.e., $R_T \triangleq \sum_{t=1}^T f^\star - f(\bx_t)$, where $f^\star = \max_{\bx\in\armset} f(\bx)$ is the maximum of $f$. 
In many applications (e.g., optimization of chemical products) it is possible to execute multiple experiments in parallel. In this case, at step $t=1$ the learner can select \emph{a batch} of candidates and wait for all feedback $y_1,..,y_{t'}$ before moving to the next batch, starting at $t' > t$.
To relate steps with their batch, we denote by $\fb{t}$ the index of the last step of the previous batch, \ie at step
$t$ we have access only to feedback $\by_{\fb{t}}$ up to step $\fb{t}$.
Finally, $[t]=\{1,\ldots,t\}$ denotes the set of integers up to $t$.

\textbf{Sparse Gaussian processes and \nystrom embeddings.}
GPs \cite{rasmussen_gaussian_2006} are traditionally defined in terms of a mean function $\mu$, which we assume to be zero, and a covariance defined by the (bounded) kernel function $k: \armset\times\armset\rightarrow [0,\kappa^2]$.
Given $\mu$, $k$, and some data, the learner can compute the posterior of the GP.

In the following we introduce the GP posterior using a formulation based on \emph{inducing points} \citep{quinonero-candela_approximation_2007,huggins2018scalable}, also known as sparse GP approximations, which is later convenient to illustrate our algorithm. 
Given a so-called dictionary of inducing points $\coldict \triangleq \{\bx_i\}_{i=1}^m$, let $\kermatrix_{\coldict} \in \Real^{m \times m}$
be the kernel matrix constructed by evaluating $k(\bx_i, \bx_j)$ for all the points in $\coldict$,
and similarly let 
$\bk_{\coldict}(\bx) = [\kerfunc(\bx_1, \bx), \dots, \kerfunc(\bx_m, \bx)]^\transp$.
Then we define a \nystrom embedding
as
$\embfunc(\cdot, \coldict) \triangleq \bK_{\coldict}^{+/2}\bk_{\coldict}(\cdot) : \Real^d \rightarrow \Real^m$,
where $(\cdot)^{+/2}$ indicates the square root of the pseudo-inverse. We
can now introduce the matrix $\embmat(\bX_t, \coldict) = [\embfunc(\bx_1, \coldict), \dots, \embfunc(\bx_t, \coldict)]^\transp \in \Real^{t \times m}$ containing
all candidates selected so far after embedding, and define
$\bV_t = \embmat(\bX_t, \coldict)^\transp\embmat(\bX_t, \coldict) + \lambda\bI~\in~\Real^{m \times m}$.
Following \citet{calandriello_2019_coltgpucb}, the \bkb approximation of the posterior mean, covariance, and variance is\todoaout{Is this Cala or already the original formulation of sparse GP?}
\begin{align}
&\wt{\mu}_t(\bx_i, \coldict) = \embvec(\bx_i,\coldict)^\transp\bV_t^{-1}\embmat(\bx_i,\coldict)^\transp\by_t,\label{eq:nyst-post-gp-mean}\\
&\begin{aligned}
\wt{k}_t(\bx_i, \bx_j, \coldict) = \tfrac{1}{\lambda}\Big(&\kerfunc(\bx_i,\bx_j)- \embvec(\bx_i, \coldict)^\transp\embvec(\bx_j, \coldict)\Big)
 + \embvec(\bx_i, \coldict)^\transp\bV_t^{-1}\embvec(\bx_j, \coldict),
\end{aligned}\label{eq:nyst-post-gp-cov}\\
&\wt{\sigma}_t^2(\bx_i, \coldict) = \wt{k}_t(\bx_i, \bx_i, \coldict),\label{eq:nyst-post-gp}
\end{align}
where $\lambda$ is a parameter to be tuned. The subscript $t$
in $\wt{\mu}_t$ and $\asigma_t$ indicates what we already observed
(\ie $\bX_t$ and $\by_t$), and $\coldict$ indicates the dictionary used for the embedding.
Moreover, if $\coldict_{\text{exact}}$ is a \textit{perfect} dictionary
we recover a formulation almost equivalent\footnote{
We refer to $\wt{\mu}_t$ and $\asigma_t$ as posteriors with a slight abuse
of terminology.
In particular, up to a $1/\lambda$ rescaling, they correspond
to the Bayesian DTC approximation \cite{quinonero-candela_approximation_2007},
which is not a GP posterior in a strictly Bayesian sense.
Our rescaling $1/\lambda$ is also not present when deriving the exact $\mu_t$ and $\sigma_t$ 
as Bayesian posteriors, but is necessary to simplify
the notation of our frequentist analysis.
For more details, see \Cref{sec:app-soa}.
}
to the standard
posterior mean and covariance of a GP, which we denote with
$\mu_t(\bx) \triangleq  \wt{\mu}_t(\bx, \coldict_{\text{exact}})$ and
$\sigma_t(\bx) \triangleq \wt{\sigma}_t(\bx, \coldict_{\text{exact}})$.
Examples of possible $\coldict_{\text{exact}}$ are the whole set $\armset$ if
finite, or the set of all candidates $\{\bx_1, \dots, \bx_T\}$ selected so far.

Finally, 
we define the effective dimension
after $t$ steps as
\begin{align*}
\deff(\bX_t) = \sum_{s=1}^t \sigma_t(\bx_s) = \Tr(\bK_t(\bK_t + \lambda\bI)^{-1}).
    \end{align*}

Intuitively, $\deff(\bX_t)$ captures the effective
number of parameters in $f$, \ie the posterior $f$ can be represented using roughly $\deff(\bX_t)$ coefficients.
We use $\deff$ to denote $\deff(\bX_T)$ at the end of the process.
Note that $\deff$ is equivalent
to the maximal conditional mutual information $\gamma_T$ of the GP
\cite{srinivas2010gaussian}, up to logarithmic terms~\citespecific{Lem.~1}{calandriello_2017_icmlskons}.

\textbf{The \gpucb family.}
\gpucb-based algorithms aim to construct an \textit{acquisition function} $u_t(\cdot): \armset \to \Real$ to act as an upper confidence bound (UCB)
for the unknown function $f$.
Whenever $u_t(\bx)$ is a valid UCB (\ie $f(\bx) \leq u_t(\bx)$) and it converges to $f(\bx)$
"sufficiently" fast, selecting candidates that are optimal \wrt to $u_t$ leads to low regret,
\ie the value $f(\bx_{t+1})$ of $\bx_{t+1} = \argmax_{\bx \in \armset}u_t(\bx)$ tends to $\max_{\bx \in \armset}f(\bx)$ as $t$ increases.

In particular, \gpucb~\cite{srinivas2010gaussian} defines
 $u_t(\bx) = \mu_t(\bx) + \beta_t\sigma_t(\bx)$. 
Unfortunately, \gpucb is computationally and experimentally inefficient, as evaluating $u_t(\bx)$
requires $\bigotime(t^2)$ per-step and no parallel experiments are possible.
To improve computations, \bkb \cite{calandriello_2019_coltgpucb}
replaces $u_t$ with an approximate $\wt{u}_t^{\bkb}(\bx) = \amu_t(\bx,\coldict_t) + \wt{\beta}_t\asigma_t(\bx,\coldict_t)$,
which is proven to be sufficiently close to $u_t$ to achieve low regret.
However, maintaining accuracy requires $\bigotime(t)$ per step to update the
dictionary $\coldict_t$ at each iteration,
and the queries are still selected sequentially.
\bgpucb \cite{desautels_parallelizing_2014} tries to increase \gpucb's
experimental efficiency by selecting a batch of queries that
are all evaluated in parallel.
In particular, \bgpucb approximate the UCB as
$\wt{u}_t^{\bgpucb}(\bx) = \mu_{\fb{t}}(\bx) + \beta_t\sigma_t(\bx)$,
where the mean is not updated until new feedback arrives,
while due to its definition the variance only depends on $\bX_t$ and can
be updated in an unsupervised manner.
Nonetheless, \bgpucb is as computationally slow as \gpucb.
More details about these methods are reported in
\Cref{sec:app-soa}.

\textbf{Controlling regret in batched Bayesian optimization.}
For all steps \textit{within} a batch, \bgpucb can be seen as \emph{fantasizing} or \emph{hallucinating} a constant feedback $\mu_{\fb{t}}(\bx_t)$ so that the mean does not change,
while the variances keep \emph{shrinking}, thus promoting diversity in the batch.
However, incorporating fantasized feedback causes
${u}^{\bgpucb}_t$ to drift away from $u_t$ to the extent that it
may not be a valid UCB anymore. \citet{desautels_parallelizing_2014} show that this issue can be managed by adjusting \gpucb's parameter $\beta_t$. In fact, it is possible to take the ${u}^{\bgpucb}_t$ at the beginning of the batch (which is a valid UCB by definition), and \emph{correct} it to hold for each hallucinated step as
$f(\bx)
\leq  \mu_{\fb{t}}(\bx) + \rho_{\fb{t},t}(\bx)\beta_{\fb{t}}\sigma_t(\bx),$
where $\rho_{\fb{t},t}(\bx) \triangleq \frac{\sigma_{\fb{t}}(\bx)}{\sigma_t(\bx)}$ is the posterior variance ratio. By using any $\alpha_t \geq \rho_{\fb{t},t}(\bx)\beta_{\fb{t}}$, we have that ${u}^{\bgpucb}_t$ is a valid UCB. As the length of the batch increases, the ratio $\rho_{\fb{t},t}$ may become larger, and the UCB becomes less and less tight. Crucially, the drift of the ratio can be estimated.
\begin{proposition}[\citet{desautels_parallelizing_2014}, Prop.\,1]\label{p:std.dev.ratio}
	At any step $t$, for any $\bx\in\armset$ the posterior ratio is bounded as
    \begin{align*}
	\rho_{\fb{t},t}(\bx) \triangleq \frac{\sigma_{\fb{t}}(\bx)}{\sigma_t(\bx)}
    \leq \prod_{s = \fb{t}+1}^{t}\left(1 + \sigma_{s-1}^2(\bx_s)\right).
	\end{align*}
\end{proposition}
Based on this result, \bgpucb continues the construction of the batch while $\prod_{s = \fb{t}+1}^{t}\left(1 + \sigma_{s-1}^2(\bx_s)\right) \leq C$ for some designer-defined threshold of drift $C$.
Therefore, applying \Cref{p:std.dev.ratio}, we have that the ratio $\rho_{\fb{t},t}(\bx) \leq C$ for any $\bx$, and setting $\alpha_t \triangleq C\beta_{\fb{t}}$ guarantees the validity of the UCB, just as in \gpucb. As a consequence, \gpucb's analysis can be leveraged to provide guarantees on the regret of \bgpucb.

\section{Efficient Batch GP Optimization}\label{sec:bbkb}

In this section, we introduce \bbkb which both generalizes and improves over \bgpucb and \bkb.

\begin{algorithm}[!t]
    \begin{small}
    \setstretch{1.1}
\begin{algorithmic}[1]
\REQUIRE{Set of candidates $\armset$, $\{\wt{\alpha}_t\}_{t=1}^T$, $T$}, $\wt C$, $\{\qbar_t\}_{t=1}^T$
\STATE Sample $\bx_1$ uniformly, receive $\by_1$
\STATE Initialize $\coldict_0 = \{\}$, $\fb{0} = 0$
\FOR{$t=\{0, \dots, T-1\}$}
	\STATE Select $\bx_{t+1} = \argmax_{\bx \in \armset}\wt{u}_{t}(\bx, \coldict_{\fb{t}})$\;
	\IF{$1 + \sum_{s = \fb{t}+1}^{t+1} \asigma_{\fb{t}}(\bx_s, \coldict_{\fb{t}}) \leq \wt{C}$}
	\STATE \textit{// $\fb{t+1} = \fb{t}$, batch construction step}
	\STATE Update $\wt{u}_{t+1}(\bx_{t+1}, \coldict_{\fb{t+1}})$ with the new $\asigma_{t+1}$
	\STATE Update $\wt{u}_{t+1}(\bx_i, \coldict_{\fb{t+1}})$ for all\\* \hspace*{1cm} $\{\bx:\wt{u}_{t}(\bx, \coldict_{\fb{t}}) \geq \wt{u}_{t+1}(\bx_{t+1}, \coldict_{\fb{t}})\}$
	\ELSE
	\STATE \textit{// $\fb{t+1} = t+1$, resparsification step}
		\STATE Initialize $\coldict_{\fb{t+1}} = \emptyset$
		\FOR{$\bx_s \in \bX_{\fb{t+1}}$}
		\STATE Set $\wt{p}_{\fb{t+1},s} = \qbar_t \cdot \asigma_{\fb{t}}^2(\bx_s)$
		\STATE Draw $z_{\fb{t+1}, s} \sim Bernoulli(\wt{p}_{\fb{t+1},s})$
		\STATE If $z_{\fb{t+1}, s} = 1$, add $\bx_s$ in $\coldict_{\fb{t+1}}$
		\ENDFOR
		\STATE Get feedback $\{y_s\}_{s=\fb{t}+1}^{\fb{t+1}}$
		\STATE Update $\amu_{\fb{t+1}}$ and $\wt{\sigma}_{\fb{t+1}}$ for all $\bx \in \armset$
	\ENDIF
\ENDFOR
\end{algorithmic}
    \end{small}
\caption{
\bbkb\label{alg:bbkb}
}
\end{algorithm}

\subsection{The algorithm}
The pseudocode of \bbkb is presented in \Cref{alg:bbkb}.
The dictionary is initially empty, and we have $\amu_0(\bx) = 0$ and $\asigma_0(\bx, \{\}) = \kerfunc(\bx, \bx)/\lambda = \sigma_0(\bx)$.
At each step $t$ \bbkb chooses the next candidate 
$\bx_{t+1}$ as the maximizer of the UCB $\wt{u}_t(\bx) = \amu_{\fb{t}}(\bx,\coldict_{\fb{t}}) + \wt{\alpha}_{\fb{t}}\asigma_t(\bx,\coldict_{\fb{t}})$, which combines \bkb and \bgpucb's approaches with a new element. In $\wt{u}_t$, not only we delay feedback updates as we use the posterior mean computed at the end of the last batch $\amu_{\fb{t}}$ but, unlike \bkb, we keep using the same dictionary $\coldict_{\fb{t}}$ for all steps in a batch. While freezing the dictionary leads to significantly reducing the computational complexity, delaying feedback and dictionary updates may result in poor UCB approximation. Similar to \bgpucb, after selecting $\bx_{t+1}$ we test the condition in (L5) to decide whether
to continue the batch, a \textit{batch construction step}, or not, a \textit{resparsification
step}. The specific formulation of our condition is crucial to guarantee near-linear
runtime, and improves over the condition used in \bgpucb. Notice that 
if we update dictionary and feedback at each step, \bbkb reduces to \bkb (up to an improved $\wt{\alpha}_t$ as discussed in the next section), while if $\coldict_t = \bX_t$ we recover \bgpucb, but with an improved terminating rule for batches.

In a batch construction step, we keep using the same dictionary in computing the UCBs used to select the next candidate. 
On the other hand, if condition (L5) determines that UCBs may become too loose, we interrupt the batch and update the sparse GP approximation, \ie we resparsify the dictionary.
To do this we employ \bkb's posterior sampling procedure in L11-16. 
For each candidate $\bx_s$ selected so far, we compute an inclusion probability $\wt{p}_{\fb{t+1},s} = \qbar_t \cdot \asigma_{\fb{t}}(\bx_s)$, where $\qbar_t \geq 1$ is a parameter trading-off the size of $\coldict$ and the accuracy of the approximations, and we add $\bx_s$ to the new dictionary with probability $\wt{p}_{\fb{t+1},s}$. A crucial difference w.r.t.\ \bkb is that in computing the inclusion probability we use the posterior variances computed at the beginning of the batch (instead of $\asigma_{\fb{t+1}}$). While this introduces a further source of approximation, in the next section we show that this error can be controlled. The resulting dictionary is then used to compute the embedding $\embfunc_{\fb{t+1}}$ and the UCB values whenever needed.

\textbf{Maximizing the UCB.} To provide a meaningful
Since in general $u_t$ is a highly non-linear, non-convex function,
it may be NP-hard to compute $\bx_{t+1}$ as its $\argmax$ over $\armset$. To simplify the exposition, in the rest of the paper we assume that
$\armset$ is finite with cardinality $A$ such that simple enumeration of all candidates in $\armset$ is sufficient to exactly
optimize the UCB. Both this assumption and the runtime dependency on $A$
can be easily removed if an efficient way to optimize
$u_t$ over $\armset$ is provided (e.g., see \citet{mutny_efficient_2018} for the special case of $d=1$ or when $k$ is an additive kernel).

Moreover, when $\armset$ is finite \bbkb can be implemented much
more efficiently. In particular, many of the quantities used
by \bbkb can be precomputed once at the beginning of the batch, such
as pre-embedding all arms.
In addition keeping the embeddings fixed during the batch
allows us to update the posterior variances
using efficient rank-one updates, combining the efficiency of a parametric
method with the flexibility of non-parametric GPs.
Finally, when both dictionary and feedback are fixed we can leverage
\emph{lazy} covariance evaluations,
which allows us to exactly compute the $\bx_{t+1}$ while only updating
a small fraction of the UCBs  (see \Cref{sec:app-soa} for more details).

\subsection{Computational analysis}
The global runtime of \bbkb is 
$\bigotime\big(T\Narm m^2 + \numbatches T m + \numbatches (\Narm m^2 + m^3)\big)$,
where $m = \max_t |\coldict_{\fb{t}}|$ is the maximum size of the dictionary/embedding
across batches,
and $\numbatches$ the number of batches/resparsifications (see \Cref{sec:proof-regret} for details).
In order to obtain a near-linear runtime,
we need to show that both $|\coldict_t|$ and $\numbatches$ are nearly-constant.
\begin{restatable}{theorem}{bbkbcomplexity}\label{thm:bbkb-complexity}
Given $\delta \in (0,1)$, $1 \leq \wt{C}$, and $1 \leq \lambda$, run \bbkb with
$\qbar_t \geq 8\log(4t/\delta)$.
Then, w.p.~$1-\delta$
\begin{itemize}[leftmargin=1.5em]
\item[1)] For all $t \in [T]$ we have
$|\coldict_{t}| \leq 9\wt{C}(1+\kappa^2/\lambda)\qbar_t\deff(\bX_t)$.
\item[2)] Moreover, the total number of resparsification $B$ performed by \bbkb is at most
$\bigotime(\deff(\bX_t))$.
\item[3)] As a consequence, 
\bbkb runs in near-linear time $\abigotime(T \Narm \deff(\bX_t)^2)$ .
\end{itemize}
\end{restatable}
\Cref{thm:bbkb-complexity} guarantees that whenever $\deff$, or equivalently $\gamma_T$, is near-constant (\ie $\abigotime(1)$),
\bbkb runs in $\abigotime(T\Narm)$.
\citet{srinivas2010gaussian} shows that this is the case for common kernels, \eg $\gamma_T \leq \bigotime(d\log(T))$
for the linear kernel and $\gamma_T \leq \bigotime(\log(T)^{d})$ for the Gaussian
kernel.

Among sequential GP-Opt algorithms, \bbkb is not only much faster than
the original \gpucb $\abigotime(T^3\Narm)$ runtime, but also much faster when
compared to \bkb's quadratic $\abigotime(T\max\{\Narm,T\}\deff^2)$.
\bbkb's runtime also improves over GP-optimization algorithms that are specialized for stationary kernels (e.g.~Gaussian),
such as \textsc{QFF-TS}'s \cite{mutny_efficient_2018} $\abigotime(T\Narm 2^d\deff^2)$ runtime,
without making any assumption on the kernel and without an exponential dependencies on the input dimension $d$.
When compared to batch algorithms, such as \bgpucb, the improvement is even sharper as all
existing batch algorithms that are provably no-regret \cite{desautels_parallelizing_2014,contal_parallel_2013,shah2015parallel} share
\gpucb's $\abigotime(\Narm T^3)$ runtime.

\todoa{Here we should elaborate more on $\deff$ and how much it can indeed change with $t$ and possibly its connection with $\gamma_T$ for which bounds already exist depending on the kernel.}

\todoa{I would remove this whole paragraph as we have a full section on the importance of the condition and we can recall its effect there (just trying to avoid repetitions and shorten a bit the text, otherwise I would keep the paragraph).}
One of the central elements
of this result is \bbkb's adaptive batch terminating condition. As a comparison,
\bgpucb 
uses $\prod_{s = \fb{t}+1}^{t+1} (1 + \sigma_{\fb{t}}(\bx_s))$
as a batch termination condition,
but
due to Weierstrass product inequality
\begin{align*}
1 + \sum_{s = \fb{t}+1}^{t+1} \sigma_{\fb{t}}(\bx_s) \leq \prod_{s = \fb{t}+1}^{t+1} (1 + \sigma_{\fb{t}}(\bx_s)),
\end{align*}
and
the product
is always larger than the sum which \bbkb uses.
Thanks to the tighter bound, we obtain larger batches and can guarantee
that at most $\abigotime(\deff)$ batches are necessary over $T$ steps.
This implies that, unless $\deff \to \infty$ in which case the optimization
would not converge in the first place, the size of the batches must on average grow linearly with
$T$ to compensate.
In addition to this guarantee on the average batch size,
in the next section we show how smarter initialization
schemes can guarantee a minimum batch size, which is useful to fully utilize any desired level of parallelism.
Finally, note that the $\Narm$ factor reported in the runtime is pessimistic,
since \bbkb recomputes only a
small fraction of UCB's at each step thanks to lazy evaluations, and should be considered simply as a
proxy of the time required to find the UCB maximizer.

\subsection{Regret analysis} 
We report regret guarantees for \bbkb in the so-called \textit{frequentist} setting. While the algorithm uses GP tools to define and manage the uncertainty in estimating the unknown function $f$, the analysis of \bbkb does not rely on any \textit{Bayesian} assumption about $f$ being actually drawn from the prior $\text{GP}(0,k)$, and it only requires $f$ to be bounded in the norm associated to the RKHS induced by the kernel function $k$.

\begin{theorem}\label{thm:main-regret}
Assume $\normsmall{f}_{\rkhs} \leq F < \infty$, and let $\xi^2$ be the variance of
the noise $\eta_t$.
For any desired, $0 < \delta < 1$,
$1 \leq \lambda$, $1 \leq \wt{C}$,
if we run \bbkb with
$\qbar_t \geq 72\wt{C}\log(4t/\delta)$,
$\wt{\alpha}_{\fb{t}} = \wt{C}\wt{\beta}_{\fb{t}}$, and
\begin{align*}
\wt{\beta}_{\fb{t}} = 
2\xi&\sqrt{\sum\nolimits_{s=1}^{\fb{t}}\log\left(1 + 3\asigma_{\fb{s-1}}^2(\bx_s)\right) + \log\left(\tfrac{1}{\delta}\right)}\\
&+ (1 + \sqrt{2})\sqrt{\lambda}\fnorm,
\end{align*}
then, with prob.\@ $1-\delta$, \bbkb's regret 
is bounded as
\begin{align*}
R_T^{\bbkb}
&\leq 55\wt{C}^2
R_T^{\bgpucb}
\leq 55\wt{C}^3R_T^{\gpucb}
\end{align*}
with $R_T^{\gpucb}$ bounded by
\begin{align*}
\sqrt{T}\bigg(\xi\bigg(\sum\limits_{t=1}^{T}\sigma_{t-1}^2(\bx_{t}) + \log\left(\tfrac{1}{\delta}\right)\bigg) + \sqrt{\lambda \fnorm^2\sum\limits_{t=1}^{T}\sigma_{t-1}^2(\bx_{t})}\bigg).
\end{align*}
\end{theorem}

\Cref{thm:main-regret} shows that \bbkb achieves essentially the same
regret as \bgpucb and \gpucb, but at a fraction of the computational cost.
Note that $\sum_{t=1}^{T}\sigma_{t-1}^2(\bx_{t})
\approx \logdet(\bK_T/\lambda + \bI) \approx \gamma_T$ \citespecific{Lem.~5.4}{srinivas2010gaussian}. 
Such a tight bound is achieved thanks in part to a new confidence interval
radius $\wt{\beta}_t$. In particular \citet{calandriello_2019_coltgpucb}
contains an extra $\logdet(\bK_T/\lambda + \bI) \leq \deff(\lambda, \bX_T)\log(T)$
bounding step that we do not have to make. While in the worst case this is
only a $\log(T)$ improvement, empirically the data adaptive bound
seems to lead to much better regret.\tododout{explain a bit more why it is cool}

\textbf{Discussion.}
\bbkb directly generalizes and improves both \bkb and \bgpucb. 
If $\wt{C}=1$, \bbkb is equivalent to \bkb, with a improved $\wt{\beta}_t$ and a slightly better regret by a factor $\log(T)$,
and if $\coldict_t = \coldict_{\text{exact}}$ we recover \bgpucb, with an improved
batch termination rule. \bbkb's algorithmic derivation and
analysis require several new tools.
A direct extension of \bkb to the batched setting would achieve low regret
but be computationally expensive. In particular, it is easy to extend
\bkb's analysis to delayed feedback, but only if \bkb adapts the embedding space
(\ie resparsifies the GP) after every batch construction step to maintain guarantees at all times,
\ie \Cref{thm:bbkb-complexity} must hold at all steps and not only
at $\fb{t}$. However this causes large computation issues,
as embedding the points is \bkb's most expensive operation, and
prevents any kind of lazy evaluation of the UCBs. \bbkb 
solves these two issues with a simple algorithmic fix by \emph{freezing the dictionary} during the batch. 
However, this bring additional challenges for the analysis. The reason is that 
while the dictionary is frozen, we may encounter a point $\bx_t$
that \emph{cannot be well represented} with the current $S_t$, but we 
cannot add $\bx_t$ to it since the dictionary is frozen.
This requires studying how posterior
mean and variance drift away from their values at the beginning of the batch.
We tackle this problem by \emph{simultaneously} freezing the dictionary and batching candidates.
As we will see in the next section, and prove in the appendix,
freezing the dictionary allows us to control the ratio $\asigma_{\fb{t}}(\bx_i, \coldict_{\fb{t}})/\asigma_{t}(\bx_i, \coldict_{\fb{t}})$,
obtaining a generalization of \Cref{prop:bkb-confidence-interval}.
However, changing the posterior mean $\amu_{t}(\cdot, \coldict_{\fb{t}})$ without
changing the dictionary could still invalidate the UCBs.
Batching candidates allows \bbkb to continue using the posterior mean $\amu_{\fb{t}}(\cdot, \coldict_{\fb{t}})$,
which is known to be accurate, and resolve this issue.
By terminating the batches when exceeding a prescribed potential error threshold
(\ie $\wt{C}$) we can ignore the intermediate estimate and reconnect
all UCBs with the accurate UCBs at the beginning of the batch.
This requires a deterministic, worst-case analysis
of both the evolution of $\asigma_{t}(\bx_i)$ and $\sigma_t(\bx_i)$, which we provide in the appendix.

\textbf{Proof sketch.}
One of the central elements in \bbkb's computational and regret analysis
is the new adaptive batch terminating condition.
In particular, remember that \bgpucb's regret analysis was centered around
the fact that the posterior ratio $\rho_{\fb{t},t}(\bx)$ from \Cref{p:std.dev.ratio}
could be controlled using \citet{desautels_parallelizing_2014}'s batch termination rule.
However, this result cannot be transferred directly to \bbkb for several reasons.
First, we must not only control the ratio $\rho_{\fb{t},t}(\bx)$, but also
the approximate ratio $\wt{\rho}_{\fb{t},t}(\bx, \coldict) \triangleq \frac{\asigma_{\fb{t}}(\bx, \coldict)}{\asigma_t(\bx, \coldict)}$
for some dictionary $\coldict$, since we are basing most of our choices
on $\asigma_t$ but will be judged based on $\sigma_t$ (\ie the real function
is based on $k$ and $\sigma_t$, not on some $\wt{k}$ and $\asigma_t$). Therefore our termination rule
must provide guarantees for both.
Second, \bgpucb's rule is not only expensive to compute,
but also hard to approximate.
In particular, if we approximated $\sigma_{\fb{t}}(\bx_s)$ with
$\asigma_{\fb{t}}(\bx_s)$ in \Cref{p:std.dev.ratio},
any approximation error incurred would be compounded multiplicatively by the product
resulting in an overall error \emph{exponential} in the length of the batch.
Instead, the following novel ratio bound involves only summations.
\begin{restatable}{lemma}{asigmaevol}\label{lem:approx-batch-rls-accuracy}
For any kernel $k$, dictionary $\coldict$, set of points $\bX_t$, $\bx \in \armset$, and $\fb{t} < t$,
we have
\begin{align*}
\wt{\rho}_{\fb{t},t}(\bx, \coldict)
\leq 1 + \sum\nolimits_{s=\fb{t}}^{t}\asigma_{\fb{t}}(\bx, \coldict).
\end{align*}
\end{restatable}
The proof, reported in the appendix, is based only on linear algebra and
does not involve any GP-specific derivation, making it applicable to the DTC
approximation used by \bbkb.
Most importantly, it holds regardless of the dictionary $\coldict$ used (as long as it stays
constant) and regardless of which candidates an algorithm might include in the batch.
If we replace $\asigma_{\fb{t}}$
with $\sigma_{\fb{t}}$ the bound can also be applied to ${\rho}_{\fb{t},t}(\bx, \coldict)$,
giving us an improved version of \Cref{p:std.dev.ratio} as a corollary.
Finally, replacing the product in \Cref{p:std.dev.ratio} with
the summation in \Cref{lem:approx-batch-rls-accuracy} makes it much easier
to analyse it, leveraging this result adapted from \cite{calandriello_2019_coltgpucb}.
\begin{restatable}{lemma}{bbkbaccuracy}\label{lem:bbkb-rls-accuracy}
	Under the same conditions as \Cref{thm:bbkb-complexity},
	w.p.~$1-\delta$, $\forall\; \fb{t} \in [T]$ and $\forall\; \bx \in \armset$ we have
	\begin{align*}
	\sigma_{\fb{t}}^2(\bx)/3 \leq \wt{\sigma}_{\fb{t}}^2(\bx, \coldict_{\fb{t}}) \leq 3\sigma_{\fb{t}}^2(\bx).
	\end{align*}
\end{restatable}
\Cref{lem:bbkb-rls-accuracy} shows that at the beginning
of each batch \bbkb, similarly to \bkb, does not
underestimate uncertainty, \ie unlike existing approximate batched methods it does not suffer from variance starvation
\cite{wang2018batched},
Applying~\Cref{lem:bbkb-rls-accuracy} to \Cref{lem:approx-batch-rls-accuracy}
we show that our batch terminating condition can provide
guarantees on both the approximate ratio $\asigma_{\fb{t}}/\asigma_t \leq \wt{C}$,
as well as the exact posterior ratio $\sigma_{\fb{t}}/\sigma_t \leq 3\wt{C}$
paying only an extra constant approximation factor.
Both of these conditions are necessary to obtain the final regret bound.

\todom{comment related to the discussion with Daniele, about adversary screwing us ----
in our setting, we are not  doing contextual GPs UCB, i.e., it is the learner that is choosing $x$ unlike for PROS'N'KONS,
that means we are not facing a adversary that can screw us ----
yet is seems not ease to guarantee reconstructions \emph{in the middle of the batch}
do you agree with all this?}
\tododout{I agree, but i would not know how to present it}

\section{Extensions}
In this section
we discuss two important extensions of \bbkb: \textbf{1)} how to leverage
initialization to improve experimental parallelism and accuracy, \textbf{2)} how to further trade-off a small amount of extra computation
to improve parallelism.

\textbf{Initialization to guarantee minimum batch size.}
In many cases it is desirable to achieve at least a certain prescribed
level of parallelism $\minbatchsize$, \eg to be able to fully utilize a server farm
with $\minbatchsize$ machines or a lab with $\minbatchsize$ analysis
machines\footnote{For simplicity here we assume that all evaluations require
the same time and that batch sizes are a multiple of $\minbatchsize$.
This can be easily relaxed at the only expense of a more complex notation.}.
However, \bbkb's batch termination rule is designed only to control the ratio
error, and might generate batches smaller than $\minbatchsize$, especially
in the beginning when posterior variances are large and their sum
can quickly reach the threshold $\wt{C}$.
However, it is easy to see that if at step $\fb{t}$ we have
$\max_{\bx \in \armset} \asigma_{\fb{t}}^2(\bx) \leq 1/P$ for all $\bx$,
then the batch will be at least as large as $P$.

The same problem of controlling the maximum posterior variance of a GP
was studied by \citet{desautels_parallelizing_2014}, who showed that
a specific initialization scheme (see \Cref{sec:app-soa} for details) called uncertainty sampling (US)
can guarantee that after $\Tinit$ initialization samples, 
we have that
$\max_{\bx \in \armset} \sigma_{\Tinit}^2(\bx) \leq \gamma_{\Tinit}/\Tinit$.
Since it is known that for many covariances $k$ the maximum information
gain $\gamma_t$ grows sub-linearly in $t$, we have that $\gamma_{\Tinit}/\Tinit$
eventually reaches the desired $1/\minbatchsize$. For example,
for the linear kernel $\Tinit \leq \minbatchsize d\log(\minbatchsize)$ suffices,
and $\Tinit \leq \log(\minbatchsize)^{d+1}$ for the Gaussian kernel.
All of these guarantees can be transferred to our approximate setting thanks to \Cref{lem:bbkb-rls-accuracy} and to the monotonicity of $\sigma_t$.
In particular, after a sufficient number $\Tinit$ of steps of US, and
for any $\fb{t} > \Tinit$ we have that
\begin{align*}
\asigma_{\fb{t}}(\bx, \coldict_{\fb{t}}) \leq 3\sigma_{\fb{t}}(\bx)
\leq 3\sigma_{\Tinit}(\bx) \leq 3/\minbatchsize,
\end{align*}
and US can be used to control \bbkb's batch size as well.

\textbf{Initialization to leverage existing data.}
In many domains GP-optimization is applied to existing problems in hope
to improve performance over an existing decision system (e.g., replace
uniform exploration or A/B testing with a more sophisticated alternative).
In this case, existing historical data can be used to initialize the GP model
and improve regret, as it is essentially ``free'' exploration. However
this still present a computational challenge, since computing the GP
posterior scales with the number of total evaluations, which includes the
initialization. In this aspect, \bbkb can be seamlessly integrated with
initialization using pre-existing data. All that is necessary is to
pre-compute a provably accurate dictionary $\coldict_{\Tinit}$ using any
batch sampling technique that provides guarantees equivalent to those of
\Cref{lem:bbkb-rls-accuracy}, see e.g.,\citet{calandriello_2017_icmlskons,NIPS2018_7810}.
The algorithm then continues as normal from step $\Tinit + 1$ using the embeddings
based on $\coldict_{\Tinit}$, maintaining all computational and regret guarantees.

\textbf{Local control of posterior ratios}
Finally, we want to highlight that the termination rule of \bbkb is just one
of many possible rules to guarantee that the posterior ratio is controlled.
In particular, while \bbkb's rule improves over \bgpucb's, it is still not optimal.
For example, one could imagine
recomputing \emph{all} posterior variances at each step and check that
$\max_{\bx \in \armset} \asigma_{\fb{t}}^2(\bx, \coldict_{\fb{t}})/\asigma_{t}^2(\bx, \coldict_{\fb{t}}) \leq \wt{C}$.

However this kind of local (\ie specific to a $\bx$) test is computationally expensive, as it 
requires a sweep over $\armset$ and at least $\bigotime(|\coldict_{\fb{t}}|^2)$
time to compute each variance, which is why 
\bbkb and \bgpucb's termination rule use only
global information. To try to combine the best of both worlds, we propose a novel efficient local
termination rule.
\begin{lemma}\label{lem:posterior-ratio-local-approx}
For any kernel $k$, dictionary $\coldict$, set of points $\bX_t$, $\bx \in \armset$, and $\fb{t} < t$,
\begin{align*}
\wt{\rho}_{\fb{t},t}(\bx, \coldict)
&\leq 1 + \frac{\sum_{s=\fb{t}}^t\wt{k}_{\fb{t}}^2(\bx, \bx_s, \coldict_{\fb{t}})}{\asigma^2_{\fb{t}}(\bx, \coldict_{\fb{t}})}
\end{align*}
\end{lemma}
Note that $\wt{k}_{\fb{t}}^2(\bx, \bx_s) \leq \asigma_{\fb{t}}^2(\bx)\asigma_{\fb{t}}^2(\bx_s)$,
due to Cauchy-Bunyakovsky-Schwarz,
and therefore this termination rule is tighter than the one in \Cref{lem:approx-batch-rls-accuracy}.
Moreover, with an argument similar to the one in \Cref{lem:approx-batch-rls-accuracy}
we can again show that the termination provably controls both the ratio of exact
and approximate posteriors.

Computationally, after a $\abigotime(\deff^2)$ cost to update $\bV_t^{-1}$,
computing multiple $\wt{k}_{\fb{t}}(\bx, \bx_s, \coldict_{\fb{t}})$ for a fixed
$\bx_s$ requires only $\abigotime(\deff)$ time,
\ie it requires only a vector-vector multiplication in the embedded space.
Therefore, the total cost of updating the posterior ratio estimates
using \Cref{lem:posterior-ratio-local-approx} is $\abigotime(\Narm \deff + \deff^2)$,
while recomputing all variances requires $\bigotime(\Narm \deff^2)$. However,
it still requires a full sweep over all candidates introducing
a dependency on $\Narm$. As commented in the case of posterior
maximization, lazy updates can be used to empirically alleviate this dependency.
Finally,
it is possible to combine both bounds: at first use the global bound from \Cref{lem:approx-batch-rls-accuracy}, and
then switch to the more computationally expensive local bound of \Cref{lem:posterior-ratio-local-approx} only if the constructed batch is not ``large enough''.
\section{Experiments}\label{sec:exp}

\begin{figure}[]
\includegraphics[width=\linewidth]{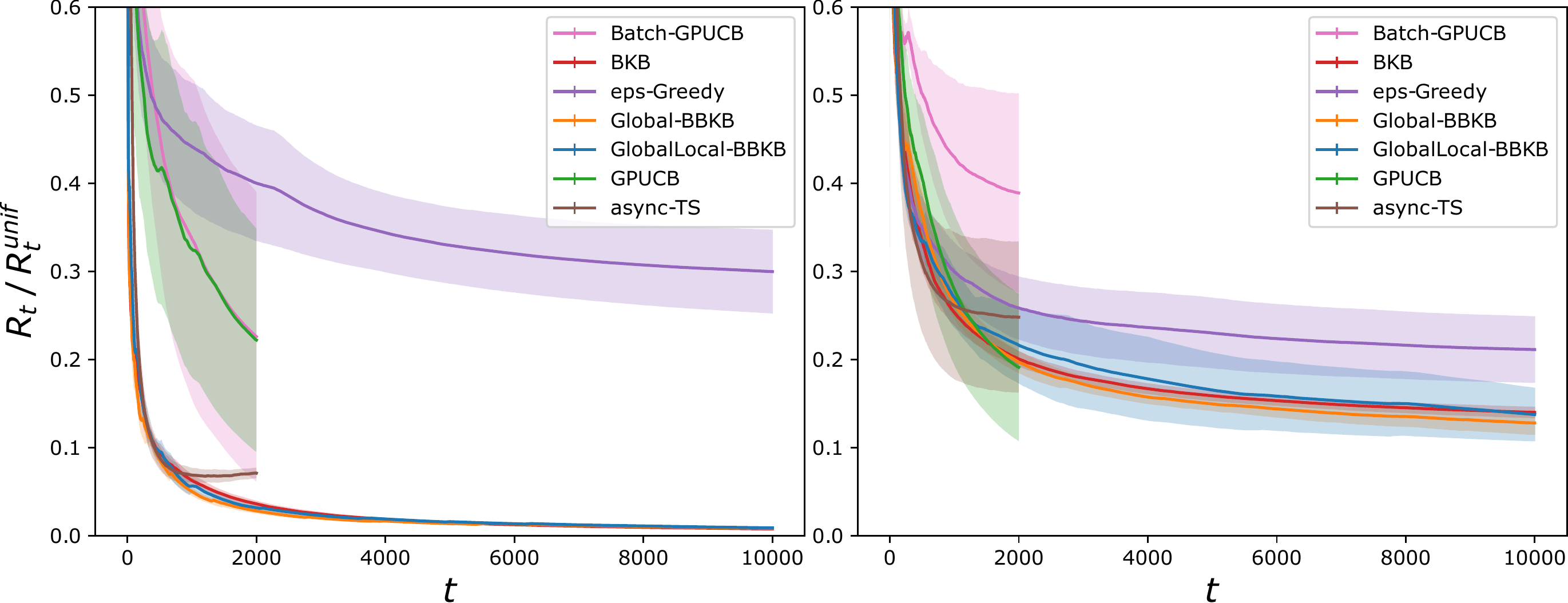}
\caption{\small Regret-ratio on Abalone (left) and Cadata (right)}\label{fig:rr_ac}
\end{figure}
\begin{figure}
\includegraphics[width=\linewidth]{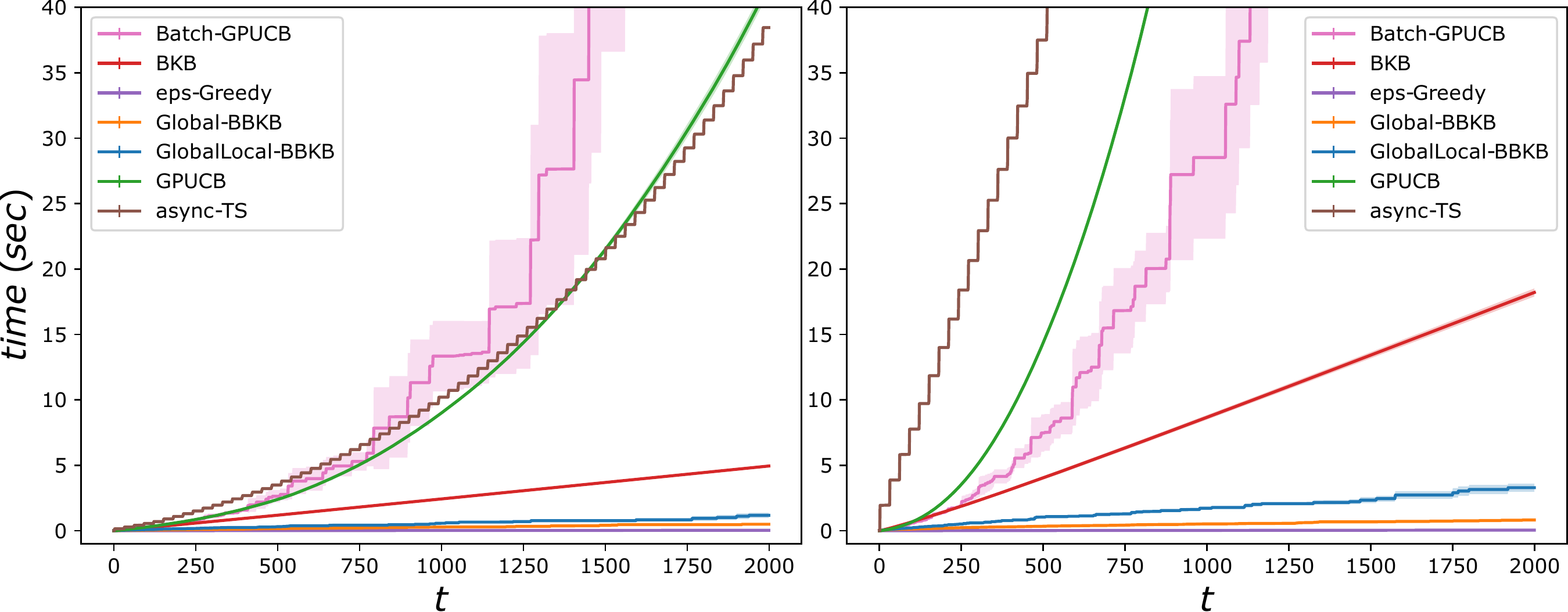}
\caption{\small Time on Abalone (left) and Cadata (right)}\label{fig:t_ac}
\end{figure}

\begin{figure*}[!tb]
\minipage{0.33\textwidth}
\includegraphics[width=\linewidth]{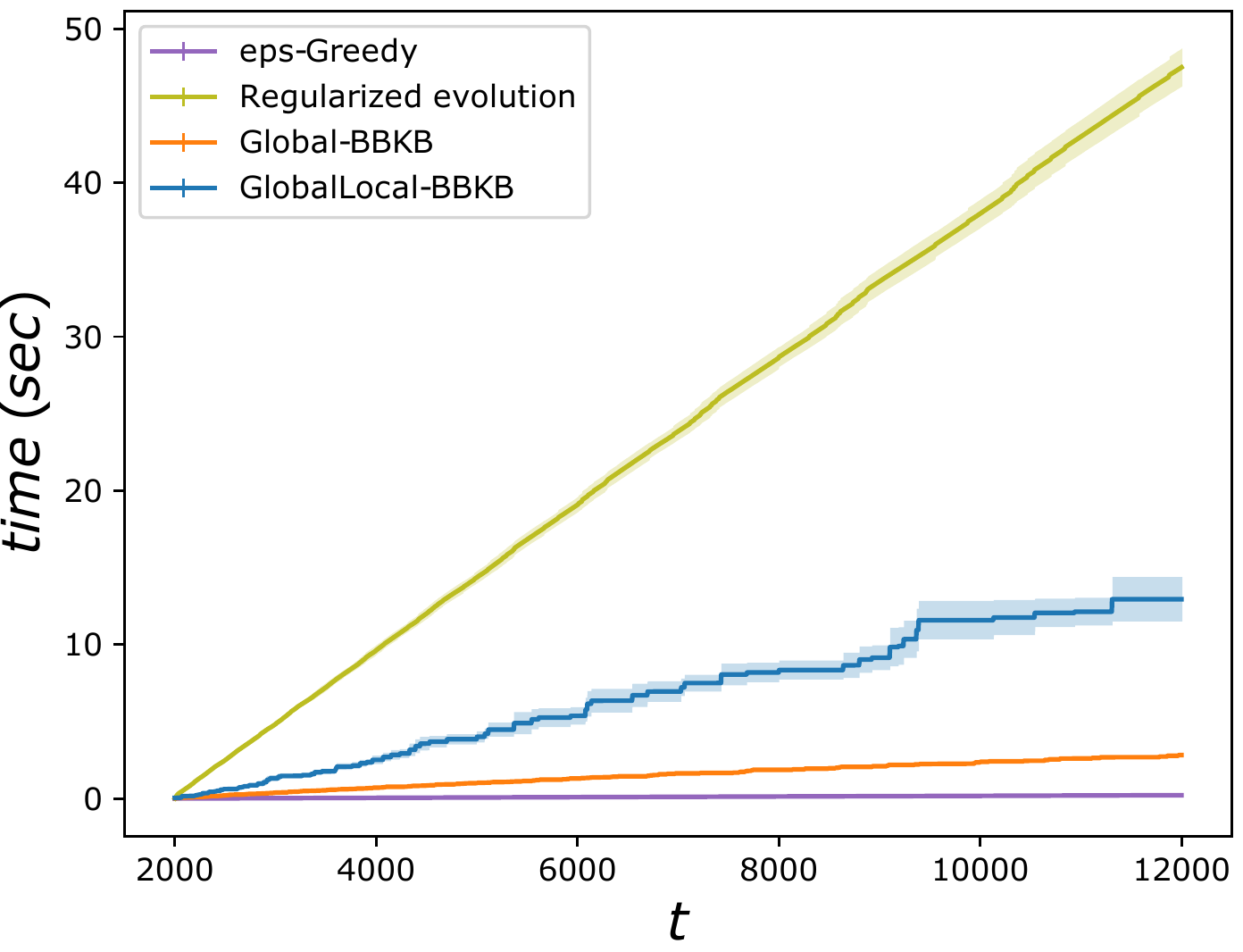}
\endminipage\hfill
\minipage{0.33\textwidth}
\includegraphics[width=\linewidth]{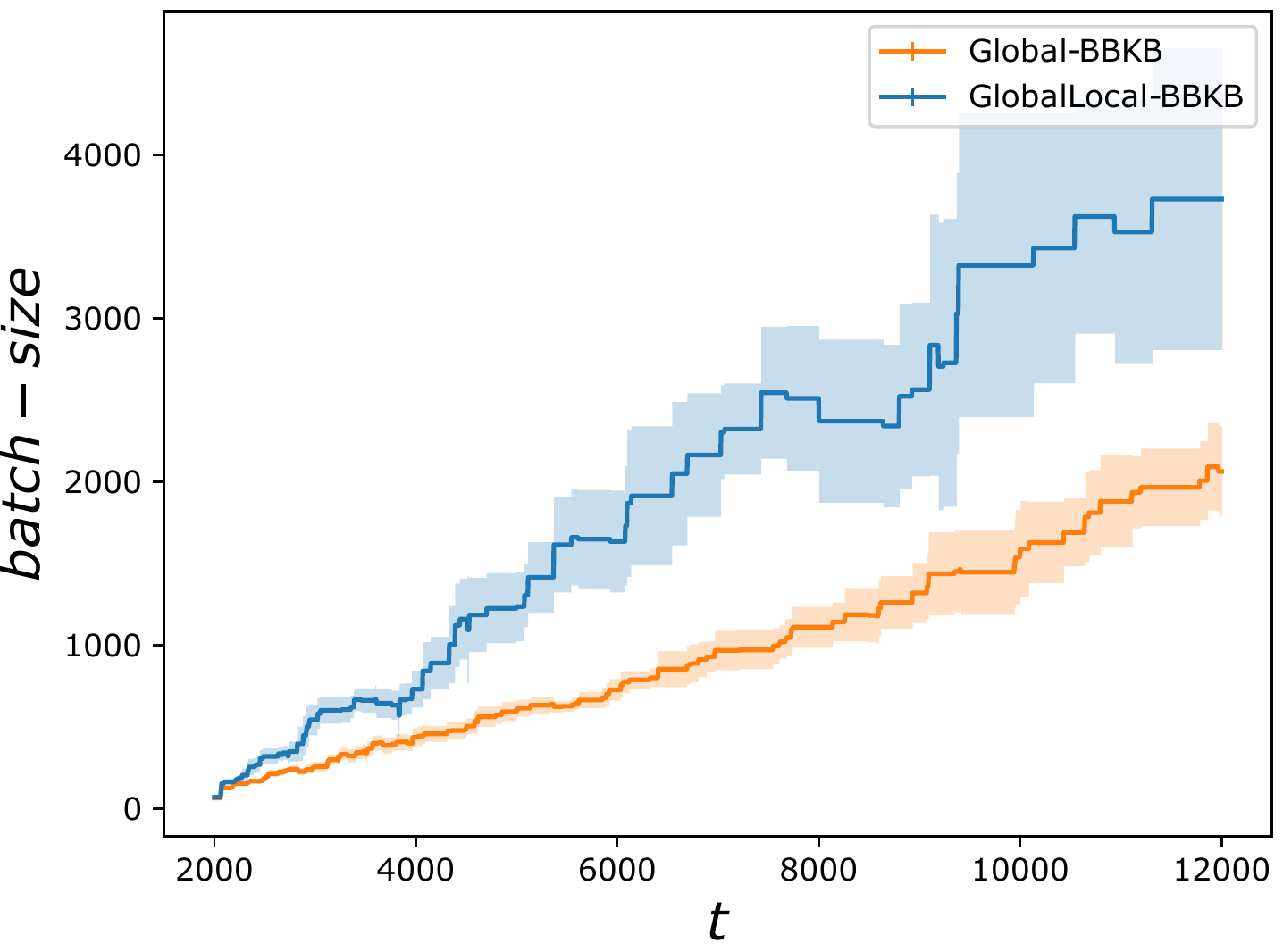}
\endminipage\hfill
\minipage{0.33\textwidth}
\includegraphics[width=\linewidth]{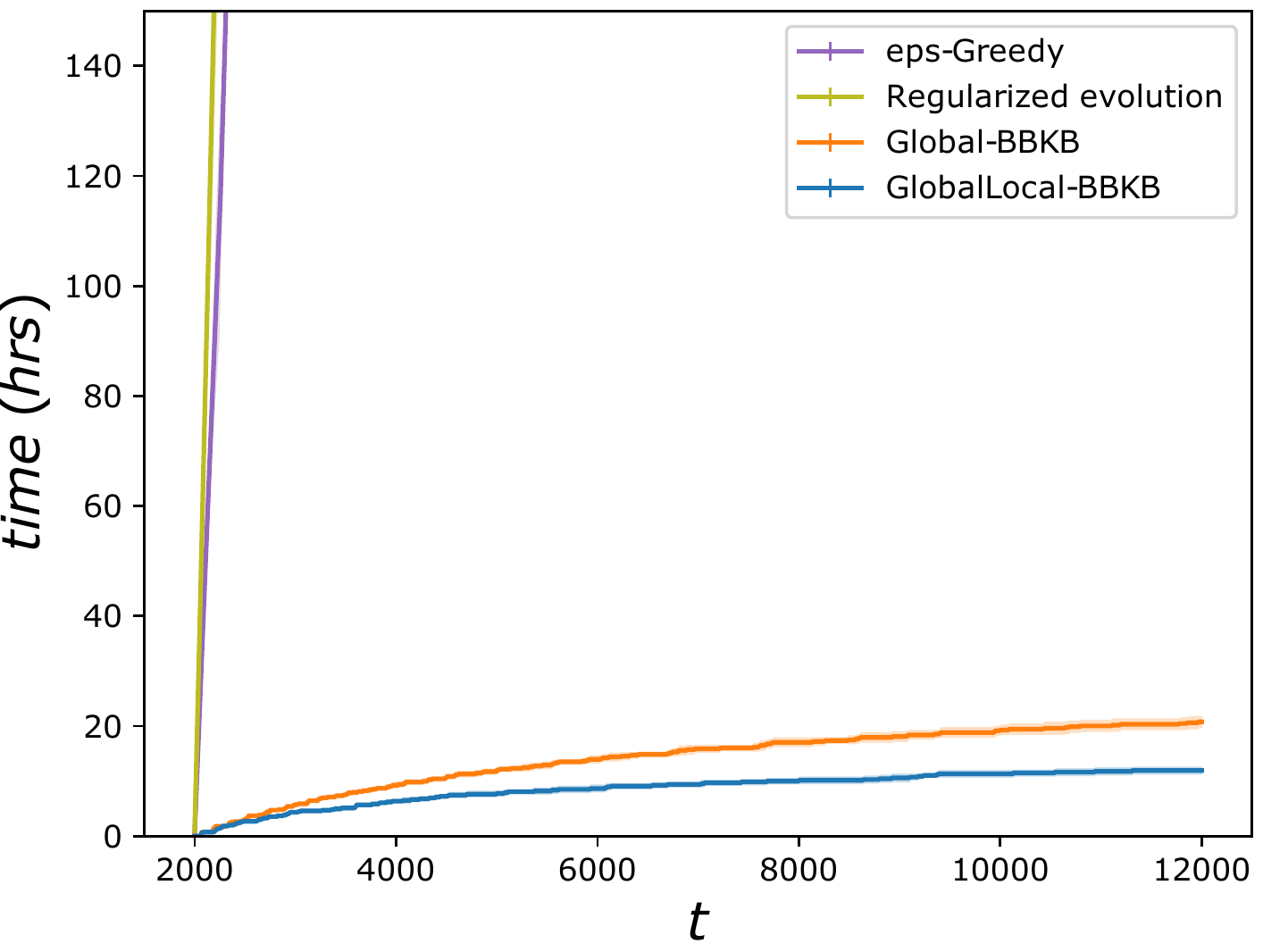}
\endminipage\hfill

\caption{\small From left to right time without experimental costs, batch-size and total runtime on NAS-bench-101}\label{fig:nas-time}
\end{figure*}

\begin{figure*}[!tb]
\minipage{0.33\textwidth}
\includegraphics[width=\linewidth]{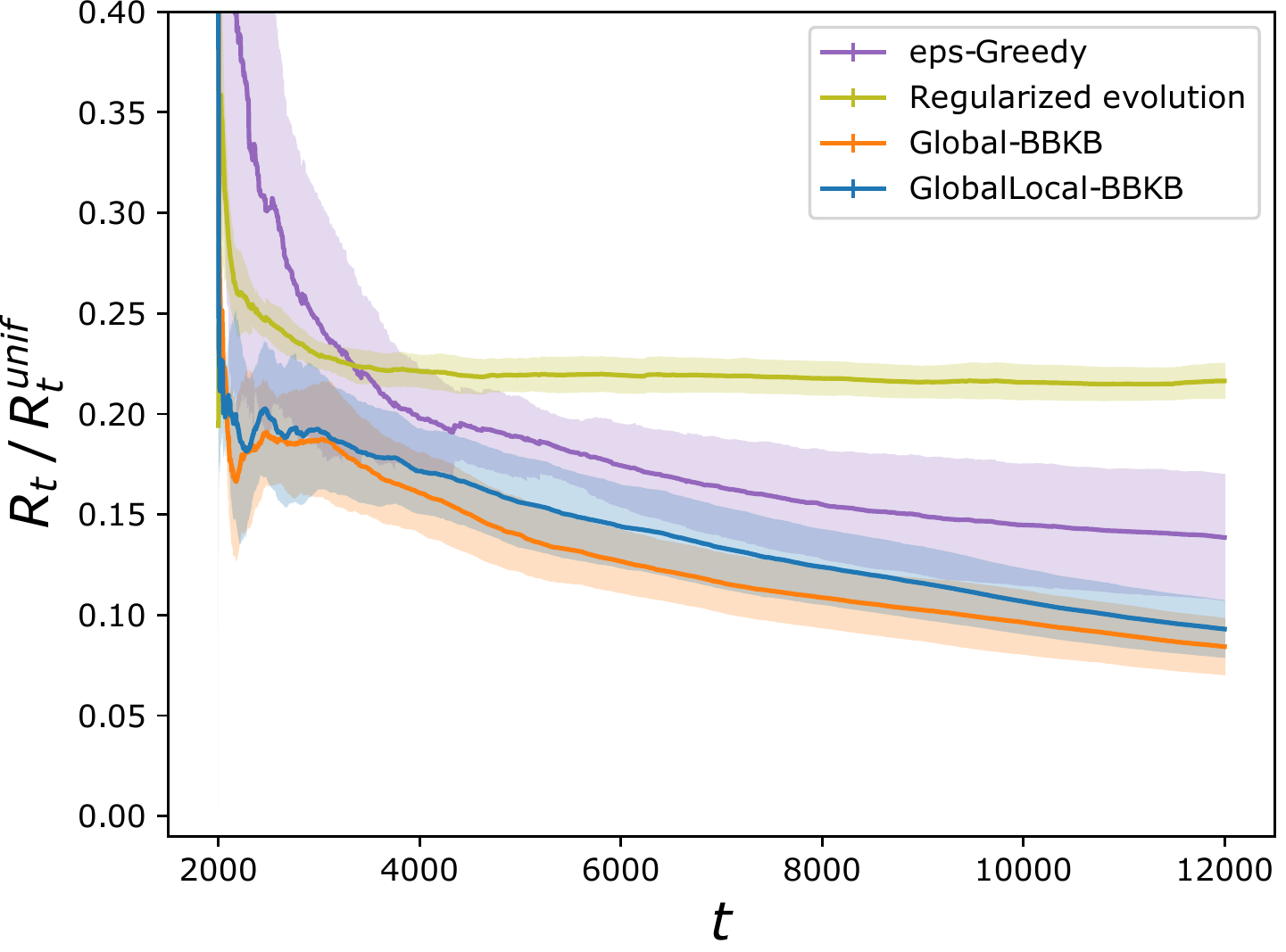}
\endminipage\hfill
\minipage{0.33\textwidth}
\includegraphics[width=\linewidth]{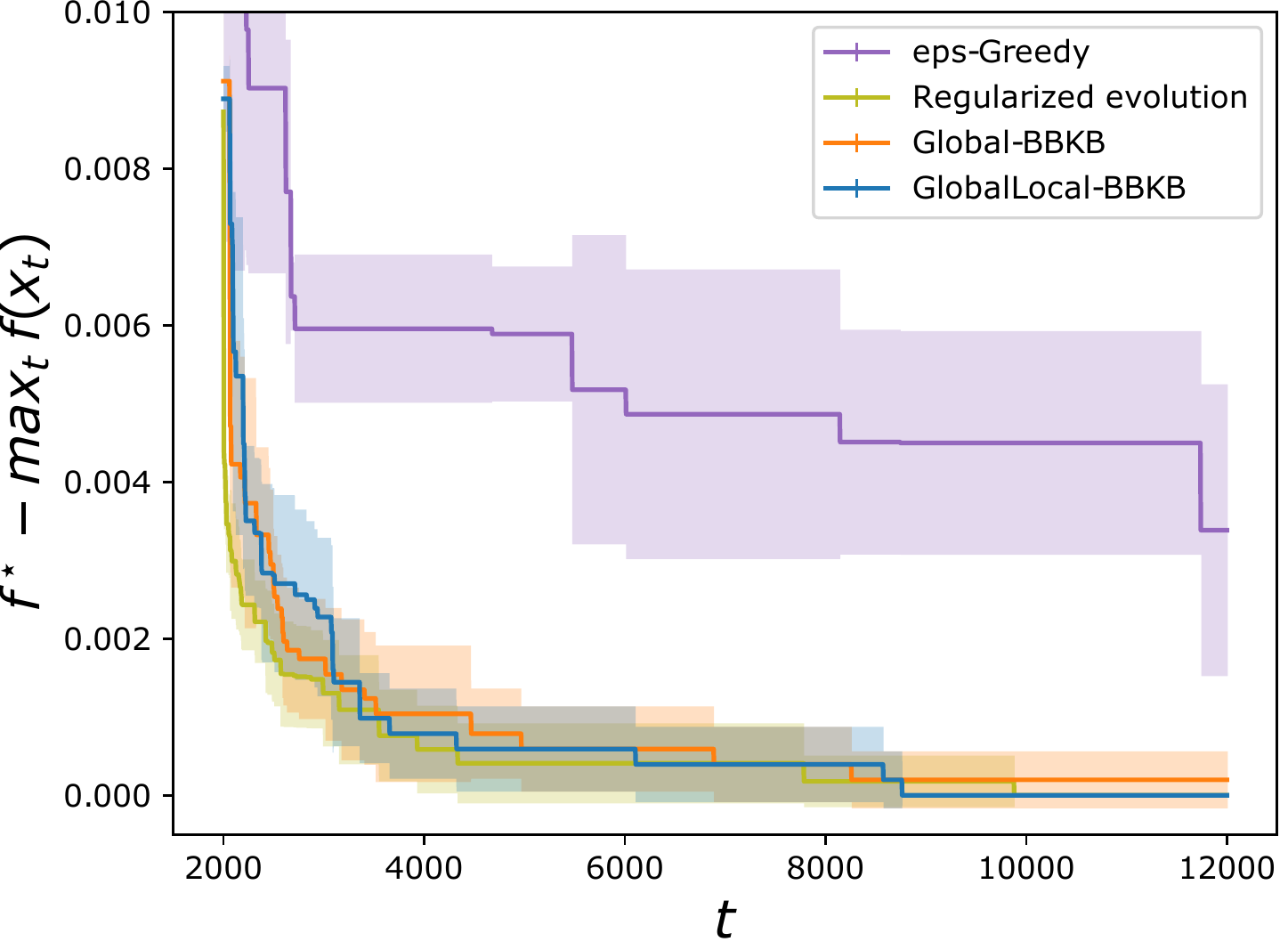}
\endminipage\hfill
\minipage{0.33\textwidth}
\includegraphics[width=\linewidth]{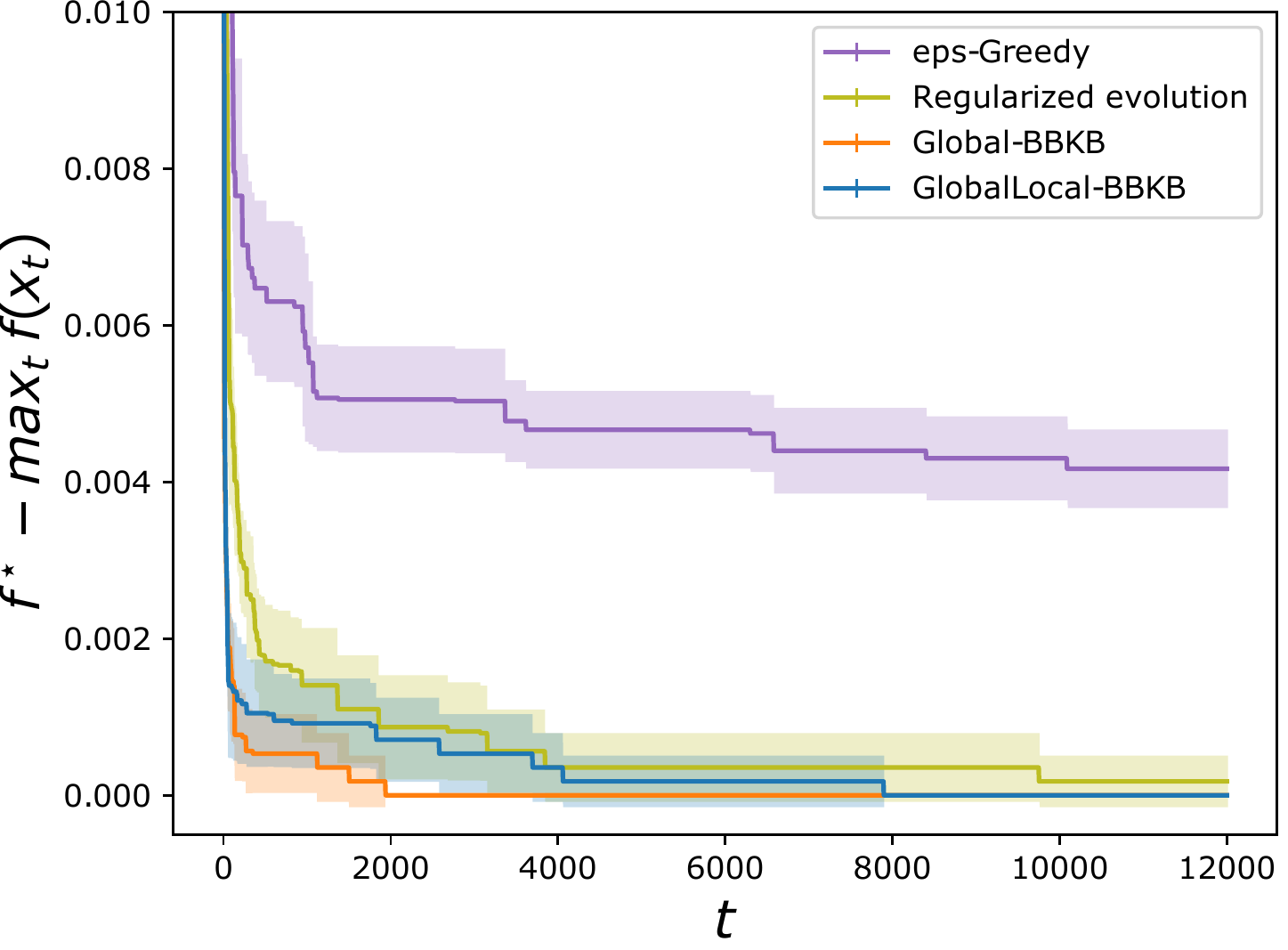}
\endminipage\hfill

\caption{\small From left to right regret-ratio, simple regret and simple regret without initialization on NAS-bench-101}\label{fig:nas-regret}
\end{figure*}

In this section we empirically study the performance in regret and computational costs of \bbkb compared to EpsGreedy, \gpucb, \bgpucb, \bkb, batch Thompson sampling (\bts) \cite{kandasamy2018parallelised} and genetic algorithms (\regev) \cite{real2019regularized}.
For \bbkb we use both the batch stopping rules presented in \Cref{lem:approx-batch-rls-accuracy} and \Cref{lem:posterior-ratio-local-approx} calling the two versions Global-\bbkb and GlobalLocal-\bbkb respectively.
For each experimental result we report mean and 95\% confidence interval using 10 repetitions.
The experiments are implemented in \path{python} using the \path{numpy}, \path{scikit-learn} and \path{botorch} library, and run on a 16 core dual-CPU server using parallelism when allowed by the libraries.
All algorithm use the hyper-parameters suggested by theory. When not applicable, cross validated parameters that perform the best for each individual algorithm are used (e.g. the kernel bandwidth). All the detailed choices and further experiments are reported in the \Cref{sec:app_exp}.

We first perform experiments on two regression datasets Abalone ($A = 4177$, $d = 8$) and Cadata ($A = 20640$, $d = 8$) datasets.
We first rescale the regression target $y$ to $[0,1]$, and construct a noisy function to optimize by artificially adding a gaussian noise with zero mean and standard deviation $\xi = 0.01$.
For a horizon of $T = 10^4$ iterations, we show in
\Cref{fig:rr_ac} 
the ratio between the cumulative {regret} $R_t$ of the desired algorithm and the cumulative regret $R_t^\text{unif}$ achieved by a baseline policy that selects candidates uniformly at random. We use this metric because it is invariant to the scale of the feedback.
We also report in \Cref{fig:t_ac} the {runtime} of the first $2\times 10^3$ iterations.
For both datasets, \bbkb achieves the smallest regret, using only a fraction of the time of the baselines. 
Moreover, note that the time reported do not take into account experimentation costs, as the function is evaluated instantly.

To test how much batching can improve experimental runtime, we then perform experiments on the NAS-bench-101 dataset \cite{ying2019bench}, a neural network architecture search (NAS) dataset. After preprocessing we are left with $A=12416$ candidates in $d=19$ dimensions (details in \Cref{sec:app_exp}).
For each candidate, the dataset contains 3 evaluations of the trained network, which we transform in a noisy function by returning one uniformly at random,
and the time required to train the network.
To simulate a realistic NAS scenario, we assume to start with already $\Tinit = 2000$ evaluated network architectures, selected uniformly at random.
Initializing \bbkb using this data is straightforward. To generate an initial dictionary we use the BLESS algorithm \cite{NIPS2018_7810}, with a time cost of $2.5s$.
\\
In \Cref{fig:nas-time} we first report on the left the runtime of each algorithm without considering experimental costs. While both \bbkb variants outperform baselines,
due to the more expensive ratio estimator GlobalLocal-\bbkb is slower than Global-\bbkb. However, while both termination rules guarantee linearly increasing batch-sizes, we can see that the local rule outperforms the global rule. When taking into account training time, this not only shows that the batched algorithm are faster than sequential \regev, but also that GlobalLocal-\bbkb with its larger batches becomes faster than Global-\bbkb.

In \Cref{fig:nas-regret} we report cumulative and simple regret of \bbkb against \regev, the best algorithm from \cite{ying2019bench}.
To measure the regret, we plot the regret ratio $R_t / R_t^\text{unif}$ and the simple regret (the gap between the best candidate and the best candidate found by the algorithm up to time $t$). We consider the simple regret metric because it is used in the NAS-bench-101 paper to evaluate \regev.
From the plot of the regret ratio $R_t / R_t^\text{unif}$, we can observe how \bbkb is significantly better than \regev as this latter algorithm has not been designed to minimize the cumulative regret. Further, \bbkb is able to match \regev's simple regret (the main target for this latter algorithm).

Finally, in the rightmost plot of
\Cref{fig:nas-regret} we report simple regret when $\Tinit = 0$ (\ie without using initialization). Surprisingly, while the performance of \regev decreases the performance of \bbkb actually increases, outperforming \regev's. This might hint that initialization is not always beneficial in Bayesian optimization. It remains an open question to verify whether this is because the uniformly sampled initialization data makes the GP harder to approximate, or because it promotes an excessive level of exploration by increasing $\beta_t$ but reducing variance only in suboptimal parts of $\armset$.
\\
\\
\noindent {\small {\bf Acknowledgement}\\ \noindent This material is based upon work supported by the Center for Brains, Minds and Machines (CBMM), funded by NSF STC award CCF-1231216, and the Italian Institute of Technology. We gratefully acknowledge the support of NVIDIA Corporation for the donation of the Titan Xp GPUs and the Tesla K40 GPU used for this research. L. R. acknowledges the financial support of the European Research Council (grant SLING 819789), the AFOSR projects FA9550-17-1-0390  and BAA-AFRL-AFOSR-2016-0007 (European Office of Aerospace Research and Development), and the EU H2020-MSCA-RISE project NoMADS - DLV-777826.}

\bibliographystyle{icml2020}

\newpage
\appendix
\onecolumn

\section{Expanded discussion}\label{sec:app-soa}
\subsection{Relationship of $\amu_t$ and $\asigma_t$ with Bayesian GP posteriors.}
To clarify the comparison between \bbkb and existing GP optimization methods,
it is important to clarify the relationship between \bbkb's approximate posterior
(\ie \Cref{eq:nyst-post-gp-mean,eq:nyst-post-gp-cov,eq:nyst-post-gp} introduced in \cite{calandriello_2019_coltgpucb}, which we will call the \bkb approximation)
and the traditional definition of GP posterior
commonly found in the literature (\eg the one found in \cite{rasmussen_gaussian_2006}).

To begin, let us first consider the case
of a perfect dictionary $\coldict_{\text{exact}}$ (\eg $\coldict_{\text{exact}} = \armset$ or
$\coldict_{\text{exact}} = \bX_t$). Then \Cref{eq:nyst-post-gp-mean,eq:nyst-post-gp-cov,eq:nyst-post-gp} can be simplified to
\begin{align}
&\mu_t\left(\bx \condbar \bX_t, \by_t\right) \!=\! \bk_{t}(\bx)^\transp(\bK_{t} + \lambda\bI)^{-1}\by_{t}\label{eq:gpucb-posterior-mean},\\
&k_{t}\left(\bx, \bx' \condbar \bX_t\right) \!=\! \tfrac{1}{\lambda}\left(\kerfunc(\bx, \bx') - \bk_{t}(\bx)^\transp(\bK_{t} + \lambda\bI)^{-1}\bk_{t}(\bx')\right)\label{eq:gpucb-posterior-variance},\\
&\sigma_{t}^2\left(\bx \condbar \bX_t\right) \!=\! k_{t}\left(\bx, \bx \condbar \bX_t\right),
\end{align}
where $\bK_t \in \Real^{t \times t}$ is the kernel matrix with $[\bK_t]_{i,j} = \kerfunc(\bx_i, \bx_j)$ for $\bx_i,\bx_j$ in $\bX_t$, and $\bk_t(\bx) = [\kerfunc(\bx_1, \bx), \dots, \kerfunc(\bx_t, \bx)]^\transp$.
Comparing this with e.g.,~Eq.~2.23 and 2.24 from \citet{rasmussen_gaussian_2006},
which we will call $\mu_t^{\text{bay}}$ and $\sigma_t^{\text{bay}}$,
we see that when $\lambda = \xi^2$ the definition of the posterior mean $\mu_t$ is identical to $\mu_t^{\text{bay}}$
while the posterior variance $\sigma_t^{\text{bay}} = \lambda\sigma_t = \xi^2\sigma_t$
is rescaled by a $\lambda$ factor. This rescaling is not justified
in a Bayesian prior/posterior sense, and therefore $\sigma_t$
is not a posterior in the Bayesian sense.

However note that in the context of GP optimization with a variant of \gpucb
this distinction becomes less relevant. In particular, we are mostly interested
in comparing $\beta_t\sigma_t$ against $\beta_t^{\text{bay}}\sigma_t^{\text{bay}}$
rather than simply $\sigma_t$ to $\sigma_t^{\text{bay}}$. In this case,
looking at \citet{srinivas2010gaussian} or \citet{chowdhury2017kernelized}
shows that when we choose $\lambda = \xi^2$, then $\beta_t^{\text{bay}} = \beta_t/\lambda$,
and therefore $\beta_t\sigma_t = \beta_t^{\text{bay}}\sigma_t^{\text{bay}}$.
As a consequence, when $\lambda = \xi^2$ we can modify \citet{calandriello_2019_coltgpucb}'s notation
to match the standard Bayesian notation, and the difference becomes only
a matter of simplifications. However, frequentist analysis of Kernelized-UCB
algorithms show that sometimes $\lambda \neq \xi^2$ is the optimal choice
\cite{valko2013finite}, and in this case the two views are not so easily reconcilable.
In this paper we chose to err on the side of generality, maintaining $\lambda$
separate from $\xi^2$, but also on the side of familiarity and continue to
refer to $\sigma_t$ as a posterior, with a slight abuse of terminology.

A similar argument can be made for $\amu_t$ and $\asigma_t$ and their
Bayesian counterparts $\amu_t^{\text{bay}}$ and $\asigma_t^{\text{bay}}$,
known as the deterministic training conditional (DTC) \cite{quinonero-candela_approximation_2007}
or projected process \cite{seeger2003fast}.
In particular, we have once again that the approximate posterior means $\amu_t = \amu_t^{\text{bay}}$ coincide,
while the approximate posterior variance $\asigma_t = \lambda\asigma_t^{\text{bay}}$
differ by a $\lambda$ factor. Note however that although \citet{quinonero-candela_approximation_2007}
call the DTC approximation an approximate posterior, they also remark that it
does not correspond to a GP posterior because it is not consistent.
Therefore, regardless of the rescaling $\lambda$, it is improper to
refer to the DTC or the \bkb approximation as posteriors.

We choose to maintain \citet{calandriello_2019_coltgpucb}'s notation because,
as we will see in the rest of the appendix, it makes $\sigma_t$ coincide
with the confidence intervals induced by OFUL \cite{abbasi2011improved}
and with a quantity known in randomized linear algebra as ridge leverage score
\cite{alaoui2014fast}. Since both of these tools are crucial in our derivation,
using a notation based on $\sigma_t^{\text{bay}}$ would require frequent
rescalings by a factor $\lambda$, which although trivial might become tedious
and make the exposition heavier.

\subsection{Historical overview of the \gpucb family}

There are several ways to leverage a GP posterior to
choose useful candidates to evaluate. Here we review those based on the \gpucb algorithm \cite{srinivas2010gaussian}.
All \gpucb-based algorithms rely on the construction of an \textit{acquisition function} $u_t(\cdot): \armset \to \Real$ that acts as an upper confidence bound (UCB)
for the unknown function $f$.
Whenever $u_t(\bx)$ is a valid UCB (\ie $f(\bx) \leq u_t(\bx)$) and it converges to $f(\bx)$
sufficiently fast, then selecting candidates that are optimal \wrt to $u_t$ leads to low regret,
\ie the value of the candidate $\bx_{t+1} = \argmax_{\bx \in \armset}u_t(\bx)$ tends to $\max_{\bx \in \armset}f(\bx)$ as $t$ increases.

\textbf{\gpucb.}  The original \gpucb formulation defines the UCB 
as $u_t(\bx) = \mu_t(\bx) + \beta_t\sigma_t(\bx)$.
An important property of this estimator is that the posterior variance
is strictly non-increasing as more data is collected, \ie $\sigma_{t+1}(\bx) \leq \sigma_{t}(\bx)$,
and therefore $u_t(\bx)$ naturally converges to $\mu_t(\bx)$, which in turn tends to $f(\bx)$.
\citet{srinivas2010gaussian} found an appropriate schedule for $\beta_t$
that guarantees that this happens \whp, and that therefore $u_t$ is a valid UCB at all steps.
However \gpucb is computationally and experimentally slow, as evaluating $u_t(\bx)$
requires $\bigotime(t^2)$ per-step and no parallel experiments are possible.

\textbf{\bkb.} A common approach to improve computational scalability in GPs
is to replace the exact posterior with an approximate sparse GP posterior.
The main advantage of this approximation is that if we use a dictionary $\coldict$
with size $m = |\coldict|$, then we can embed the GP in $\Real^m$ using
\Cref{eq:nyst-post-gp-mean,eq:nyst-post-gp}, and keep updating the posterior
in $\bigotime(m^2)$ time rather than $\bigotime(t^2)$.
However, this can lead to sub-optimal choices and large regret if
the dictionary $\coldict$ is not sufficiently accurate.
This brings about a trade-off between larger and more accurate dictionaries,
or smaller and more efficient ones. Moreover, we are only interested in
approximating the part of the space that we transverse in our optimization process.
Therefore, a dictionary $\coldict_t$ should naturally change over time
to reflect which part of the space $\armset$ is being tested.
\citet{calandriello_2019_coltgpucb} proposed to solve these problems in the \bkb
algorithm by replacing $u_t$ with an approximate version $u_t^{\bkb}(\bx) = \amu_t(\bx,\coldict_t) + \wt{\beta}_t\asigma_t(\bx,\coldict_t)$,
and using a procedure called posterior variance sampling (see \Cref{alg:bbkb}) and \cite{calandriello_2019_coltgpucb} for more details)
to update the dictionary $\coldict_t$ at each step.
They prove that combining these approaches guarantees that ${u}_t^{\bkb}$ is a UCB, and
that \bkb achieves low regret.
However, posterior sampling requires $\bigotime(t)$ per step to update the
dictionary $\coldict_t$ at each iteration, reducing \bkb's computational scalability,
and the algorithm still has poor experimental scalability since
the candidates are selected sequentially.

        \begin{figure}[t]
		\begin{algorithm}[H]
				\begin{algorithmic}
					\REQUIRE{Set of candidates $\armset$, UCB parameters $\{\alpha_t\}_{t=1}^T$}, threshold $C$
					\STATE Initialize $\fb{0} = 0$
					\FOR{$t=\{0, \dots, T-1\}$}
					\STATE Compute $\bu_{t}^{\bgpucb}(\bx) \gets \mu_{\fb{t}}(\bx) + \alpha_t\sigma_{t}(\bx)$
					\STATE Select $\bx_{t+1} \gets \argmax_{\bx \in \armset}\bu_t^{\bgpucb}(\bx)$\;
                    \IF{$\prod_{s = \fb{t}+1}^{t+1}\left(1 + \sigma_{s-1}^2(\bx_s)\right) \leq {C}$}
                    \STATE $\fb{t+1} = \fb{t}$
                    \STATE Update $\wt{\bu}_{t+1}$ with the new $\asigma_{t+1}$
                    \ELSE
                    \STATE $\fb{t+1} = t+1$
                        \STATE Get feedback $\{y_s\}_{s=\fb{t}+1}^{\fb{t+1}}$
                        \STATE Update ${\bu}_{t+1}^{\bgpucb}$ with the new $\mu_{\fb{t+1}}$ and ${\sigma}_{\fb{t+1}}$
                    \ENDIF
					\ENDFOR
				\end{algorithmic}
			\caption{
				\bgpucb\label{alg:bgpucb}}
		\end{algorithm}
        \end{figure}

\textbf{Batch \gpucb.} 
Finally, \bgpucb \cite{desautels_parallelizing_2014} tries to increase \gpucb's
experimental scalability by selecting a batch\footnote{
Since we present only one of the batched \gpucb variants from \citet{desautels_parallelizing_2014}, we refer for simplicity to
it as \bgpucb. Note that the particular variant
with adaptive batching we compare with is called \textsc{GP-AUCB} in~\cite{desautels_parallelizing_2014},
as an adaptive variant of what they refer to as \bgpucb.
} of candidates that
are all evaluated in parallel.
The complete structure of \bgpucb~\citep{desautels_parallelizing_2014} is illustrated in~\Cref{alg:bgpucb}.
\bgpucb exploits the fact that $\sigma_t$ does not depend on the feedback and within a batch it defines the UCB ${u}^{\bgpucb}_t = \mu_{\fb{t}}(\bx) + \alpha_t\sigma_t(\bx)$, where the mean $\mu_{\fb{t}}$ uses only the feedback up the end of the last batch, while $\sigma_t(\bx)$ is updated within each batch depending on the candidates until $t$. The batch is constructed by selecting candidates as $\bx_{t+1} = \arg\max_{\bx \in \armset} {u}^{\bgpucb}_t (\bx)$, then they all are evaluated in parallel, their feedback is received, and $\mu_t$ is updated. Notice that at the beginning of batch the UCB coincides with the one computed by \gpucb, \ie ${u}^{\bgpucb}_{\fb{t}} = u_{\fb{t}}$.
For all steps \textit{within} a batch \bgpucb can be seen as \emph{fantasizing} or \emph{hallucinating} a constant feedback $\mu_{\fb{t}}(\bx_t)$ so that the mean does not change,
while the variances keep \emph{shrinking}, thus promoting diversity in the batch.
However, incorporating fantasized feedback instead of actual feedback causes
\bgpucb's ${u}^{\bgpucb}_t$ criteria to drift away from $u_t$, which
might not make it a valid UCB anymore. \citet{desautels_parallelizing_2014} show that this issue can be managed simply by adjusting \gpucb's parameter $\beta_t$. In fact, it is possible to take the valid h.p.\,\gpucb confidence bound at the beginning of the batch, and \emph{correct} it to hold for each hallucinated step as
\begin{align}\label{eq:corrected.ucb}
f(\bx)
\leq  \mu_{\fb{t}}(\bx) + \beta_{\fb{t}}\sigma_{\fb{t}}(\bx)
\leq  \mu_{\fb{t}}(\bx) + \rho_{\fb{t},t}(\bx)\beta_{\fb{t}}\sigma_t(\bx),
\end{align}
where $\rho_{\fb{t},t}(\bx) \triangleq \frac{\sigma_{\fb{t}}(\bx)}{\sigma_t(\bx)}$ is the posterior variance ratio. By using any $\alpha_t \geq \rho_{\fb{t},t}(\bx)\beta_{\fb{t}}$, we have that ${u}^{\bgpucb}_t$ is a valid UCB. As the length of the batch increases, the ratio $\rho_{\fb{t},t}$ may become larger, and the UCB becomes less and less tight. As a result, \citet{desautels_parallelizing_2014}
introduce an adaptive batch termination condition that ends the batch
at a designer-defined level of drift $C$. Note that when selecting $C = 1$ (\ie enforcing batches of size 1) \bgpucb reduces to the original \gpucb. Instead of checking the ratio $\rho_{\fb{t},t}$ for any possible candidate $\bx$, \citet{desautels_parallelizing_2014} rely on the following result to derive a global condition that can be checked at any step $t$ depending only on the posterior variance computed on the candidates selected within the batch so far.

\begin{proposition}[\citep{desautels_parallelizing_2014}, Prop.\,1]\label{p:std.dev.ratio}
	At any step $t$, let $\sigma_{\fb{t}}$ and $\sigma_t$ be the posterior standard deviation at the end of the previous batch and at the current step. Then for any $\bx\in\armset$ their ratio is bounded as
	\begin{align}\label{eq:std.dev.ratio}
	\rho_{\fb{t},t}(\bx) \triangleq \frac{\sigma_{\fb{t}}(\bx)}{\sigma_t(\bx)}
    \leq \prod_{s = \fb{t}+1}^{t}\left(1 + \sigma_{s-1}^2(\bx_s)\right).
	\end{align}
\end{proposition}

This shows that while the standard deviation may shrink within each batch, the ratio w.r.t.\ the posterior at the beginning of the batch is bounded. Note that \bgpucb continues the construction of the batch while $\prod_{s = \fb{t}+1}^{t}\left(1 + \sigma_{s-1}^2(\bx_s)\right) \leq C$ for some threshold $C$. Therefore, applying \Cref{p:std.dev.ratio}, we have that the ratio $\rho_{\fb{t},t}(\bx) \leq C$ for any $\bx$, and setting $\alpha_t \triangleq C\beta_{\fb{t}}$ guarantees the validity of the UCB. Finally, the choice of $C$ directly translates into an equivalent constant increase in the regret of \bgpucb \wrt \gpucb.

\begin{proposition}[\cite{desautels_parallelizing_2014}, Thm\,2]\label{p:regret}
The regret of \bgpucb is bounded as $R_T^{\bgpucb} \leq C R_T^{\gpucb}$, 
where $R_T^{\gpucb}$ is the original regret of \gpucb (see Thm.~\ref{thm:main-regret} for its explicit formulation).
\end{proposition}

Despite the gain in experimental parallelism and the low regret, \bgpucb still inherits the same
computational bottlenecks as \gpucb (i.e., $O(T^3)$ in time and $O(T^2)$ in
memory).

\subsection{Lazy UCB evaluation.}

 Once a new dictionary is generated at the end of a batch, we compute $\embvec_{\fb{t+1}}(\bx)$, the posterior mean and variance, and the UCB for all candidates in $\armset$. This is an expensive operation but worth the effort, since it is done only once per batch and all subsequent computations within the batch can be done efficiently. As $\bz(\cdot, \coldict_{\fb{t}})$ is frozen, posterior variances can be updated using efficient rank-one updates to compute posterior variances. Furthermore, for a fixed $\coldict_{\fb{t}}$,
$\asigma_{t+1}(\bx,\coldict_{\fb{t}}) \leq \asigma_{t}(\bx,\coldict_{\fb{t}})$,
and since $\amu_{\fb{t}}(\bx,\coldict_{\fb{t}})$ and $\wt{\alpha}_{\fb{t}}$
remain fixed within the batch, the UCBs $\wt{u}_t(\bx, \coldict_{\fb{t}})$ are strictly non-increasing.
Therefore, after selecting $\bx_{t+1}$ we only need to recompute $\wt{u}_{t+1}(\bx_{t+1}, \coldict_{\fb{t}})$
and the UCBs for arms that had $\wt{u}_{t}(\bx_i, \coldict_{\fb{t}}) \geq \wt{u}_{t+1}(\bx_{t+1}, \coldict_{\fb{t}})$
to guarantee that we are still selecting the $\argmax$ correctly.
While this lazy update of UCBs does not improve the worst-case complexity, in practice it may provide important practical speedups.

Crucially, lazy updates require that both dictionary updates and feedback are
delayed during the batch. In particular, even if we do not receive new feedback
simply updating the dictionary changes the embedding,
and in the new representation the mean $\amu_{\fb{t}}$ and variance $\asigma_{\fb{t}}$
can both be larger or smaller, while still remaining valid. Therefore after
each dictionary updates all of our UCB might be potentially the new maximizer,
and we need to trigger a complete recomputation.
Similarly, even if our embedding remains fixed receiving feedback can increase the mean $\amu_{\fb{t}}$
of potentially all candidates, which need to be updated.
While updating all means is a slightly cheaper operation than updating the embeddings,
requiring only vector-vector multiplications rather than matrix-vector multiplications,
it is still an expensive operation that would prevent \bbkb from achieving near-linear runtime.

\subsection{Why freezing both dictionary and feedback}
We remark that \Cref{lem:approx-batch-rls-accuracy} can be immediately
applied to \bgpucb to improve it, while the application to \bkb must
be handled more carefully.
In particular, if $\coldict = \coldict_{\text{exact}}$ then the ratios
$\wt{\rho}_{\fb{t},t}(\bx, \coldict_{\text{exact}}) = {\rho}_{\fb{t},t}(\bx)$
coincide and as we discussed due to Weierstrass's product inequality
\Cref{lem:approx-batch-rls-accuracy} improves
on \Cref{p:std.dev.ratio}, resulting in an improved \bgpucb.
We can also apply \Cref{lem:approx-batch-rls-accuracy} in two different
ways to \bkb. The first naive approach is to try to improve \bkb's 
experimental scalability through batching, 
\ie use a batched UCB $\bx_{t+1} = \argmax_{\bx \in \armset} \amu_{\fb{t}}(\bx, \coldict_t) + \wt{\alpha}_t\asigma_t(\bx, \coldict_t)$.
However, \Cref{lem:approx-batch-rls-accuracy} cannot be used to guarantee
that this is still a valid UCB, as the dictionary changes over time.
The second naive approach is to try to improve \bkb's
computational scalability through dictionary freezing and adaptive resparsification,
\ie use a UCB with fixed dictionary as
$\bx_{t+1} = \argmax_{\bx \in \armset} \amu_{t}(\bx, \coldict_{\fb{t}}) + \wt{\alpha}_{\fb{t}}\asigma_t(\bx, \coldict_{\fb{t}})$).
However \Cref{lem:approx-batch-rls-accuracy} only applies to the posterior variance
$\asigma_t$ and not the posterior mean $\amu_t$, which is much harder to control.
Already \citet{calandriello_2019_coltgpucb} remark that updating the dictionary
at every step is vital to guarantee that we can correctly approximate
the posterior mean. Already after the first step $\fb{t} + 1$ of dictionary freezing
we might be losing crucial information, \eg $\bx_{\fb{t} +1}$ might be the optimal arm
but if the frozen dictionary $\coldict_{\fb{t}}$ is orthogonal to $\bx_{\fb{t} +1}$
we are going to ignore it until the next resparsification.
Therefore, if we suspend the dictionary update for an amount of time
sufficient to improve computational complexity, we might incur an equally
large regret.
Crucially, introducing both batching and dictionary freezing results
in \bbkb's valid UCB, solving both of these problems.

\subsection{Initialization and uncertainty sampling}

We provide here a simplified proof of the result thanks again to
our stopping rule. Let us first consider the exact case. 
Then if $\fb{t}$ is the beginning of a batch and $\fb{t'}$ the beginning
of the successive batch,
then the termination rule guarantees that the sum of the candidates in the batch exceeds
the threshold
$\sum_{s=\fb{t}}^{\fb{t'}} \sigma_{\fb{t}}^2(\bx_s) \geq \wt{C} - 1$.
Therefore, if at time $\fb{t}$ (\ie at the
beginning of a batch) we can guarantee that
$\max_{\bx} \sigma_{\fb{t}}^2(\bx) \leq 1/\minbatchsize$, then
it is easy to see that this implies
\begin{align*}
&\wt{C} - 1 \leq \sum_{s=\fb{t}}^{\fb{t'}} \sigma_{\fb{t}}^2(\bx_s) \leq \sum_{s=\fb{t}}^{\fb{t'}} 1/\minbatchsize \leq (\fb{t'} - \fb{t})/\minbatchsize,
\end{align*}
which implies that $\fb{t'} - \fb{t} \geq \minbatchsize(\wt{C} - 1)$ and the batch
has size at least $\minbatchsize(\wt{C} - 1)$. Similarly, for the actual termination rule
used by \bbkb we have that
$\sum_{s=\fb{t}}^{\fb{t'}} \asigma_{\fb{t}}^2(\bx_s, \coldict_{\fb{t}}) \geq \wt{C} - 1$.
Since each different dictionary might result in slightly different lower
bounds for the batch size, we can use \Cref{lem:bbkb-rls-accuracy} to
bring ourselves back to the exact case
\begin{align*}
&\wt{C} - 1
\leq \sum_{s=\fb{t}}^{\fb{t'}} \asigma_{\fb{t}}^2(\bx_s,\coldict_{\fb{t}})
\leq 3\sum_{s=\fb{t}}^{\fb{t'}} \sigma_{\fb{t}}^2(\bx_s)
\leq 3(\fb{t'} - \fb{t})/\minbatchsize,
\end{align*}
and therefore $\fb{t'} - \fb{t} \geq \minbatchsize(\wt{C} - 1)/3$.
All that is left is to guarantee that 
$\max_{\bx} \sigma_{\fb{t}}^2(\bx) \leq 1/\minbatchsize$. 

\begin{algorithm}[t]
        \begin{algorithmic}
            \REQUIRE{Arm set $\armset$, $\minbatchsize$}
            \ENSURE{Init.\,set $\bX_{\Tinit}$}
            \STATE Select $\bx_1 \gets \argmax_{\bx_i \in \armset}\sigma_{0}(\bx_i)$
            \STATE $\Tinit \gets 1$
            \WHILE{$\sigma_{\Tinit-1}(\bx_{\Tinit}) \geq 1/\minbatchsize$}
            \STATE $\Tinit \gets \Tinit + 1$
            \STATE Select $\bx_{\Tinit} \gets \argmax_{\bx_i \in \armset}\sigma_{\Tinit-1}(\bx_i)$
            \STATE Update all $\sigma_{\Tinit}(\bx)$ with $\bx_{\Tinit}$
            \ENDWHILE
        \end{algorithmic}
    \caption{
        Uncertainty sampling\label{alg:unc-sampling}}
\end{algorithm}

First we remark that since $\sigma_t$ is non-increasing, the batch sizes
(up to small fluctuations due to GP approximation), will also be non-decreasing
over time. We can then use this intuition to see that selecting beforehand
an initial set of $\Tinit$ candidates to force $\max_{\bx} \sigma_{\Tinit}^2(\bx) \leq 1/\minbatchsize$
is sufficient to guarantee the minimum batch size for the whole optimization process.
For this purpose, we can use \citespecific{Lem.~4}{desautels_parallelizing_2014}.
\begin{proposition}[{\citespecific{Lem.~4}{desautels_parallelizing_2014}}]
Given the uncertainty sampling procedure reported in \Cref{alg:unc-sampling},
we have $\max_{\bx \in \armset} \sigma_{\Tinit}^2(\bx) \leq \gamma_{\Tinit}/\Tinit$.
\end{proposition}
Combining this with the bounds on $\gamma_{\Tinit}$ reported in \citet{desautels_parallelizing_2014}
we can guarantee a minimum degree of experimental parallelism for \bbkb.
 
\section{Controlling posterior ratios}\label{sec:proof-ratio-bound}
In this section we collect most results related to providing guarantees
that exact and approximate posteriors remain close during the whole optimization process.

\subsection{Preliminary results}
Several results presented in this appendix are easier to express and prove
using the so-called feature-space view of a GP
\cite{rasmussen_gaussian_2006}. In particular, to every covariance
$\kerfunc(\cdot,\cdot)$ and reproducing kernel Hilbert space $\rkhs$
we can associate a \emph{feature map} $\featmap(\cdot)$ such that
$\kerfunc(\bx_i,\bx_j) = \featmap(\bx_i)^\transp\featmap(\bx_j)$,
and that $\kerfunc(\bx_i,\bx_i) = \featmap(\bx_i)^\transp\featmap(\bx_i)
= \normsmall{\featmap(\bx_i)}^2$. Let $\phimat{t} = [\featmap(\bx_1), \dots, \featmap(\bx_{t})]^\transp$ be the map where each row
corresponds to a row of the matrix $\bX_t$ after the application of $\featmap(\cdot)$.
Finally, given operator $\bA$, we use $\normsmall{\bA}$ to indicate its
$\ell_2$ operator norm, also known as sup norm. For symmetric positive semi-definite
matrices, this corresponds simply to its largest eigenvalue.

Using the feature-space view of a GP we can introduce an important reformulation\footnote{See \Cref{sec:app-soa} for a detailed discussion on the difference between
the standard GP feature-space view and our definition of $\sigma_t$.}
of the posterior variance $\sigma_t^2(\bx_i)$
\begin{align*}
\sigma^2_{t}(\bx_i)
= \phivec{i}^\transp(\phimat{t}^\transp\phimat{t} + \lambda\bI)^{-1}\phivec{i}.
\end{align*}
In particular, this quadratic form
is well known in randomized numerical linear algebra 
as ridge leverage scores (RLS) \cite{alaoui2014fast},
and used extensively in linear sketching algorithms
\cite{woodruff2014sketching}. Therefore, some of the results we will present
now are inspired from this parallel literature.
For example the proof of \Cref{lem:bbkb-rls-accuracy}, restated here for convenience, is based on
concentration results for RLS sampling.
\bbkbaccuracy*

\begin{proof}
We briefly show here that we can apply \citet[Thm.~1]{calandriello_2019_coltgpucb}'s
result from sequential RLS sampling to our batch setting. In particular,
\citespecific{Thm.~1}{calandriello_2019_coltgpucb} gives identical guarantees as
\Cref{lem:bbkb-rls-accuracy}, but only when the dictionary is resparsified
at each step, and we must compensate for the delays.

At a high level, their result shows that given
a so-called $(\varepsilon,\lambda)$-accurate dictionary $\coldict_t$
it is possible to sample a $(\varepsilon,\lambda)$-accurate
dictionary $\coldict_{t+1}$ using the posterior variance estimator $\asigma_{t}(\bx, \coldict_t)$ from \Cref{eq:nyst-post-gp}.
Since all other guarantees directly follow from $(\varepsilon,\lambda)$-accuracy,
we only need to show that the same inductive argument holds if we apply
it on a batch-by-batch basis instead of a step-by-step basis.
To simplify, we will also only consider the case $\varepsilon = 1/2$.
For more details, we refer the reader to the whole proof in \citespecific{App.~C}{calandriello_2019_coltgpucb}.

In particular, consider the state of the algorithm
at the beginning of the first batch, \ie just after initialization ended.
Since the subset $\coldict_{1} = \bX_{1}$ includes all arms pulled so far (\ie $\bx_1$)
it clearly perfectly preserves $\bX_{1}$, and is therefore infinitely accurate and also $(1/2,\lambda)$-accurate.
Note that \citet{calandriello_2019_coltgpucb} make the same reasoning for their
base case.

Thereafter, assume that $\coldict_{\fb{t}}$ is $(1/2,\lambda)$-accurate,
and let $t' > \fb{t}$ be the time step where we resparsify the dictionary
(\ie the beginning of the following batch) such that
$\fb{t'-1} = \fb{t}$ and $\fb{t'} = t'$.
To guarantee that $\coldict_{\fb{t'}}$ is also $(1/2,\lambda)$-accurate
we must guarantee that the probabilities $\wt{p}_{\fb{t'}}$ used to sample
are at least as large as the true posterior $\sigma_{\fb{t'}}^2$
scaled by a factor $24\log(4T/\delta)$,
\ie $\wt{p}_{\fb{t'}} \geq (24\log(4T/\delta))\cdot\sigma_{\fb{t'}}$.
From the inductive assumption we know that 
$\coldict_{\fb{t}}$ is $(1/2,\lambda)$-accurate,
and therefore \Cref{lem:bbkb-rls-accuracy} holds
and $\asigma_{\fb{t}} \geq \sigma_{\fb{t}}/3 \geq \sigma_{\fb{t'}}^2/3$,
since it is a well known property of RLS and $\sigma_t$ that they are non-increasing in $t$ \cite{calandriello_disqueak_2017}.
Adjusting $\qbar$ to match this condition, we guarantee that we are sampling
at least as much as required by \citet{calandriello_2019_coltgpucb},
and therefore achieve the same accuracy guarantees.
\end{proof}

\subsection{Global ratio bound}

Before moving on to \Cref{lem:approx-batch-rls-accuracy} and \Cref{lem:exact-batch-rls-accuracy},
we will first consider exact posterior variances $\sigma_t^2(\bx_i)$, which represent a simpler
case since we do not have to worry about the presence of a dictionary.
The following Lemma will form a blueprint for the derivation of \Cref{lem:approx-batch-rls-accuracy}.
\begin{lemma}\label{lem:ratio-exact-rls}
For any kernel $k$,  set of points $\bX_t$, $\bx_i \in \armset$, and $\fb{t} < t$,
\begin{align*}
\sigma^2_{{t}}(\bx_i)
\leq \sigma^2_{\fb{t}}(\bx_i)
\leq \left(1 + \sum_{s=\fb{t}+1}^{t}\sigma^2_{\fb{t}}(\bx_s)\right)\sigma^2_{t}(\bx_i)
\end{align*}
\end{lemma}
\begin{proof}
Denote with $\bA = \phimat{\fb{t}}^\transp\phimat{\fb{t}} + \lambda\bI$,
and with $\bB = \phimat{[\fb{t}+1,t]}$ the concatenation of only the arms
between rows $\fb{t}+1$ and $t$, \ie in the context of \bbkb $\phimat{[\fb{t}+1,t]}$
contains the arms in the current batch whose feedback has not been received yet.
Then we have $\sigma^2_{\fb{t}}(\bx_i) = \phivec{i}^\transp\bA^{-1}\phivec{i}$ and
\begin{align*}
\sigma^2_{t}(\bx_i)
= \phivec{i}^\transp\left(\phimat{t}^\transp\phimat{t} + \lambda\bI\right)^{-1}\phivec{i}
= \phivec{i}^\transp\left(\bA + \bB^\transp\bB\right)^{-1}\phivec{i},
\end{align*}
We can now collect $\bA$ to obtain
\begin{align*}
\sigma^2_{t}(\bx_i)
&= \phivec{i}^\transp(\bA + \bB^\transp\bB)^{-1}\phivec{i}
=\phivec{i}^\transp\bA^{-1/2}(\bI+ \bA^{-1/2}\bB^\transp\bB\bA^{-1/2})^{-1}\bA^{-1/2}\phivec{i}\\
&\geq \lambda_{\min}\left((\bI+ \bA^{-1/2}\bB^\transp\bB\bA^{-1/2})^{-1}\right) \phivec{i}^\transp\bA^{-1}\phivec{i}\\
&= \lambda_{\min}\left((\bI+ \bA^{-1/2}\bB^\transp\bB\bA^{-1/2})^{-1}\right) \sigma_{\fb{t}}(\bx_i).
\end{align*}
Focusing on the first part
\begin{align*}
\lambda_{\min}\left((\bI+ \bA^{-1/2}\bB^\transp\bB\bA^{-1/2})^{-1}\right)
&= \frac{1}{\lambda_{\max}\left(\bI+ \bA^{-1/2}\bB^\transp\bB\bA^{-1/2}\right)}\\
&= \frac{1}{1 + \lambda_{\max}(\bA^{-1/2}\bB^\transp\bB\bA^{-1/2})}
= \frac{1}{1 + \lambda_{\max}(\bB\bA^{-1}\bB^\transp)}.
\end{align*}
Expanding the definition of $\bB$, and using $\lambda_{\max}(\bB\bA^{-1}\bB^\transp) \leq \Tr(\bB\bA^{-1}\bB^\transp)$
due to the fact that $\bB\bA^{-1}\bB^\transp$ is PSD
we have
\begin{align*}
\lambda_{\max}(\bB\bA^{-1}\bB^\transp)
\leq \Tr(\bB\bA^{-1}\bB^\transp)
= \sum_{j=\fb{t} + 1}^{t}\phivec{j}\bA^{-1}\phivec{j}
= \sum_{j=\fb{t} + 1}^{t}\sigma^2_{\fb{t}}(\bx_j).
\end{align*}
Putting it all together, and inverting the ratio
\begin{align*}
\sigma^2_{\fb{t}}(\bx_i)
\leq \left(1 + \sum_{s=\fb{t}+1}^{t}\sigma^2_{\fb{t}}(\bx_s)\right)\sigma^2_{t}(\bx_i),
\end{align*}
while to obtain the other side we simply observe that $\bA + \bB^\transp\bB \succeq \bA$
since $\bB^\transp\bB \succeq \mathbf{0}$
and therefore $(\bA + \bB^\transp\bB)^{-1} \preceq \bA^{-1}$ and $\sigma_t^2(\bx_i) \leq \sigma^2_{\fb{t}}(\bx_i)$.
\end{proof}

\textbf{Approximate posterior.} We are now ready to prove \Cref{lem:approx-batch-rls-accuracy}, which we
restate here for clarity.

\asigmaevol*

\begin{proof}
Note that our approximate posterior can be similarly formulated in a feature-space view.
Let us denote with $\bP = \phimat{\coldict}^\transp(\phimat{\coldict}\phimat{\coldict}^\transp)^{+}\phimat{\coldict}$
the projection on the arms in the arbitrary dictionary $\coldict$.
Then, referring to {Sec.~4.1} from \citet{calandriello_2019_coltgpucb} for more details, we have
\begin{align*}
\asigma^2_{t}(\bx_i, \coldict)
&= \phivec{i}^\transp(\bP\phimat{t}^\transp\phimat{t}\bP + \lambda\bI)^{-1}\phivec{i}
= \phivec{i}^\transp(\abA + \wt{\bB}^\transp\wt{\bB})^{-1}\phivec{i},
\end{align*}
where we denote with $\abA = \bP\phimat{\fb{t}}^\transp\phimat{\fb{t}}\bP + \lambda\bI$
our approximation of $\bA$ and with $\wt{\bB} = \phimat{[\fb{t}+1,t]}\bP$ our approximation of $\bB$.
Denote $\aphivec{} \triangleq \bP\phivec{}$. With the exact same reasoning as in the proof of \Cref{lem:ratio-exact-rls} we can
derive
\begin{align*}
\asigma^2_{t}(\bx_i, \coldict)
&= \phivec{i}^\transp(\abA + \wt{\bB}^\transp\wt{\bB})^{-1}\phivec{i}
\geq \asigma_{\fb{t}}(\bx_i, \coldict)\lambda_{\min}\left((\bI+ \abA^{-1/2}\wt{\bB}^\transp\wt{\bB}\abA^{-1/2})^{-1}\right)\\
&\geq \asigma_{\fb{t}}(\bx_i,\coldict)/\left(1 + \Tr(\wt{\bB}\abA^{-1}\wt{\bB}^\transp)\right)
\geq \asigma_{\fb{t}}(\bx_i,\coldict)/\left(1 + \sum_{s=\fb{t} + 1}^{t}\aphivec{s}\abA^{-1}\aphivec{s}\right).
\end{align*}
This is still not exactly what we
wanted,
as $ \aphivec{s}\abA^{-1}\aphivec{s} \neq  \phivec{s}\abA^{-1}\phivec{s} = \asigma_{\fb{t}}^2(\bx_s, \coldict)$,
but we can apply the following Lemma, which we will prove later, to connect the two quantities.
\begin{lemma}\label{lem:rls-proj-residual}
Denote with $\bP^{\bot} = \bI - \bP$ the orthogonal projection
on the complement of $\bP$. We have
\begin{align*}
\phivec{s}^\transp\abA^{-1}\phivec{s}
= \aphivec{s}^\transp\abA^{-1}\aphivec{s} + \lambda^{-1}\phivec{s}^\transp\bP^{\bot}\phivec{s}
\geq \aphivec{s}^\transp\abA^{-1}\aphivec{s}
\end{align*}
\end{lemma}
Putting it together and inverting the bound we have
\begin{align*}
\asigma^2_{t}(\bx_i, \coldict)
&
\geq \asigma^2_{\fb{t}}(\bx_i, \coldict)/\left(1 + \sum_{s=\fb{t} + 1}^{t}\aphivec{s}\abA^{-1}\aphivec{s}\right)\\
&
\geq \asigma^2_{\fb{t}}(\bx_i, \coldict)/\left(1 + \sum_{s=\fb{t} + 1}^{t}\phivec{s}\abA^{-1}\phivec{s}\right)\\
&
\geq \asigma^2_{\fb{t}}(\bx_i, \coldict)/\left(1 + \sum_{s=\fb{t} + 1}^{t}\asigma_{\fb{t}}(\bx_s, \coldict)\right).
\end{align*}
\end{proof}
Finally, combining \Cref{lem:bbkb-rls-accuracy,lem:approx-batch-rls-accuracy},
we can prove \Cref{lem:exact-batch-rls-accuracy}, which we now restate.

\begin{restatable}{lemma}{sigmaevol}\label{lem:exact-batch-rls-accuracy}
Under the same conditions as \Cref{lem:bbkb-rls-accuracy,lem:approx-batch-rls-accuracy},
\begin{align*}
{\rho}_{\fb{t},t}(\bx) \leq 3\left(1 + \sum\nolimits_{s=\fb{t}}^{t}\asigma_{\fb{t}}(\bx, \coldict_{\fb{t}})\right).
\end{align*}
\end{restatable}

\begin{proof}
Note that \Cref{lem:approx-batch-rls-accuracy,lem:ratio-exact-rls}
followed a deterministic derivation based only on linear algebra and
therefore held in any case, including the worst.
To prove \Cref{lem:exact-batch-rls-accuracy} we must instead rely
on the high probability event and guarantees from \Cref{lem:bbkb-rls-accuracy},
and therefore this statement holds only for \bbkb run with the correct
$\qbar$ value and using the reported batch termination condition.
The derivation is straightforward
\begin{align*}
\sigma_t^2(\bx)
&\stackrel{(a)}{\geq}
\sigma^2_{\fb{t}}(\bx_i)/\left(1 + \sum_{s=\fb{t} + 1}^{t}\sigma_{\fb{t}}(\bx_s)\right)\\
&\stackrel{(b)}{\geq}
\sigma^2_{\fb{t}}(\bx_i)/\left(1 + 3\sum_{s=\fb{t} + 1}^{t}\asigma_{\fb{t}}(\bx_s, \coldict_{\fb{t}})\right)\\
&\geq
\sigma^2_{\fb{t}}(\bx_i)/\left(3\left(1 + \sum_{s=\fb{t} + 1}^{t}\asigma_{\fb{t}}(\bx_s, \coldict_{\fb{t}})\right)\right)
\stackrel{(c)}{\geq}
\sigma^2_{\fb{t}}(\bx_i)/(3\wt{C}),
\end{align*}
where $(a)$ is due to \Cref{lem:ratio-exact-rls}, $(b)$ is due to \Cref{lem:bbkb-rls-accuracy},
and $(c)$ is due to the fact that by construction each batch is
terminated at a step $t$ where $1 + \sum_{s=\fb{t} + 1}^{t}\asigma_{\fb{t}}(\bx_s, \coldict_{\fb{t}}) \leq \wt{C}$
still holds.
\end{proof}

To conclude the section, we report the proof of \Cref{lem:rls-proj-residual}
\begin{proof}[Proof of \Cref{lem:rls-proj-residual}]
We have
\begin{align*}
\phivec{s}^\transp\abA_{\fb{t}}^{-1}\phivec{s}
&= \phivec{s}^\transp(\aphimat{\fb{t}}\aphimat{\fb{t}}^\transp + \lambda\bI)^{-1}\phivec{s}\\
&= \phivec{s}^\transp(\aphimat{\fb{t}}\aphimat{\fb{t}}^\transp + \lambda\bP + \lambda\bP^{\bot})^{-1}\phivec{s}\\
&\overset{(a)}{=} \phivec{s}^\transp\left((\aphimat{\fb{t}}\aphimat{\fb{t}}^\transp + \lambda\bP)^{-1} + (\lambda\bP^{\bot})^{-1}\right)\phivec{s}\\
&\overset{(b)}{=} \phivec{s}^\transp(\aphimat{\fb{t}}\aphimat{\fb{t}}^\transp + \lambda\bP)^{-1}\phivec{s} + \lambda^{-1}\phivec{s}^\transp\bP^{\bot}\phivec{s}
\end{align*}
where $(a)$ is due to the fact that $\bP^{\bot}$ is complementary
to both $\bP$ and $\aphimat{\fb{t}}$ since $\Ran(\aphimat{\fb{t}}) \subseteq \Ran(\bP)$,
and $(b)$ is due to the fact that $\bP^{\bot}$ is a projection and therefore equal to
its inverse. We focus now on the first term
\begin{align*}
&\phivec{s}^\transp(\aphimat{\fb{t}}\aphimat{\fb{t}}^\transp + \lambda\bP)^{-1}\phivec{s}\\
&\overset{(a)}{=} \phivec{s}^\transp(\bP\phimat{\fb{t}}\phimat{\fb{t}}^\transp\bP + \lambda\bP)^{-1}\phivec{s}\\
&\overset{(b)}{=} \phivec{s}^\transp\bP(\bP\phimat{\fb{t}}\phimat{\fb{t}}^\transp\bP + \lambda\bP)^{-1}\bP\phivec{s}\\
&\overset{(c)}{=} \aphivec{s}^\transp(\aphimat{\fb{t}}\aphimat{\fb{t}}^\transp + \lambda\bP)^{-1}\aphivec{s}\\
&\overset{(d)}{=} \aphivec{s}^\transp(\aphimat{\fb{t}}\aphimat{\fb{t}}^\transp + \lambda\bI)^{-1}\aphivec{s}\\
&\overset{(e)}{=} \aphivec{s}^\transp\abA_{\fb{t}}^{-1}\aphivec{s}
\end{align*}
where $(a)$ is the definition of $\aphimat{\fb{t}}$,
$(b)$ is because we can collect $\bP$ and extract it from the inverse,
$(c)$ is the definition of $\aphivec{s}$, $(d)$ is because $\aphivec{s}$
lies in $\Ran(\bP)$ and therefore placing $\bP$ or $\bI$ in the inverse
is indifferent, and $(e)$ is the definition of $\abA_{\fb{t}}$.
Putting it together
\begin{align*}
\phivec{s}^\transp\abA_{\fb{t}}^{-1}\phivec{s}
= \aphivec{s}^\transp\abA_{\fb{t}}^{-1}\aphivec{s} + \lambda^{-1}\phivec{s}^\transp\bP^{\bot}\phivec{s}
\geq \aphivec{s}^\transp\abA_{\fb{t}}^{-1}\aphivec{s},
\end{align*}
since $\phivec{s}^\transp\bP^{\bot}\phivec{s}$ is a norm and therefore non-negative.
\end{proof}

\subsection{Local-global bound}
We focus first on the exact posterior variance $\sigma_t$ for simplicity.
Given an arbitrary step $t > \fb{t}$,
we once again denote with $\bA = \phimat{\fb{t}}^\transp\phimat{\fb{t}} + \lambda\bI$,
and with $\bB = \phimat{[\fb{t}+1,t]}$ the concatenation of only the arms
between steps $\fb{t}+1$ and $t$.
Using the Woodbury matrix identity we can obtain a different expansion of the posterior
variance
\begin{align*}
\sigma^2_{t}(\bx_i)
&= \phivec{i}^\transp(\bA + \bB^\transp\bB)^{-1}\phivec{i}\\
&=\phivec{i}^\transp\bA^{-1}\phivec{i}
-\phivec{i}^\transp\bA^{-1}\bB^\transp(\bI+ \bB\bA^{-1}\bB^\transp)^{-1}\bB\bA^{-1}\phivec{i}.
\end{align*}
However this quantity is computationally expensive to compute. In particular,
updating the inverse $(\bI+ \bB\bA^{-1}\bB^\transp)^{-1}$ at each step
is time consuming. For this reason, we instead focus on the following lower
bound
\begin{align*}
\sigma^2_{t}(\bx_i)
&\geq \phivec{i}^\transp\bA^{-1}\phivec{i}
-\phivec{i}^\transp\bA^{-1}\bB^\transp\bB\bA^{-1}\phivec{i}\\
&= \phivec{i}^\transp\bA^{-1}\phivec{i}
-\sum_{j=\fb{t}+1}^t\phivec{i}^\transp\bA^{-1}\phivec{j}\phivec{j}^\transp\bA^{-1}\phivec{i}\\
&= \phivec{i}^\transp\bA^{-1}\phivec{i}
-\sum_{j=\fb{t}+1}^t(\phivec{i}^\transp\bA^{-1}\phivec{j})^2
= \sigma^2_{\fb{t}}(\bx_i)
-\sum_{j=\fb{t}+1}^t\kerfunc_{\fb{t}}^2(\bx_i, \bx_j),
\end{align*}
where we used the fact that $\bB\bA^{-1}\bB^\transp \succeq \mathbf{0}$ and therefore $(\bI+ \bB\bA^{-1}\bB^\transp)^{-1} \preceq \bI$,
the fact that $\bB^\transp\bB =\sum_{j=\fb{t}+1}^t\phivec{j}\phivec{j}^\transp$, and the definition
of $\kerfunc_{\fb{t}}^2(\bx_i, \bx_j)$.
After inversion this bound becomes
\begin{align*}
\frac{\sigma^2_{\fb{t}}(\bx_i)}{\sigma^2_{\fb{t}+B}(\bx_i)}
\leq 1 + \frac{\sum_{j=\fb{t}+1}^t\kerfunc_{\fb{t}}^2(\bx_i, \bx_j)}{\sigma^2_{\fb{t}}(\bx_i)}.
\end{align*}
Note also that thanks to Cauchy-Bunyakovsky-Schwarz's inequality we have
\begin{align*}
\kerfunc_{\fb{t}}^2(\bx_i, \bx_j)
= (\phivec{i}^\transp\bA^{-1}\phivec{j})^2
\leq  \phivec{i}^\transp\bA^{-1}\phivec{i}\phivec{j}^\transp\bA^{-1}\phivec{j}
= \sigma^2_{\fb{t}}(\bx_i)\sigma^2_{\fb{t}}(\bx_j).
\end{align*}
and therefore
the local-global bound is tighter of the global bound,
\begin{align*}
1 + \frac{\sum_{j=1}^B\kerfunc_{\fb{t}}^2(\bx_i, \bx_j)}{\sigma^2_{\fb{t}}(\bx_i)}
\leq 1 + \sum_{j=1}^B\sigma^2_{\fb{t}}(\bx_j)
\end{align*}
and falls back to the global bound in the worst case.

\textbf{Approximate posterior:} with the same derivation, but applied to the
approximate posterior $\asigma_t$ we obtain
\begin{align*}
\asigma^2_{t}(\bx_i)
&=\phivec{i}^\transp\abA^{-1}\phivec{i}
-\phivec{i}^\transp\abA^{-1}\wt{\bB}^\transp(\bI+ \wt{\bB}\abA^{-1}\wt{\bB}^\transp)^{-1}\wt{\bB}\abA^{-1}\phivec{i}\\
&\geq \phivec{i}^\transp\abA^{-1}\phivec{i}
-\phivec{i}^\transp\abA^{-1}\wt{\bB}^\transp\wt{\bB}\abA^{-1}\phivec{i},
\end{align*}
where once again we denote with $\abA = \bP\phimat{\fb{t}}^\transp\phimat{\fb{t}}\bP + \lambda\bI$
our approximation of $\bA$ and with $\wt{\bB} = \phimat{[\fb{t}+1,t]}\bP$ our approximation of $\bB$.
Inverting the bound we obtain
\begin{align*}
\frac{\asigma^2_{\fb{t}}(\bx_i)}{\asigma^2_{\fb{t}+B}(\bx_i)}
&\leq 1 + 
\frac{\sum_{j=\fb{t}+1}^t(\phivec{i}^\transp\abA^{-1}\aphivec{j})^2}{\asigma^2_{\fb{t}}(\bx_i)}.
    \end{align*}
Note that this is different from having $\sum_{j=\fb{t}+1}^t\wt{\kerfunc}_{\fb{t}}^2(\bx_i, \bx_j)$,
since
\begin{align*}
\wt{\kerfunc}_{\fb{t}}^2(\bx_i, \bx_j) = \phivec{i}^\transp\abA^{-1}\phivec{j} \neq \phivec{i}^\transp\abA^{-1}\aphivec{j}.
\end{align*}
Computing this upper bound requires only $\bigotime(\deff^3)$ per step to pre-compute
$\abA^{-1}\aphivec{i}$, and $\bigotime(\Narm\deff)$ time to finish the computation of $\phivec{i}^\transp\abA^{-1}\aphivec{j}$.
However, a new problem arises. With this new stopping criterion
we terminate the batch when
\begin{align*}
\max_i 
1 + 
\frac{\sum_{j=\fb{t}+1}^t(\phivec{i}^\transp\abA^{-1}\aphivec{j})^2}{\asigma^2_{\fb{t}}(\bx_i)}
= \wt{C},
\end{align*}
which guarantees 
$\frac{\asigma^2_{\fb{t}}(\bx_i)}{\asigma^2_{t}(\bx_i)} \leq \wt{C}$.
However this stopping condition cannot give us a similar bound on
$\frac{\sigma^2_{\fb{t}}(\bx_i)}{\sigma^2_{t}(\bx_i)}$.
Note that in
the previous global bound we used the fact that $\bA$ and $\abA$ are close
(\ie \Cref{thm:bbkb-complexity}) to bound $\frac{\sigma^2_{\fb{t}}(\bx_i)}{\sigma^2_{t}(\bx_i)}$. In this local-global bound
the approximate posterior is computed using $\abA + \wt{\bB}^\transp\wt{\bB}$
instead of $\bA + {\bB}^\transp{\bB}$,
and while $\abA$ and $\bA$ are close, nothing can be said on $\wt{\bB}^\transp\wt{\bB}$
and ${\bB}^\transp{\bB}$ because freezing the projection $\bP_{\fb{t}}$
results in a loss of guarantees.

To compensate, when moving to the approximate setting we will use a slightly
different terminating criterion for the batch. In particular we will
terminate based on a worst case between both possibilities
\begin{align*}
\frac{\asigma^2_{\fb{t}}(\bx_i)}{\min\{\phivec{i}^\transp(\abA + \wt{\bB}^\transp\wt{\bB})^{-1}\phivec{i} , \phivec{i}^\transp(\abA + \bB^\transp\bB)^{-1}\phivec{i}\}}
\end{align*}
It is easy to see that this termination rule
is more conservative than the normal ratio as
\begin{align*}
\frac{\asigma^2_{\fb{t}}(\bx_i)}{\min\{\phivec{i}^\transp(\abA + \wt{\bB}^\transp\wt{\bB})^{-1}\phivec{i} , \phivec{i}^\transp(\abA + \bB^\transp\bB)^{-1}\phivec{i}\}}
\geq
\frac{\asigma^2_{\fb{t}}(\bx_i)}{\phivec{i}^\transp(\abA + \wt{\bB}^\transp\wt{\bB})^{-1}\phivec{i}}
 =\frac{\asigma^2_{\fb{t}}(\bx_i)}{\asigma^2_{t}(\bx_i)}.
\end{align*}
Moreover, using the guarantees on $\asigma_t$ and $\sigma_t$ from \Cref{lem:exact-batch-rls-accuracy}, we also have
\begin{align*}
& \frac{\asigma^2_{\fb{t}}(\bx_i)}{\min\{\phivec{i}^\transp(\abA + \wt{\bB}^\transp\wt{\bB})^{-1}\phivec{i} , \phivec{i}^\transp(\abA + \bB^\transp\bB)^{-1}\phivec{i}\}}
\geq \frac{\asigma^2_{\fb{t}}(\bx_i)}{\phivec{i}^\transp(\abA + \bB^\transp\bB)^{-1}\phivec{i}}\\
&\geq \epsfrac^{-1}\frac{\asigma^2_{\fb{t}}(\bx_i)}{\phivec{i}^\transp(\bA + \bB^\transp\bB)^{-1}\phivec{i}}
\geq \epsfrac^{-1}\frac{\asigma^2_{\fb{t}}(\bx_i)}{\sigma^2_{t}(\bx_i)}
\geq \epsfrac^{-2}\frac{\sigma^2_{\fb{t}}(\bx_i)}{\sigma^2_{t}(\bx_i)}.
\end{align*}
Therefore, to provide guarantees on both exact and approximate ratios it is sufficient
to choose a stopping condition such that
\begin{align*}
&\frac{\asigma^2_{\fb{t}}(\bx_i)}{\min\{\phivec{i}^\transp(\abA + \wt{\bB}^\transp\wt{\bB})^{-1}\phivec{i} , \phivec{i}^\transp(\abA + \bB^\transp\bB)^{-1}\phivec{i}\}}\\
&=\max\left\{
\frac{\asigma^2_{\fb{t}}(\bx_i)}{\phivec{i}^\transp(\abA + \wt{\bB}^\transp\wt{\bB})^{-1}\phivec{i}},
\frac{\asigma^2_{\fb{t}}(\bx_i)}{\phivec{i}^\transp(\abA + \bB^\transp\bB)^{-1}\phivec{i}}\right\}\\
&\leq \max\left\{
1 + \frac{\phivec{i}^\transp\abA^{-1}\wt{\bB}^\transp\wt{\bB}\abA^{-1}\phivec{i}}{\asigma^2_{\fb{t}}(\bx_i)},
1 + \frac{\phivec{i}^\transp\abA^{-1}\bB^\transp\bB\abA^{-1}\phivec{i}}{\asigma^2_{\fb{t}}(\bx_i)}\right\}\\
&= \max\left\{
1 + \frac{\sum_{j=\fb{t}+1}^t(\phivec{i}^\transp\abA^{-1}\aphivec{j})^2 }{\asigma^2_{\fb{t}}(\bx_i)},
1 + \frac{\sum_{j=\fb{t}+1}^t(\phivec{i}^\transp\abA^{-1}\phivec{j})^2}{\asigma^2_{\fb{t}}(\bx_i)}
\right\}.
\end{align*}
Finally, note that while both $\frac{\asigma^2_{\fb{t}}(\bx_i)}{\phivec{i}^\transp(\abA + \wt{\bB}^\transp\wt{\bB})^{-1}\phivec{i}}$
and $\frac{\asigma^2_{\fb{t}}(\bx_i)}{\phivec{i}^\transp(\abA + \bB^\transp\bB)^{-1}\phivec{i}}$ could be the larger
element in the $\max$ (\ie one does not dominate the other),
after we upper bound $\frac{\asigma^2_{\fb{t}}(\bx_i)}{\phivec{i}^\transp(\abA + \wt{\bB}^\transp\wt{\bB})^{-1}\phivec{i}}
\leq 1 + \frac{\phivec{i}^\transp\abA^{-1}\wt{\bB}^\transp\wt{\bB}\abA^{-1}\phivec{i}}{\asigma^2_{\fb{t}}(\bx_i)}$
and $\frac{\asigma^2_{\fb{t}}(\bx_i)}{\phivec{i}^\transp(\abA + \bB^\transp\bB)^{-1}\phivec{i}}
\leq 1 + \frac{\phivec{i}^\transp\abA^{-1}\bB^\transp\bB\abA^{-1}\phivec{i}}{\asigma^2_{\fb{t}}(\bx_i)}$
we do have that one dominates the other. In other words, one of the two bounding operations
is looser. In particular, 
let us denote with $\wb{\bA}$ the operator such that
\begin{align*}
\wb{\bA} &= \bP\phimat{\fb{t}}^\transp\phimat{\fb{t}}\bP + \lambda\bP\\
\abA &= 
\bP\phimat{\fb{t}}^\transp\phimat{\fb{t}}\bP + \lambda\bI
= \bP\phimat{\fb{t}}^\transp\phimat{\fb{t}}\bP + \lambda\bP + \lambda\bP^{\bot}
= \wb{\bA} + \lambda\bP^{\bot}.
\end{align*}
Moreover, note that $\abA^{-1} = \wb{\bA}^{-1} + \bP^{\bot}/\lambda$, and that
$\bB\wb{\bA}^{-1}\bB^\transp = \wt{\bB}\wb{\bA}^{-1}\wt{\bB}^\transp$.
Then
\begin{align*}
\phivec{i}^\transp\abA^{-1}\aphivec{j}
= \phivec{i}^\transp\wb{\bA}^{-1}\aphivec{j} + \phivec{i}^\transp\bP^{\bot}\aphivec{j}/\lambda
= \aphivec{i}^\transp\wb{\bA}^{-1}\aphivec{j} + 0,
\end{align*}
while
\begin{align*}
\phivec{i}^\transp\abA^{-1}\phivec{j}
&= \phivec{i}^\transp\wb{\bA}^{-1}\phivec{j} + \phivec{i}^\transp\bP^{\bot}\phivec{j}/\lambda\\
&= \aphivec{i}^\transp\wb{\bA}^{-1}\aphivec{j} + \phivec{i}^\transp\bP^{\bot}\phivec{j}/\lambda\\
&= \aphivec{i}^\transp\wb{\bA}^{-1}\aphivec{j} + (\phivec{i}^\transp\phivec{j} - \aphivec{i}^\transp\aphivec{j})/\lambda
\geq \phivec{i}^\transp\abA^{-1}\aphivec{j},
\end{align*}
and therefore $\phivec{i}^\transp\abA^{-1}\phivec{j}$ dominates $\phivec{i}^\transp\abA^{-1}\aphivec{j}$.
Putting all together we obtain that the local-global bound terminates a batch
when
\begin{align*}
1 + \frac{\sum_{j=\fb{t}+1}^t(\phivec{i}^\transp\abA^{-1}\phivec{j})^2}{\asigma^2_{\fb{t}}(\bx_i)}
= 1 + \frac{\sum_{j=\fb{t}+1}^t(\wt{\kerfunc}_{\fb{t}}(\bx_i,\bx_j))^2}{\asigma^2_{\fb{t}}(\bx_i)}
\leq \wt{C},
\end{align*}
which, w.h.p.~as in \Cref{lem:exact-batch-rls-accuracy}, gives us the guarantee
both that $\frac{\asigma^2_{\fb{t}}(\bx_i)}{\asigma^2_{t}(\bx_i)} \leq \wt{C}$ and
$\frac{\sigma^2_{\fb{t}}(\bx_i)}{\sigma^2_{t}(\bx_i)} \leq \epsfrac^2\wt{C}$.

Note that since $\wb{\bA}$ is completely contained in $\Ran(\bP)$, we can
compute $\aphivec{i}^\transp\wb{\bA}^{-1}\aphivec{j}$ directly using the embedded
arms. In particular, we only need to store the pre-computed $\wb{\bA}^{-1}$
at the beginning of the batch, apply it to $\aphivec{j}$ in $\bigotime(\deff^2)$
time, and then apply $\wb{\bA}^{-1}\aphivec{j}$ to each $\aphivec{i}$
in $\bigotime(\Narm\deff)$ time. Similarly, computing $(\phivec{i}^\transp\phivec{j} - \aphivec{i}^\transp\aphivec{j})/\lambda$
can be done in $\bigotime(\Narm\deff)$ time.
 \section{Proofs from \Cref{sec:bbkb}}\label{sec:proof-regret}
\subsection{Complexity analysis (proof of \Cref{thm:bbkb-complexity})}
We restate \Cref{thm:bbkb-complexity} for completeness.
\bbkbcomplexity*

\begin{proof} The proof will be divided in three parts, one for each of the statements.

\textbf{Bounding $|\coldict_t|$.} The first part of the result concerns space guarantees for $\coldict_t$.
Like in the proof of \Cref{lem:bbkb-rls-accuracy}, we simply need to
show that the conditions outlined in \citespecific{Thm.~1}{calandriello_2019_coltgpucb}
are satisfied.
Let us consider
again a step $t' > \fb{t}$ where we perform a resparsification
(\ie be the beginning of the following batch) such that
$\fb{t'-1} = \fb{t}$ and $\fb{t'} = t'$.
Conversely from \Cref{lem:bbkb-rls-accuracy},
where we had to show that our inclusion probabilities $\wt{p}_{\fb{t}}$ were
not much smaller than $\sigma_{\fb{t'}}^2$,
here we have to show that they are not much larger than
$\sigma_{\fb{t'}}^2$.
This is because our goal is to sample $\coldict_{\fb{t'}}$
according to $\sigma_{\fb{t'}}^2$, and if our sampling
probabilities $\wt{p}_{\fb{t}} \propto \asigma_{\fb{t}} \propto \sigma_{\fb{t}}$
are much larger than necessary we are going to wastefully include a number
of points larger than necessary. Since \bbkb gets computationally heavy
if the dictionary gets too large, we want to prove that this does not happen
w.h.p.

We begin by invoking \Cref{lem:bbkb-rls-accuracy} to bound
$\asigma_{\fb{t}} \leq 3\sigma_{\fb{t}}$. The second step is to split the
quantity of interest in two parts: one from $\fb{t}$ until the end of the
batch $\fb{t'}-1$, and the crucial step from $\fb{t'} - 1$ to $\fb{t}$
\begin{align*}
\asigma_{\fb{t}}^2(\bx, \coldict_{\fb{t}})
\leq 3\sigma_{\fb{t}}^2(\bx)
= 3\overbracket{\frac{\sigma_{\fb{t}}^2(\bx)}{\sigma_{\fb{t'}-1}^2(\bx)}}^{(a)}
\overbracket{\frac{\sigma_{\fb{t'}-1}^2(\bx)}{\sigma_{\fb{t'}}^2(\bx)}}^{(b)}\sigma_{\fb{t'}}^2(\bx).
\end{align*}
Since $\fb{t}$ and $\fb{t'}-1$ are both in the same batch, we can use
\bbkb's batch termination condition and \Cref{lem:exact-batch-rls-accuracy} to bound $(a)$
as $\sigma_{\fb{t}}^2(\bx)/\sigma_{\fb{t'}-1}^2(\bx) \leq 3\wt{C}$.
However, $(b)$ crosses the batch boundaries and does not satisfy the terminating
condition. Instead, we will re-use the worst-case guarantees of
\Cref{lem:ratio-exact-rls} to bound the single step increase as
\begin{align*}
\sigma_{\fb{t'}-1}^2(\bx)/\sigma_{\fb{t'}}^2(\bx)
\leq (1 + \sigma_{\fb{t'}}^2(\bx))\sigma_{\fb{t'}}^2(\bx)
\leq (1 + \kappa^2/\lambda)\sigma_{\fb{t'}}^2(\bx),
\end{align*}
where we used the fact that the posterior variance can never exceed
$\kappa^2/\lambda$, as can be easily derived from the definition.
Putting it all together we have
\begin{align}\label{eq:whole-batch-evol-bound}
\asigma_{\fb{t}}^2(\bx, \coldict_{\fb{t}})
\leq 3\sigma_{\fb{t}}^2(\bx)
\leq 3\cdot3\wt{C}\cdot(1+\kappa^2/\lambda)\cdot\sigma_{\fb{t'}}^2(\bx)
= 9\wt{C}(1+\kappa^2/\lambda)\sigma_{\fb{t'}}^2(\bx),
\end{align}
and our overestimate error constant is $9\wt{C}(1+\kappa^2/\lambda)$,
which when plugged into \citespecific{Thm.~1}{calandriello_2019_coltgpucb} gives us
\begin{align*}
|\coldict_t|
\leq 9\wt{C}(1+\kappa^2/\lambda) \cdot 9(1 + \kappa^2/\lambda)\qbar\deff(\bX_t)
= 81\wt{C}(1+\kappa^2/\lambda)^2\qbar\deff(\bX_t).
\end{align*}

\textbf{Bounding the total number of resparsifications.} The most expensive
operation that \bbkb can perform is the GP resparsification, and to guarantee
low amortized runtime we now prove
that we do not do it too frequently. For this, we will leverage the terminating
condition of each batch, since a resparsification is triggered only at the end of
each batch.

In particular, we know that if \bbkb resparsifies at step $t$, such that $\fb{t} = t$.
Then we have that, to not have triggered a resparsification, up to step $t-1$ we have $1 + \sum_{s=\fb{t-1} + 1}^{t-1} \asigma_{\fb{t-1}}^2(\bx_s, \coldict_{\fb{t-1}}) \leq \wt{C}$, while we have the opposite
inequality
$\wt{C} < 1 + \sum_{s=\fb{t-1} + 1}^{\fb{t}} \asigma_{\fb{t-1}}^2(\bx_s, \coldict_{\fb{t-1}})$
if we include the last term $\asigma_{\fb{t-1}}^2(\bx_{\fb{t}}, \coldict_{\fb{t-1}})$.
Moreover, we have one of such inequalities
for each batch in the optimization process. Indicating the number of batches
with $B$, and summing over all the inequalities
\begin{align*}
B\wt{C} \leq B + \sum_{t=1}^{T} \asigma_{\fb{t}}^2(\bx_t, \coldict_{\fb{t}})\indfunc\{t \neq \fb{t}\} + \asigma_{\fb{t-1}}^2(\bx_t, \coldict_{\fb{t-1}})\indfunc\{t = \fb{t}\},
\end{align*}
where we have used the indicator function $\indfunc\{\cdot\}$ to differentiate between batch construction steps and resparsification steps
since at the resparsification step we are still using the posterior only \wrt the previous choiches
$\fb{t-1}$, and more importantly the old dictionary $\coldict_{\fb{t-1}}$, since
the resparsification happens only after the check. However, the only thing
that matters to be able to apply \Cref{lem:bbkb-rls-accuracy} is that
the subscript of the posterior $\asigma_{\fb{t}}$ and of the dictionary $\coldict_{\fb{t}}$ coincide,
so we can further upper bound
\begin{align*}
B\wt{C} \leq B + 3\sum_{t=1}^{T} \sigma_{\fb{t}}^2(\bx_t)\indfunc\{t \neq \fb{t}\} + \sigma_{\fb{t-1}}^2(\bx_t)\indfunc\{t = \fb{t}\}.
\end{align*}
Finally, we again exploit the bound
$\sigma_{\fb{t-1}}^2(\bx_t) \leq 3\wt{C}(1 + \kappa^2/\lambda)\sigma_{t}^2(\bx_t)$,
we derived in \Cref{eq:whole-batch-evol-bound}
for the evolution of RLS across a whole batch to bound the elements in the summation
where $t = \fb{t}$, and apply \Cref{lem:exact-batch-rls-accuracy} to the elements
where $t \neq \fb{t}$. We obtain
\begin{align*}
B\wt{C}
&
\leq B + 3\sum_{t=1}^{T} \sigma_{\fb{t}}^2(\bx_t)\indfunc\{t \neq \fb{t}\} + \sigma_{\fb{t-1}}^2(\bx_t)\indfunc\{t = \fb{t}\}\\
&
\leq B + 3\sum_{t=1}^{T} 3\wt{C}\sigma_{t}^2(\bx_t)\indfunc\{t \neq \fb{t}\} + 3\wt{C}(1 + \kappa^2/\lambda)\sigma_{t}^2(\bx_t)\indfunc\{t = \fb{t}\}\\
&
\leq B +  9\wt{C}(1 + \kappa^2/\lambda)\sum_{t=1}^{T} \sigma_{t}^2(\bx_t).
\end{align*}
Reshuffling terms and normalizing we obtain
\begin{align*}
B\leq \frac{\wt{C}}{\wt{C}-1}9(1 + \kappa^2/\lambda)\sum_{t=1}^{T} \sigma_{t}^2(\bx_t),
\end{align*}
and using the fact that $\sum_{t=1}^{T} \sigma_{t}^2(\bx_t) \leq \logdet(\bK_T/\lambda + \bI) \leq \bigotime(\deff(\bX_T)\polylog(T)) = \abigotime(\deff(\bX_T))$
from \citespecific{Lem.~3}{calandriello_2017_icmlskons}, we obtain our result.

\textbf{Complexity analysis.} Now that we have a bound on the size of the
dictionary, and on the frequency of the resparsifications, we only need to
quantify how much each operation costs and amortize it over $T$ iterations.

The resparsification steps are more computationally intensive.
Resampling the new $\coldict_{\fb{t+1}}$ takes $\bigotime(\min\{\Narm, t\})$,
as we reuse the variances computed at the beginning of the batch.
Given the new embedding function $\embfunc_{\fb{t+1}}(\cdot)$, we must first
recomputing the embeddings for all arms in $\bigotime(\Narm m_t^2 + m_t^3)$,
and then update all variances in $\bigotime(\Narm m_t^2 + m_t^3)$.
Finally, updating the means takes $\bigotime(t m_t + m_t^3)$
time. Overall, a resparsification step requires $\bigotime(\Narm m_t^2 + m_t^3 + t m_t)$,
since in all cases of interest $m_t \leq \deff \ll \Narm$.

In each non-resparsification step, updating the variances requires 
$\bigotime(m_t^2)$ to update the inverse of $\bV_t$ and $\bigotime(m_t^2)$
for each $\asigma_t(\bx_i)$ updated. While the updated actions can be as large
as $\bigotime(\Narm)$, lazy evaluations usually require to update
just a few entries of $\wt{\bu}_t$. 

Using again $B$ to indicate the number of batches during the optimization,
\ie the number of resparsifications, the overall complexity of the algorithm
is thus $\bigotime(\sum_{t=1}^T \Narm m_t^2 + \max_{t=1}^T B(\Narm m_t^2 + m_t^3 + t m_t))$.
Using the dictionary size guarantees of \Cref{thm:bbkb-complexity} we can further
upper bound this to $\abigotime(B(\Narm \deff^2 + \deff^3 + T\deff) + T\Narm \deff^2)$,
and using the bound on resparsifications that we just derived we obtain the final complexity
$\abigotime(T \Narm \deff^2 +\deff^4)$ where we used the fact that $\deff \leq \abigotime(T)$.
\end{proof}

\subsection{Regret analysis (proof of \Cref{thm:main-regret})}
We will leverage the following result from \cite{calandriello_2019_coltgpucb}.
This is a direct rewriting of their statement with two small modifications.
First we express the statement in terms of confidence intervals
on the function $f(\bx)$ rather than in their feature-space view of the GPs.
Second, we do not upper bound $\logdet(\bK_t/\lambda + \bI) \leq
\bigotime(\log(t)\sum_{s=1}^t\asigma_t^2(\bx_s))$. \citet{calandriello_2019_coltgpucb}
use this upper bound for computational reasons, but as we will see we can
obtain a tighter (\ie without the $\log(t)$ factor) alternative bound that is still efficient to compute.
\begin{proposition}[{\citespecific{App.~D, Thm.~9}{calandriello_2019_coltgpucb}}]\label{prop:bkb-confidence-interval}
Under the same assumptions of \Cref{thm:main-regret},
with probability at least $1 - \delta$ and for all $\bx_i \in \armset$ and $\fb{t} \geq 1$
\begin{align*}
\wt{\mu}_{\fb{t}}(\bx_i, \coldict_{\fb{t}}) - \beta_{\fb{t}}\asigma_{\fb{t}}(\bx_i, \coldict_{\fb{t}})
\leq f(\bx_i) \leq \wt{\mu}_{\fb{t}}(\bx_i, \coldict_{\fb{t}}) + \beta_{\fb{t}}\asigma_{\fb{t}}(\bx_i, \coldict_{\fb{t}})
\end{align*}
with
\begin{align*}
\beta_{\fb{t}} &\eqdef
2\xi\sqrt{\logdet(\bK_{\fb{t}}/\lambda + \bI) + \log\left({1}/{\delta}\right)}
+ \left(1 + \sqrt{2}\right)\sqrt{\lambda}\fnorm
\end{align*}
\end{proposition}
We can bound $\logdet(\bK_{\fb{t}}/\lambda + \bI)$ as follows.
Consider $\bK_s$ as a block matrix split between the $s$-th column and row,
\ie the latest arm pulled, and all other $s-1$ rows and columns.
Then using Schur's determinant identity, we have that
\begin{align*}
\det(\bK_{s}/\lambda + \bI)
&= \det(\bK_{s-1}/\lambda + \bI)
\det\left(1 + \kerfunc(\bx_s, \bx_s) - \bk_{s-1}(\bx_s)^\transp(\bK_{s-1}/\lambda + \bI)^{-1}\bk_{s-1}(\bx_s)\right)\\
&= \det(\bK_{s-1}/\lambda + \bI)\left(1 + \sigma_{s-1}^{2}(\bx_s)\right).
\end{align*}
Combining this with the fact that $\sigma_{s-1}^{2}(\bx_s) \leq \sigma_{\fb{s-1}}^{2}(\bx_s)$,
and unrolling the product into a sum using the logarithm we obtain
\begin{align*}
\logdet(\bK_{\fb{t}}/\lambda + \bI)
= \sum_{s=1}^{\fb{t}}\log(1 + \sigma_{s-1}^2(\bx_s))
\leq \sum_{s=1}^{\fb{t}}\log(1 + \sigma_{\fb{s-1}}^2(\bx_s)).
\end{align*}
We can further upper bound $\sigma_{\fb{s-1}}^2(\bx_s)
\leq 3\asigma_{\fb{s-1}}^2(\bx_s, \coldict_{\fb{s-1}})$ using \Cref{lem:bbkb-rls-accuracy},
and obtain
\begin{align*}
\beta_{\fb{t}} \leq \wt{\beta}_{\fb{t}} \eqdef
2\xi\sqrt{\sum_{s=1}^{\fb{t}}\log\left(1 + 3\asigma_{\fb{s-1}}^2(\bx_s, \coldict_{\fb{s-1}})\right) + \log\left({1}/{\delta}\right)}
+ (1 + \sqrt{2})\sqrt{\lambda}\fnorm
\end{align*}
This gives us that at all steps $t$ where $t = \fb{t}$ (\ie right after a resparsification)
\begin{align*}
\wt{\mu}_{\fb{t}}(\bx_i, \coldict_{\fb{t}}) - \wt{\beta}_{\fb{t}}\asigma_{\fb{t}}(\bx_i, \coldict_{\fb{t}})
\leq f(\bx_i) \leq \wt{\mu}_{\fb{t}}(\bx_i, \coldict_{\fb{t}}) + \wt{\beta}_{\fb{t}}\asigma_{\fb{t}}(\bx_i, \coldict_{\fb{t}})
\end{align*}
We can bound the instantaneous regret $r_t = f(\bx_{*}) - f(\bx_t)$ as follows.
First we bound
\begin{align*}
f(\bx_{*})
&\leq \amu_{\fb{t}}(\bx_*, \coldict_{\fb{t}}) + \wt{\beta}_{\fb{t}}\asigma_{\fb{t}}(\bx_{*}, \coldict_{\fb{t}})\\
&\stackrel{(a)}{\leq} \amu_{\fb{t}}(\bx_*, \coldict_{\fb{t}}) + \wt{\beta}_{\fb{t}}\wt{C}\asigma_{t-1}(\bx_{*}, \coldict_{\fb{t}})\\
&\stackrel{(b)}{\leq} \amu_{\fb{t}}(\bx_{t}, \coldict_{\fb{t}}) + \wt{\beta}_{\fb{t}}\wt{C}\asigma_{t-1}(\bx_{t}, \coldict_{\fb{t}})
\end{align*}
where $(a)$ is due to \Cref{lem:approx-batch-rls-accuracy}, and $(b)$ is due to the greediness of $\bx_t$ \wrt $\wt{\bu}_{t}$.
Similarly, we can bound
\begin{align*}
f(\bx_t)
&\geq \amu_{\fb{t}}(\bx_t, \coldict_{\fb{t}}) - \wt{\beta}_{\fb{t}}\asigma_{\fb{t}}(\bx_{t}, \coldict_{\fb{t}})\\
&\geq \amu_{\fb{t}}(\bx_t, \coldict_{\fb{t}}) - \wt{\beta}_{\fb{t}}\wt{C}\asigma_{t-1}(\bx_{t}, \coldict_{\fb{t}}).
\end{align*}
Putting it together
\begin{align}\label{eq:intermediate-regret-proof-bound}
R_T &= \sum_{t=1}^{T} r_t = \sum_{t=1}^{T}f(\bx_{*}) - f(\bx_t)\nonumber\\
 &\leq \sum_{t=1}^{T} \amu_{\fb{t}}(\bx_t, \coldict_{\fb{t}}) + \wt{\beta}_{\fb{t}}\wt{C}\asigma_{t-1}(\bx_{t}, \coldict_{\fb{t}})
- \amu_{\fb{t}}(\bx_t, \coldict_{\fb{t}}) + \wt{\beta}_{\fb{t}}\wt{C}\asigma_{t-1}(\bx_{t}, \coldict_{\fb{t}})\nonumber\\
&= 2\sum_{t=1}^{T}\wt{\beta}_{\fb{t}}\wt{C}\asigma_{t-1}(\bx_{t}, \coldict_{\fb{t}})\nonumber\\
&\leq 2\wt{C}\wt{\beta}_{\fb{T}}\sum_{t=1}^{T}\asigma_{t-1}(\bx_{t}, \coldict_{\fb{t}}).
\end{align}
We first focus on bounding $\wt{\beta}_{\fb{T}} \leq \wt{\beta}_{T}$,
starting from bounding a part of it as
\begin{align*}
\sum_{s=1}^{T}\log\left(1 + 3\asigma_{\fb{s-1}}^2(\bx_s)\right)
&\stackrel{(a)}{\leq} 3\sum_{s=1}^{T}\asigma_{\fb{s-1}}^2(\bx_s)
\stackrel{(b)}\leq 9\sum_{s=1}^{T}\sigma_{\fb{s-1}}^2(\bx_s)
\stackrel{(c)}\leq 21\wt{C}\sum_{s=1}^{T}\sigma_{s-1}^2(\bx_s).
\end{align*}
where we used $(a)$ the fact that $\log(1 + x) \leq x$,
$(b)$ \Cref{lem:bbkb-rls-accuracy}, and $(c)$ \Cref{lem:exact-batch-rls-accuracy}.
Plugging it back into the definition of $\wt{\beta}_T$ we have
\begin{align*}
\wt{\beta}_T
&=
2\xi\sqrt{\sum_{s=1}^{\fb{T}}\log\left(1 + 3\asigma_{\fb{s-1}}^2(\bx_s, \coldict_{\fb{s-1}})\right) + \log\left({1}/{\delta}\right)}
+ (1 + \sqrt{2})\sqrt{\lambda}\fnorm\\
&\leq
2\xi\sqrt{21\wt{C}\sum_{s=1}^{T}\sigma_{s-1}^2(\bx_s) + \log\left({1}/{\delta}\right)}
+ (1 + \sqrt{2})\sqrt{\lambda}\fnorm
\end{align*}
Going back to \Cref{eq:intermediate-regret-proof-bound}, the summation $\sum_{t=1}^{T}\asigma_{t-1}(\bx_{t}, \coldict_{\fb{t}})$
can be also bounded as
\begin{align*}
\sum_{t=1}^{T}\asigma_{t-1}(\bx_{t}, \coldict_{\fb{t}})
&\stackrel{(a)}{\leq} \sqrt{T}\left(\sum_{t=1}^{T}\asigma_{t-1}^2(\bx_{t}, \coldict_{\fb{t}})\right)^{1/2}
\stackrel{(b)}\leq \sqrt{T}\left(\sum_{t=1}^{T}\asigma_{\fb{t-1}}^2(\bx_{t}, \coldict_{\fb{t}})\right)^{1/2}\\
&\stackrel{(c)}\leq \sqrt{3}\sqrt{T}\left(\sum_{t=1}^{T}\sigma_{\fb{t-1}}^2(\bx_{t})\right)^{1/2}
\stackrel{(d)}\leq 3\wt{C}\sqrt{T}\left(\sum_{t=1}^{T}\sigma_{t-1}^2(\bx_{t})\right)^{1/2},
\end{align*}
using $(a)$ Cauchy-Schwarz, $(b)$ the fact that
$\asigma_{t-1}^2(\bx_{t}) \leq \asigma_{\fb{t-1}}^2(\bx_{t})$
by \Cref{lem:approx-batch-rls-accuracy},
$(c)$ \Cref{lem:bbkb-rls-accuracy},
and $(d)$ \Cref{lem:exact-batch-rls-accuracy}.
Putting it all together
\begin{align*}
R_T
&
\leq 2\wt{C}\cdot \wt{\beta}_{\fb{T}}\cdot\sum_{t=1}^{T}\asigma_{t-1}(\bx_{t}, \coldict_{\fb{t}})\\
&
\leq 2\wt{C} \cdot \wt{\beta}_{T} \cdot 3\wt{C}\sqrt{T}\left(\sum_{t=1}^{T}\sigma_{t-1}^2(\bx_{t})\right)^{1/2}\\
&
\leq 2\wt{C} \cdot 
\left(2\xi\sqrt{21\wt{C}\sum_{t=1}^{T}\sigma_{t-1}^2(\bx_t) + \log\left({1}/{\delta}\right)}
+ (1 + \sqrt{2})\sqrt{\lambda}\fnorm\right)
\cdot 3\wt{C}\sqrt{T}\left(\sum_{t=1}^{T}\sigma_{t-1}^2(\bx_{t})\right)^{1/2}\\
&
\leq 2\wt{C} \cdot 
\left(2\xi\sqrt{21\wt{C}\sum_{t=1}^{T}\sigma_{t-1}^2(\bx_t)} + 2\xi\sqrt{\log\left({1}/{\delta}\right)}
+ (1 + \sqrt{2})\sqrt{\lambda}\fnorm\right)
\cdot 3\wt{C}\sqrt{T}\left(\sum_{t=1}^{T}\sigma_{t-1}^2(\bx_{t})\right)^{1/2}\\
&
\leq 55\wt{C}^2\sqrt{T} \cdot 
\left(\xi\sqrt{\wt{C}\sum_{t=1}^{T}\sigma_{t-1}^2(\bx_t)} + \xi\sqrt{\log\left({1}/{\delta}\right)}
+ \sqrt{\lambda}\fnorm\right)
\cdot \left(\sum_{t=1}^{T}\sigma_{t-1}^2(\bx_{t})\right)^{1/2}\\
&
\leq 55\wt{C}^2\sqrt{T} \cdot 
\left(\xi\sqrt{\wt{C}\sum_{t=1}^{T}\sigma_{t-1}^2(\bx_t)} + \xi\log\left({1}/{\delta}\right)
+ \sqrt{\lambda}\fnorm\right)
\cdot \left(\sum_{t=1}^{T}\sigma_{t-1}^2(\bx_{t})\right)^{1/2}\\
&\leq 55\wt{C}^2\cdot\wt{C}\cdot\sqrt{T}
\left(\xi\sum_{t=1}^{T}\sigma_{t-1}^2(\bx_{t}) + (\xi\log(1/\delta) + \fnorm)\sqrt{\lambda\sum_{t=1}^{T}\sigma_{t-1}^2(\bx_{t})}\right).
\end{align*}

 \section{Details on experiments}\label{sec:app_exp}
In this section we report extended results on the experiments, integrating \Cref{sec:exp} with more details.

\subsection{Abalone and Cadata}
For the experiments on the Abalone and Cadata datasets, the kernel used by the algorithms in \Cref{fig:rr_ac} and \Cref{fig:t_ac} are Gaussian kernels with bandwidth as reported in \Cref{tab:sigmas}.
\begin{table}[h]
\centering
\begin{tabular}{lllll}
 &  Abalone & Cadata &   \\
 Global-\bbkb & $\sigma = 17.5$ & $\sigma = 12.5$ &   \\
GlobalLocal-\bbkb & $\sigma = 17.5$ &  $\sigma = 12.5$&   \\
\bkb & $\sigma = 17.5$ &$\sigma = 12.5$  &   \\
\gpucb & $\sigma = 5$ & $\sigma = 12.5$ &  \\
\bgpucb & $\sigma = 12.5$ &  $\sigma = 12.5$&  \\
\bts &  $\sigma = 10$&  $\sigma = 10$&  \\
\end{tabular}\caption{Badwith of Gaussian Kernel used in the Abalone and Cadata experiments}\label{tab:sigmas}
\end{table}
The other free hyperparameters are $F=1, \delta=1/T, \qbar = 2$.
For all batched algorithms the batch size is chosen using the corresponding rule presented in the original paper (\ie, \Cref{lem:posterior-ratio-local-approx,lem:approx-batch-rls-accuracy} for \bbkb's variants and \Cref{p:std.dev.ratio} for \bgpucb), except for $\bts$ which does not have one, and that we run with fixed batch size$ = 30$. Moreover, $\bts$ struggled to converge in our experiments, and to reduce overexploration we divide $\beta_t$ by a factor of $10$ only for $\bts$.
We report in \Cref{fig:aapp} and \Cref{fig:capp} further experiments on the Abalone and Cadata dataset respectively. For each plot, we report regret ratio between each algorithm and a uniform policy, with each plot corresponding to a different bandwidth in the Gaussian kernel, which is shared between all algorithms. Notice how \bbkb remains robust for a wider range of $\sigma$.

\begin{figure}[h]
\minipage{0.33\textwidth}
\includegraphics[width=\linewidth]{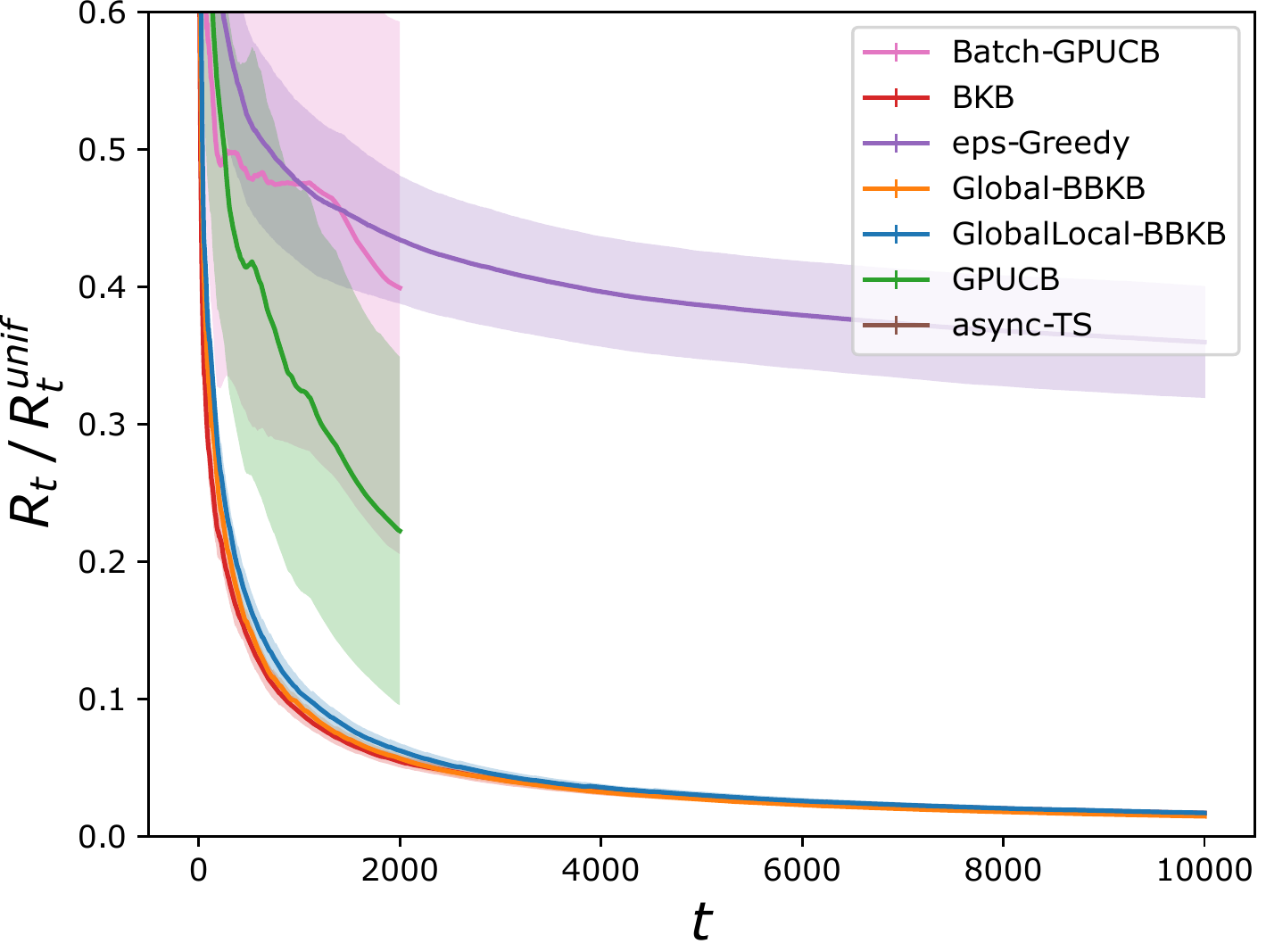}
\endminipage\hfill
\minipage{0.33\textwidth}
\includegraphics[width=\linewidth]{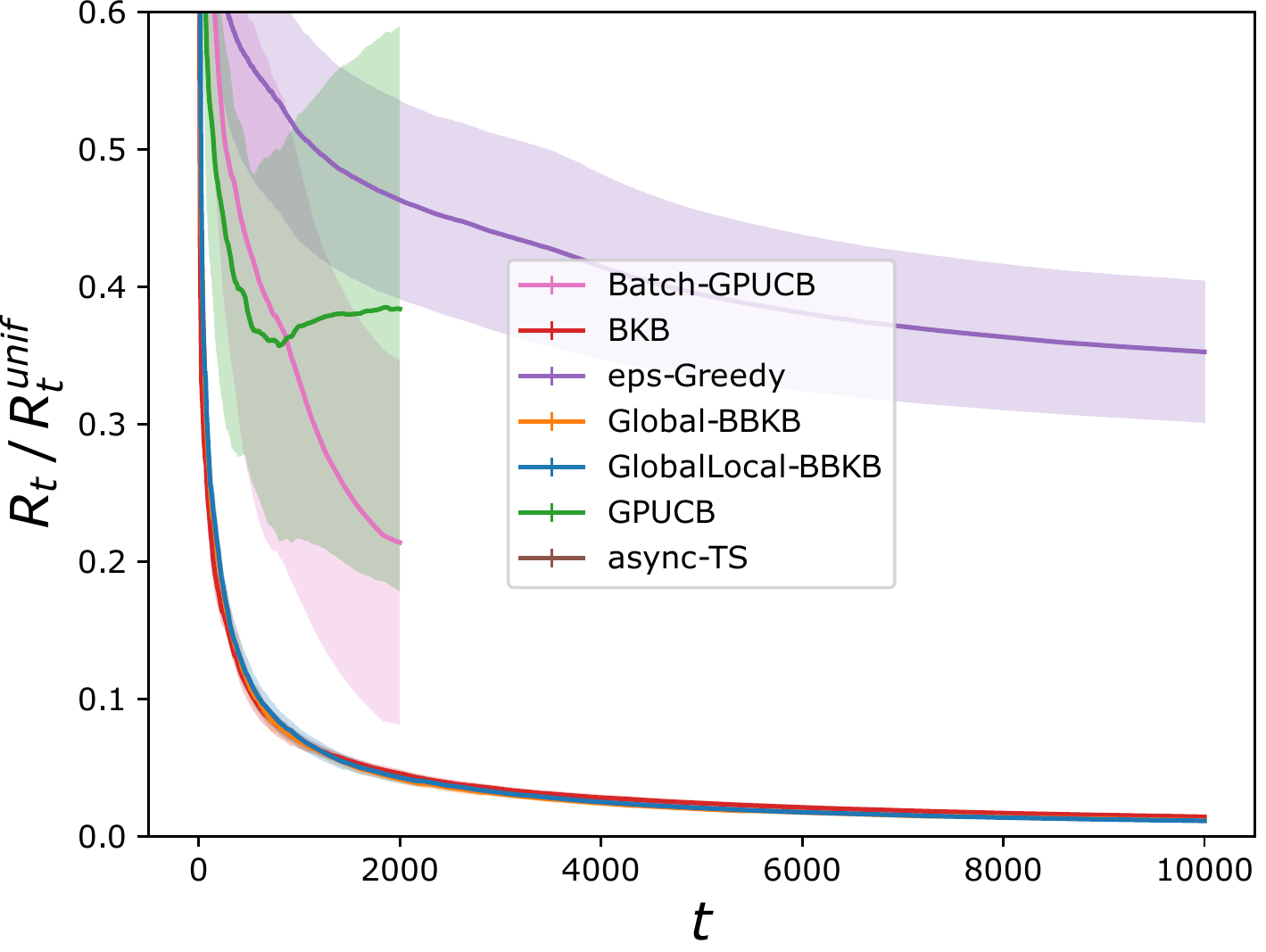}
\endminipage\hfill
\minipage{0.33\textwidth}
\includegraphics[width=\linewidth]{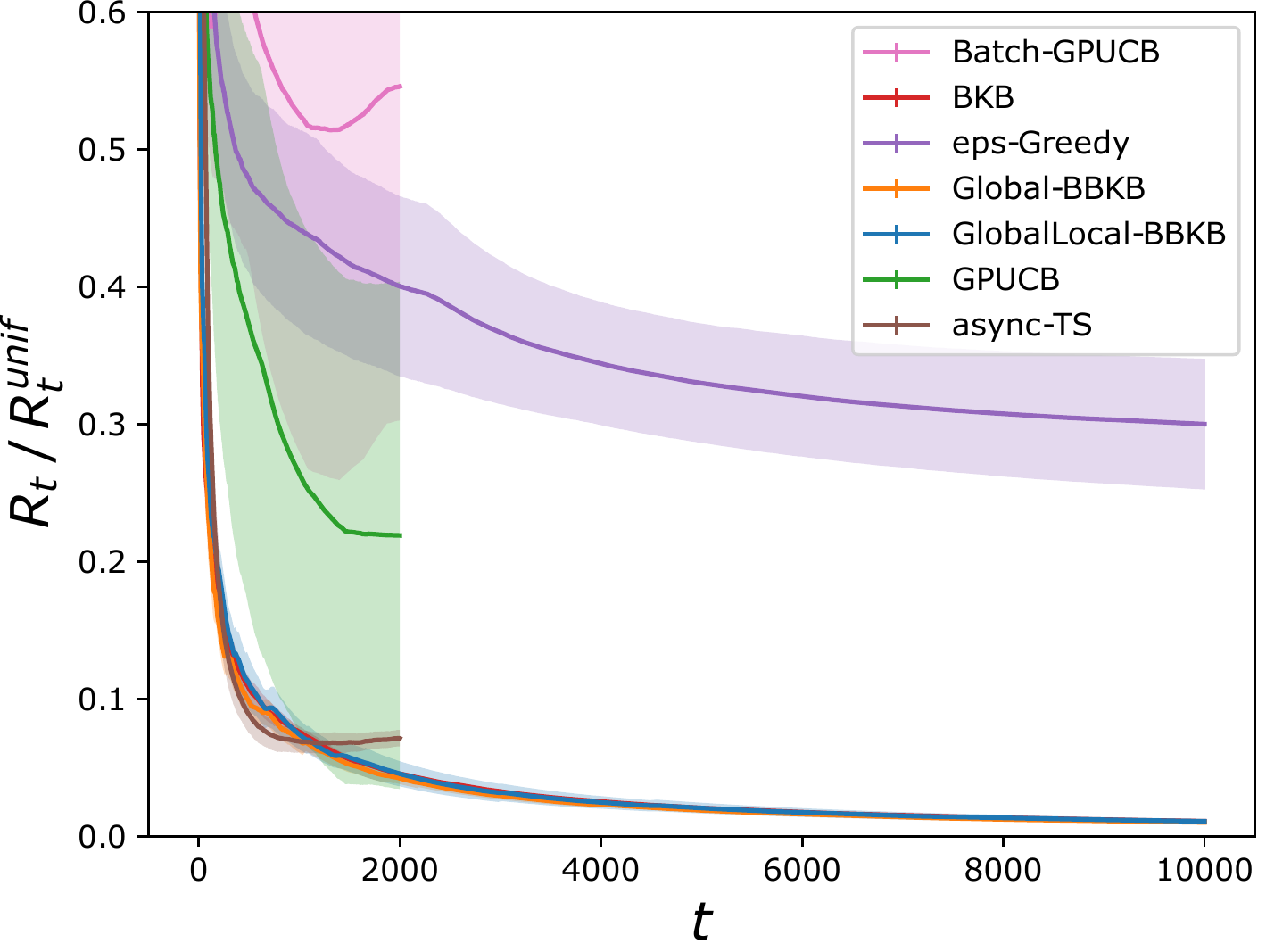}
\endminipage\hfill
\minipage{0.33\textwidth}
\includegraphics[width=\linewidth]{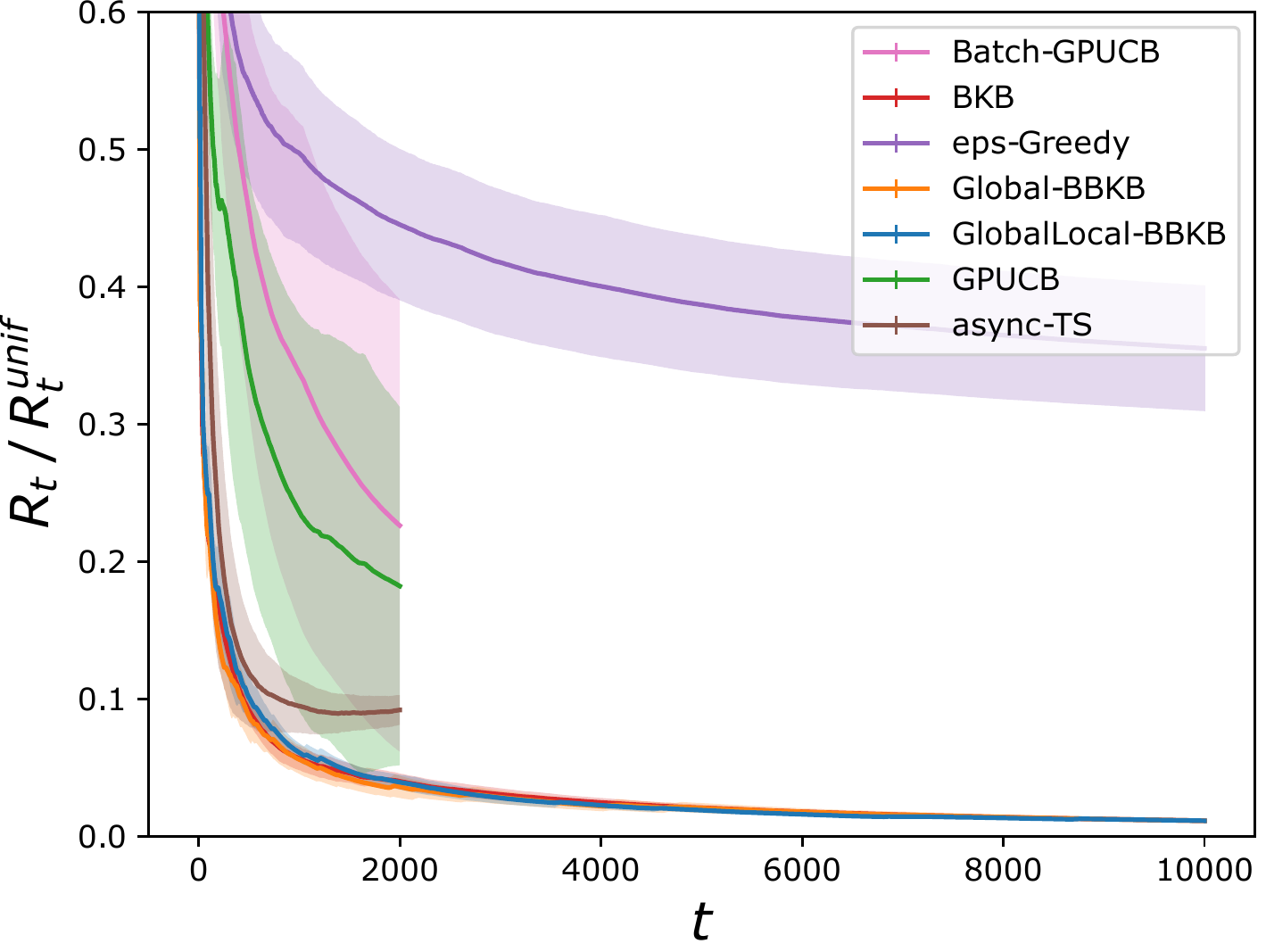}
\endminipage\hfill
\minipage{0.33\textwidth}
\includegraphics[width=\linewidth]{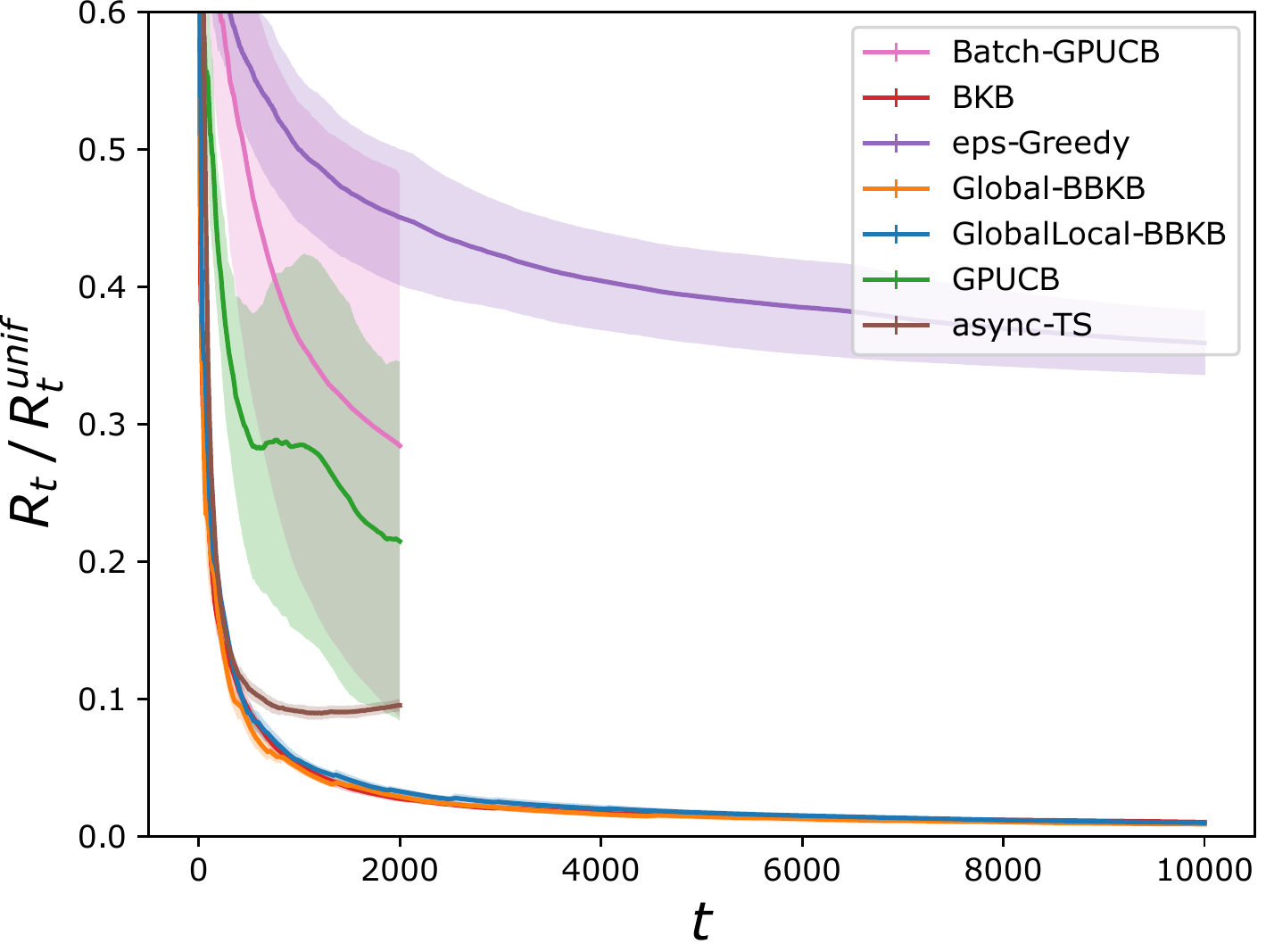}
\endminipage\hfill
\minipage{0.33\textwidth}
\includegraphics[width=\linewidth]{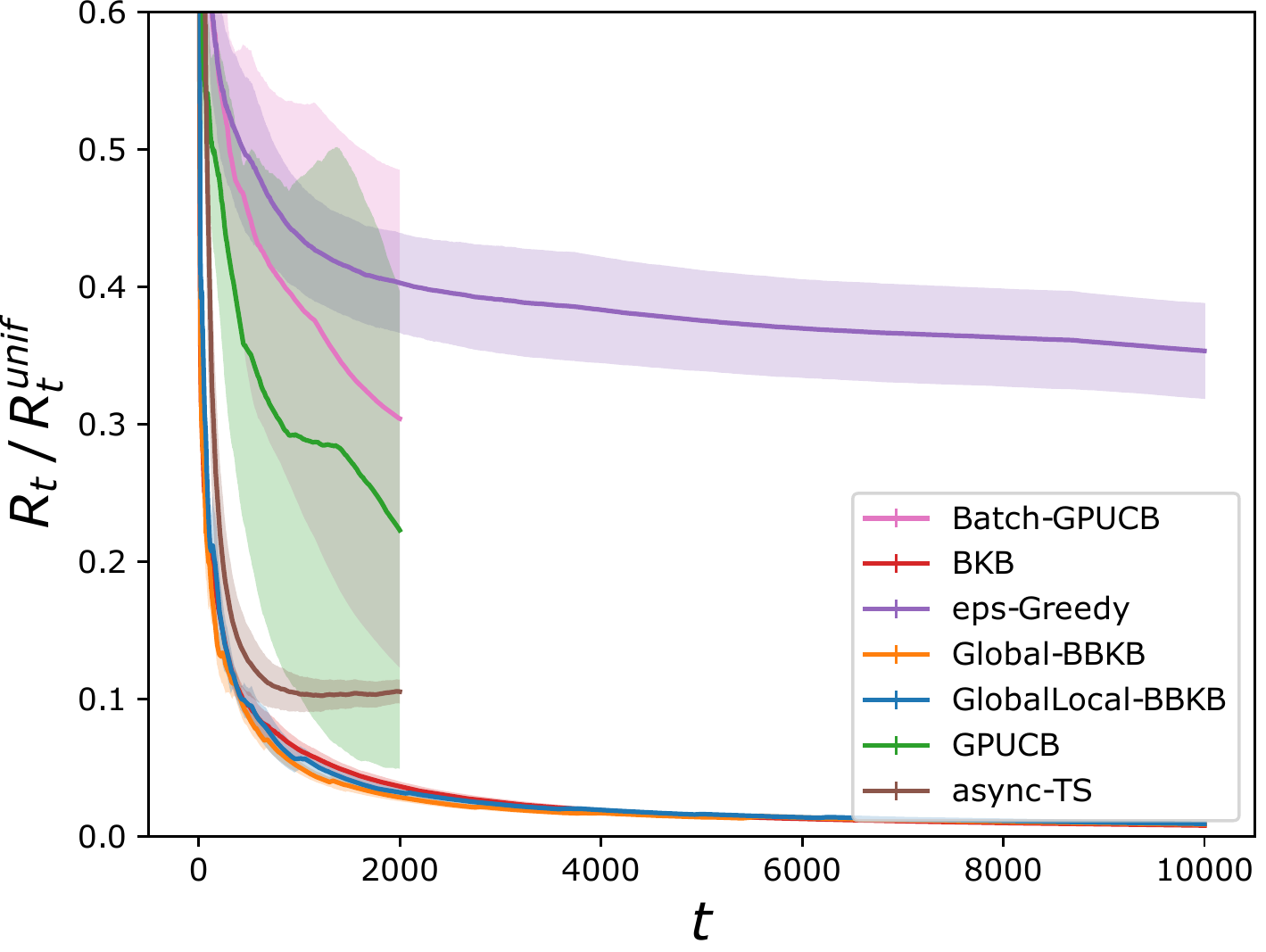}
\endminipage\hfill
\caption{Regret ratio for the Abalone dataset, with Gaussian kernel with, from left to right and from top to bottom, badwidths equal to $5, 7.5, 10, 12.5, 15, 17.5$}\label{fig:aapp}
\end{figure}

\begin{figure}[h]
\minipage{0.33\textwidth}
\includegraphics[width=\linewidth]{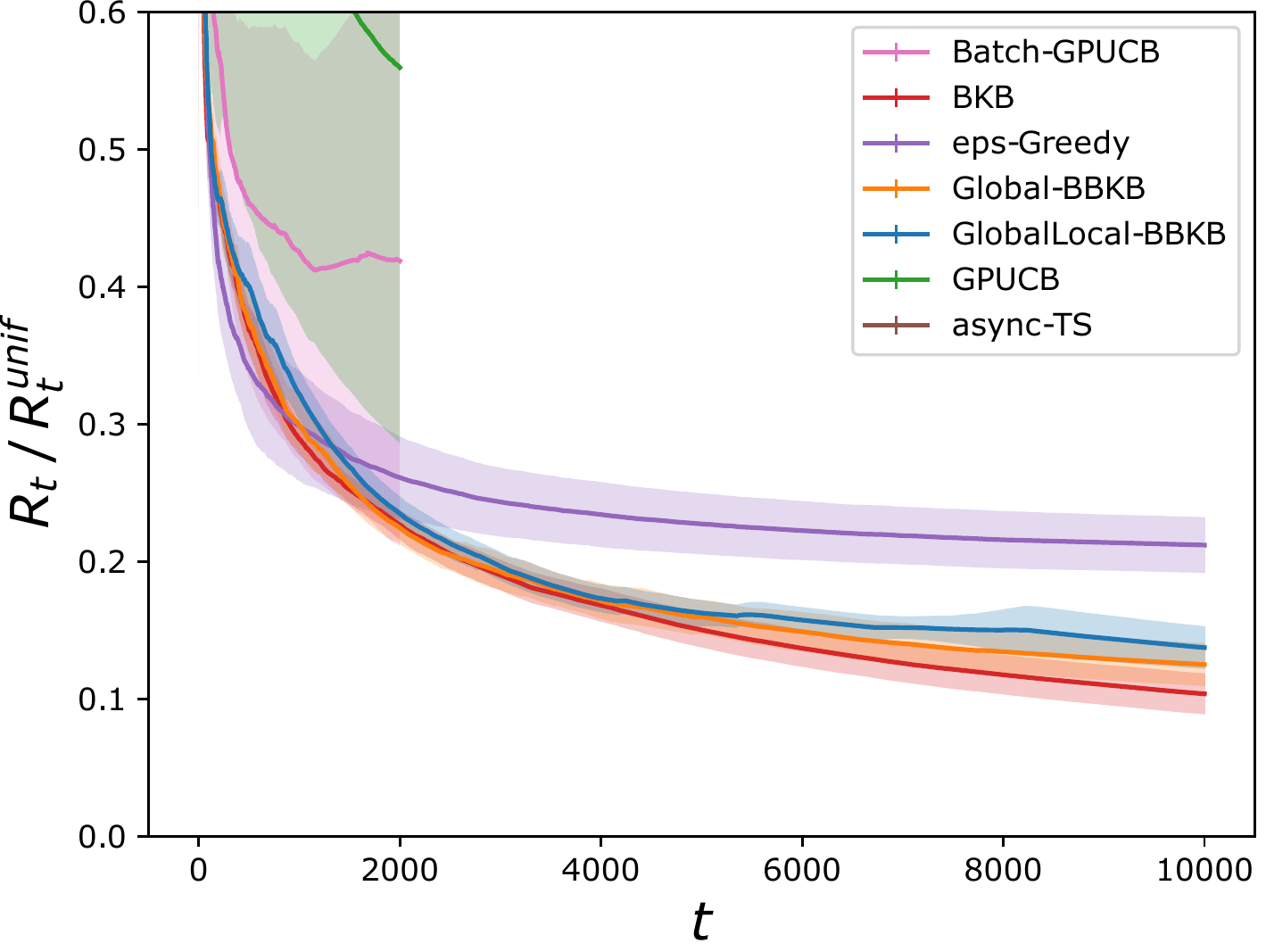}
\endminipage\hfill
\minipage{0.33\textwidth}
\includegraphics[width=\linewidth]{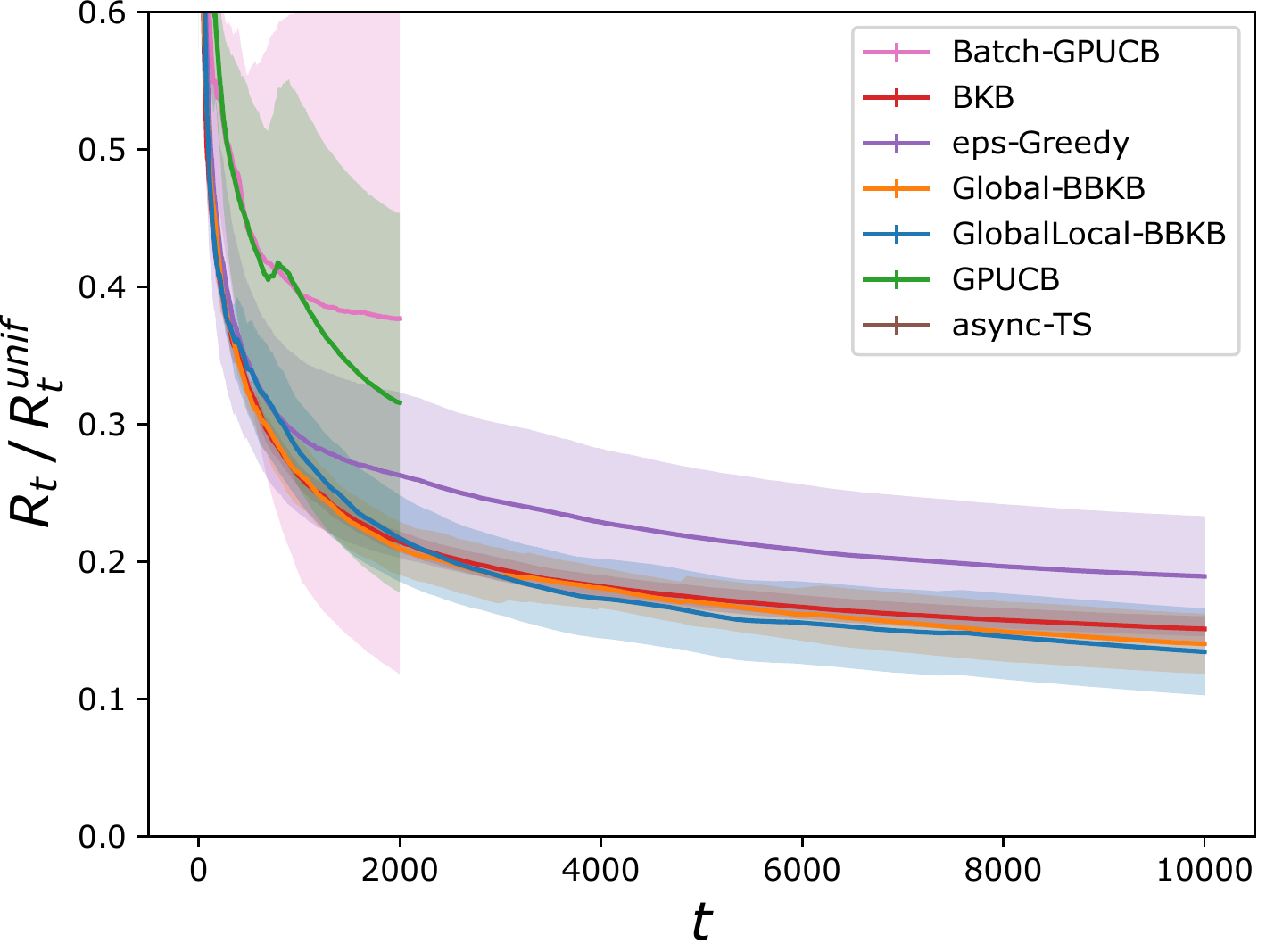}
\endminipage\hfill
\minipage{0.33\textwidth}
\includegraphics[width=\linewidth]{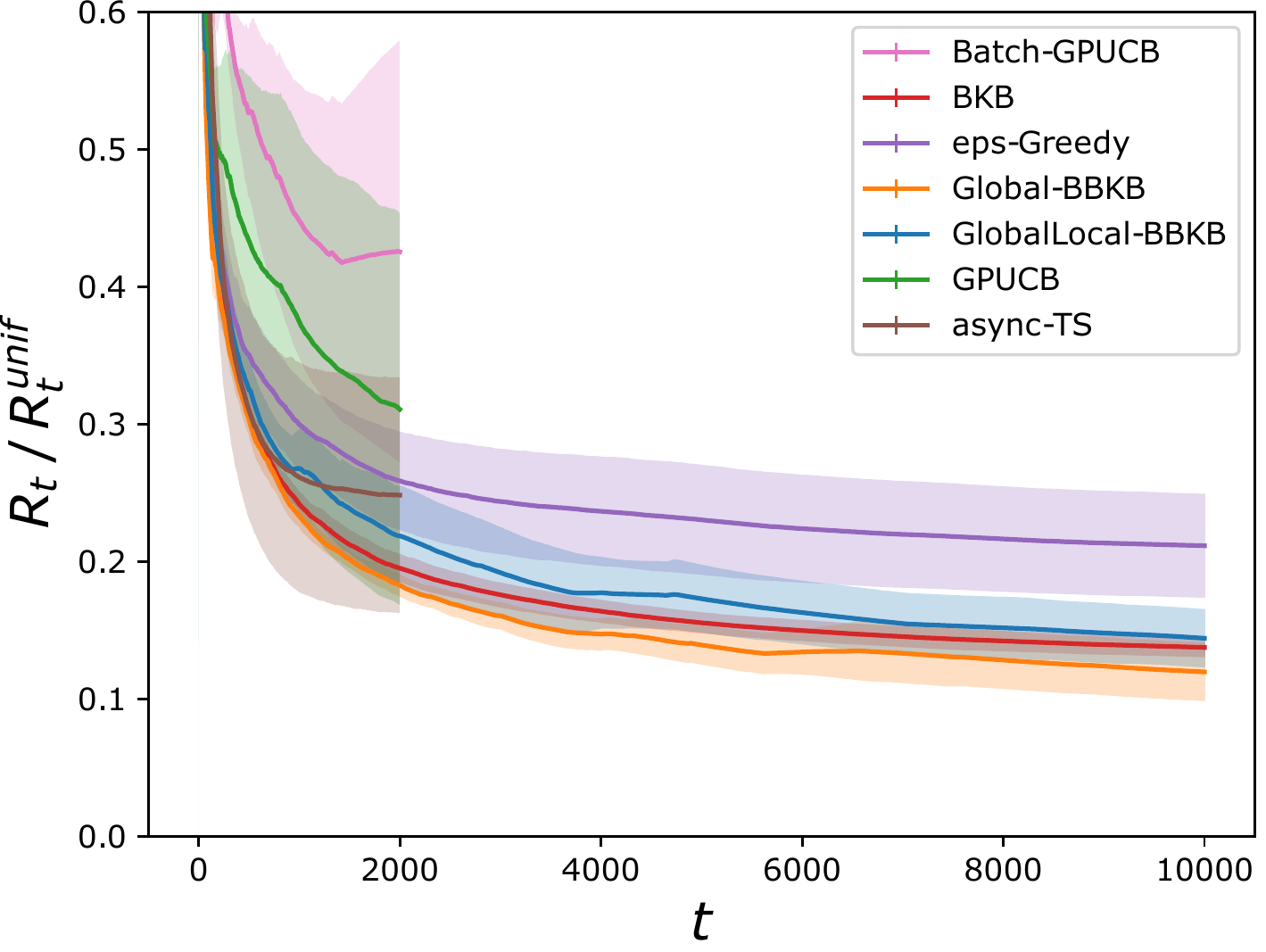}
\endminipage\hfill
\minipage{0.33\textwidth}
\includegraphics[width=\linewidth]{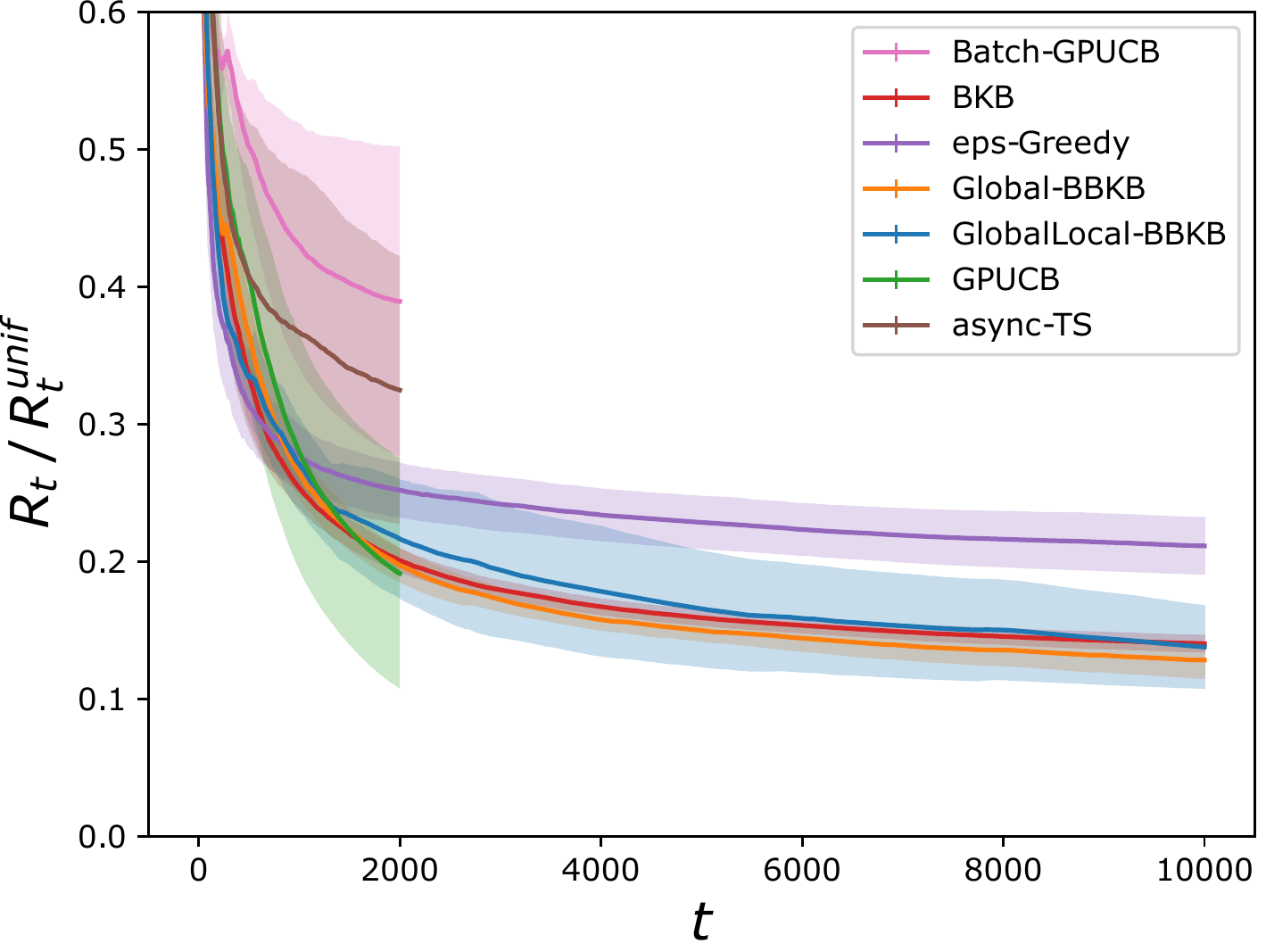}
\endminipage\hfill
\minipage{0.33\textwidth}
\includegraphics[width=\linewidth]{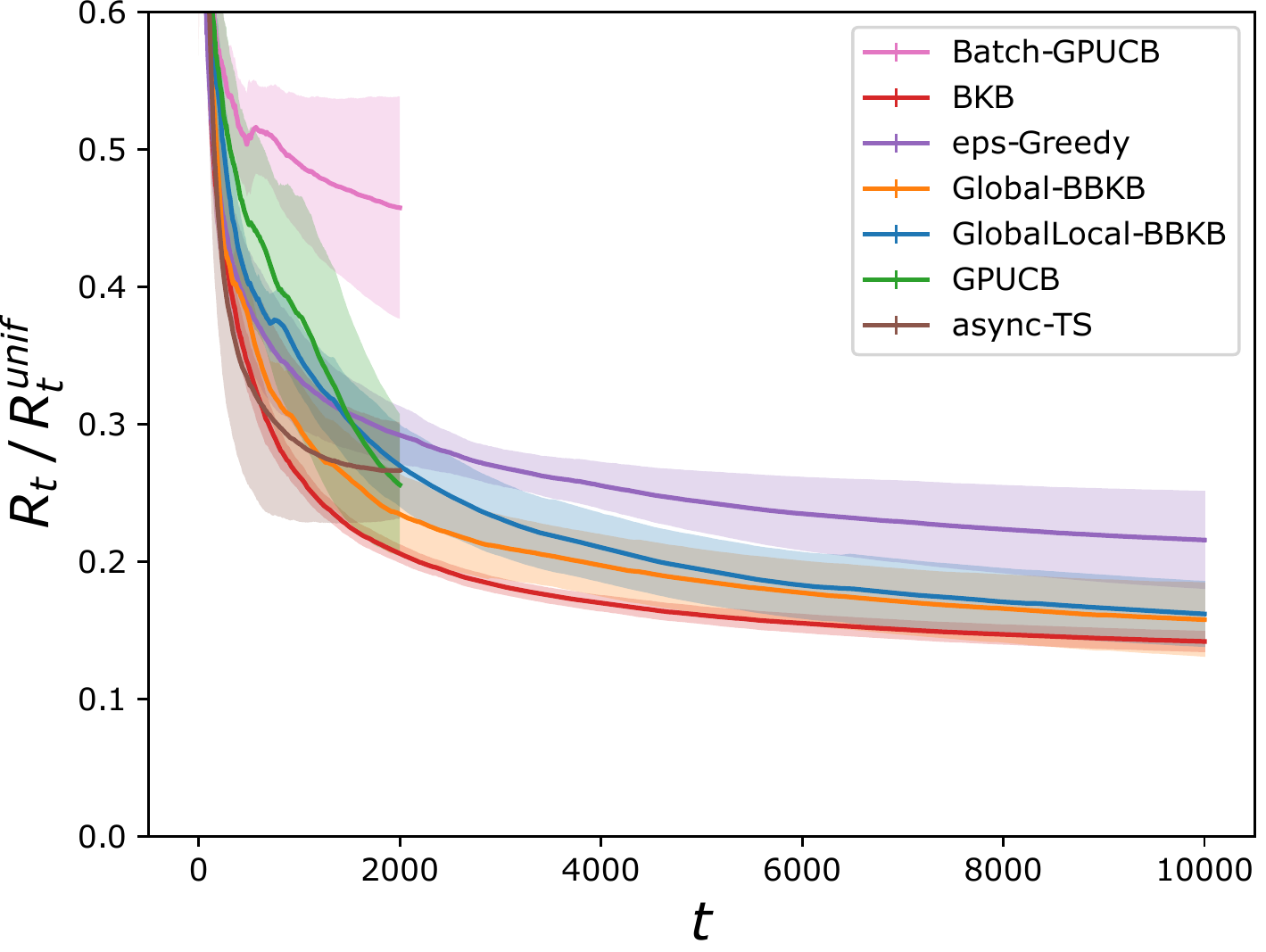}
\endminipage\hfill
\minipage{0.33\textwidth}
\includegraphics[width=\linewidth]{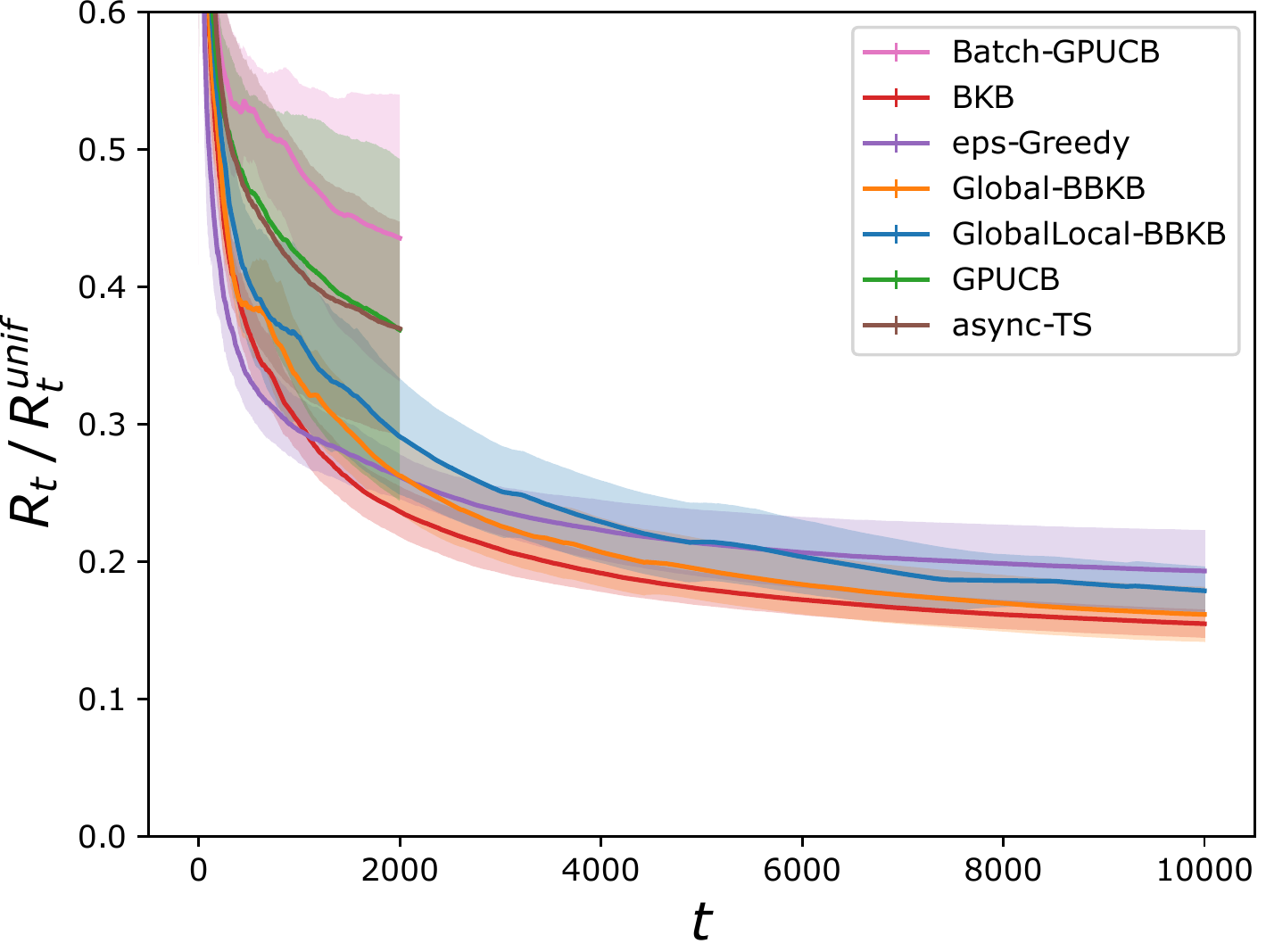}
\endminipage\hfill
\caption{Regret ratio for the Cadata dataset, with Gaussian kernel with, from left to right and from top to bottom, badwidths equal to $5, 7.5, 10, 12.5, 15, 17.5$}\label{fig:capp}
\end{figure}

\subsection{NAS-bench-101}

We use a subset of NAS-bench-101, generated using the following procedure. The whole NAS-bench-101 dataset contains architecture with 1 to 5 inner nodes, and up to 3 different kind of nodes ($3x3$ and $1x1$ convolutions, and $3x3$ max-pooling). We restrict ourselves to only architectures with exactly 4 inner nodes, and restrict the node type to only $3x3$ convolutions or max-pooling.

This kind of architectures are represented as the concatenation of a 15-dimensional vector, which is the flattened representation of the upper half of the network's adjacency matrix, and a 4-dimensional $\{0,1\}$ vector indicating whether each inner node is a convolution ($1$) or a max-pooling ($0$).

For the 15-dimensional vector, we remove the first and last feature (corresponding to connection from the input and to the output) as they are present in all architecture and result in a constant feature of 1 that is not influential for learning.
Finally, for each of these two halves of the representation, we renormalize each half separately to have at most unit norm and then concatenate the two halves. This strategy makes it so that each quantity (\ie similarity in adjacency or similarity in node type) carries roughly the same weight.

Both Global-\bbkb and GlobalLocal-\bbkb uses a Gaussian kernel with $\sigma = 125$ for the experiments with initialization, and $\sigma = 100$ for the experiment without initialization. And as for the other experiments $F=1, \delta=1/T, \qbar = 2$. The implementation of \regev is taken from \url{https://github.com/automl/nas_benchmarks}, and we leave the hyper-parameters to the default values chosen by the authors in \citet{ying2019bench} as optimal.
For completeness we report in \Cref{fig:napp} regret ratio, batch size, time and time with training of the experiments on NAS-bench-101 without initialization.

\begin{figure}[h]\centering
\minipage{0.47\textwidth}
\includegraphics[width=\linewidth]{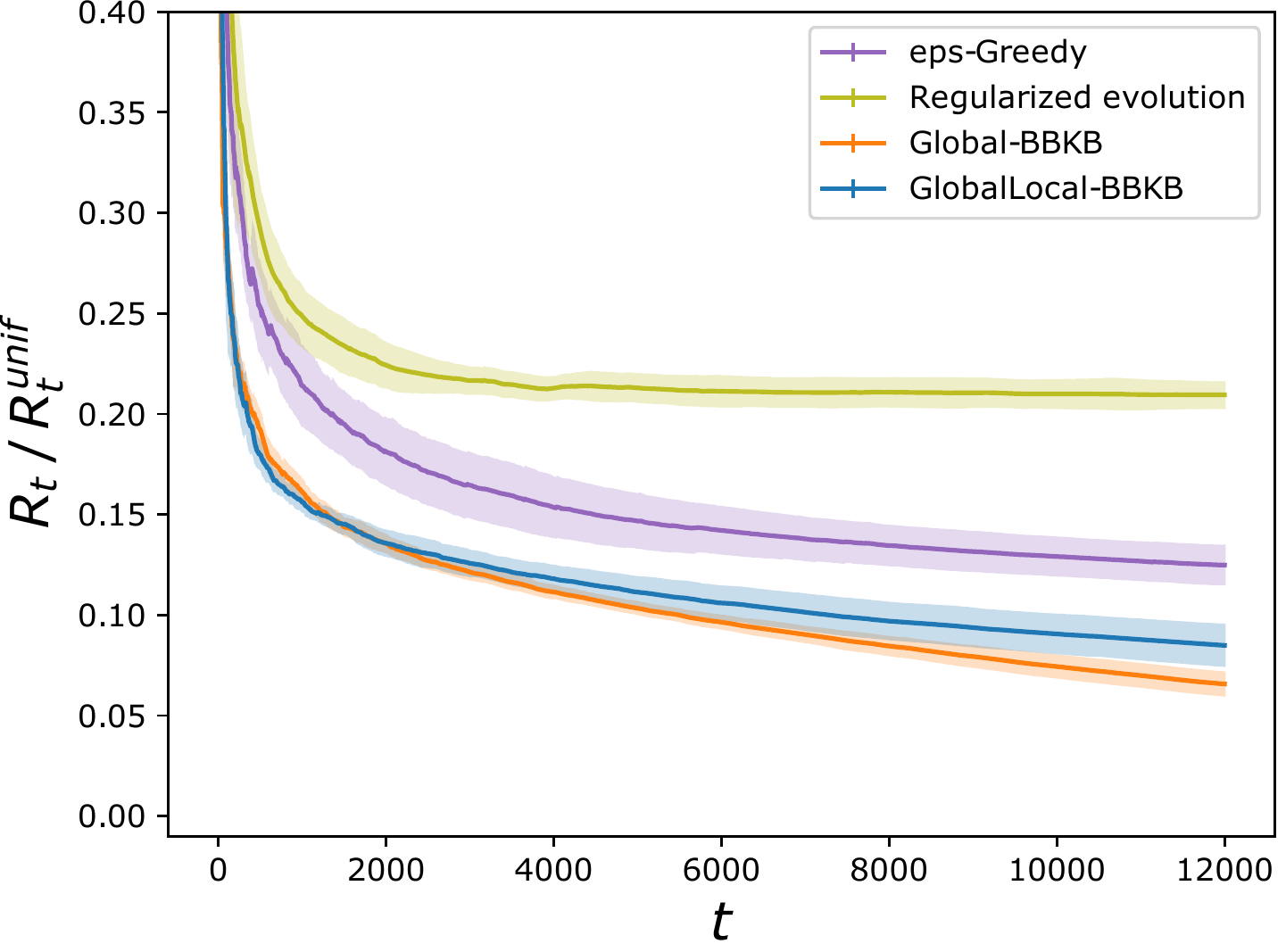}
\endminipage\hfill
\minipage{0.47\textwidth}
\includegraphics[width=\linewidth]{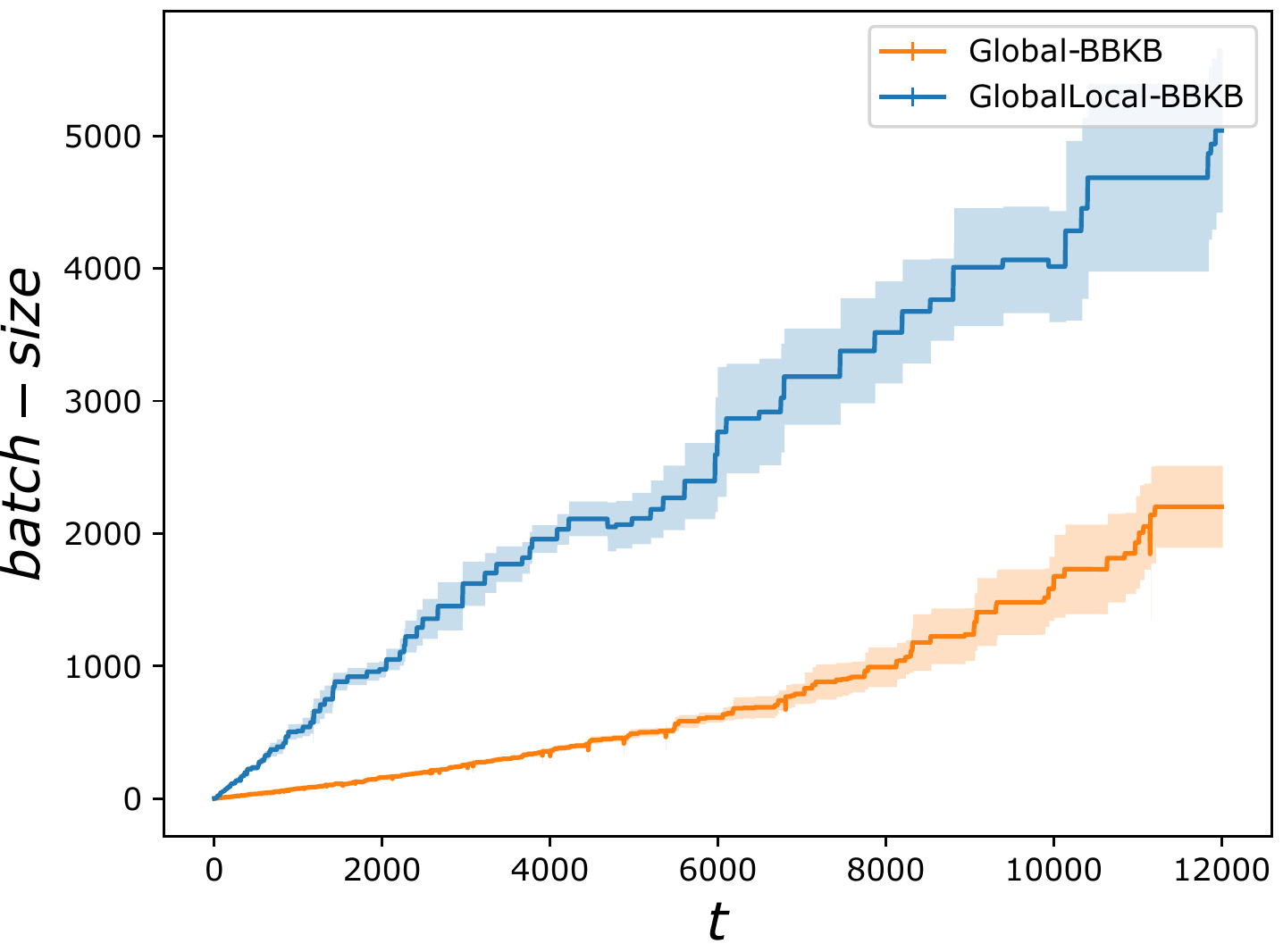}
\endminipage\hfill
\minipage{0.47\textwidth}
\includegraphics[width=\linewidth]{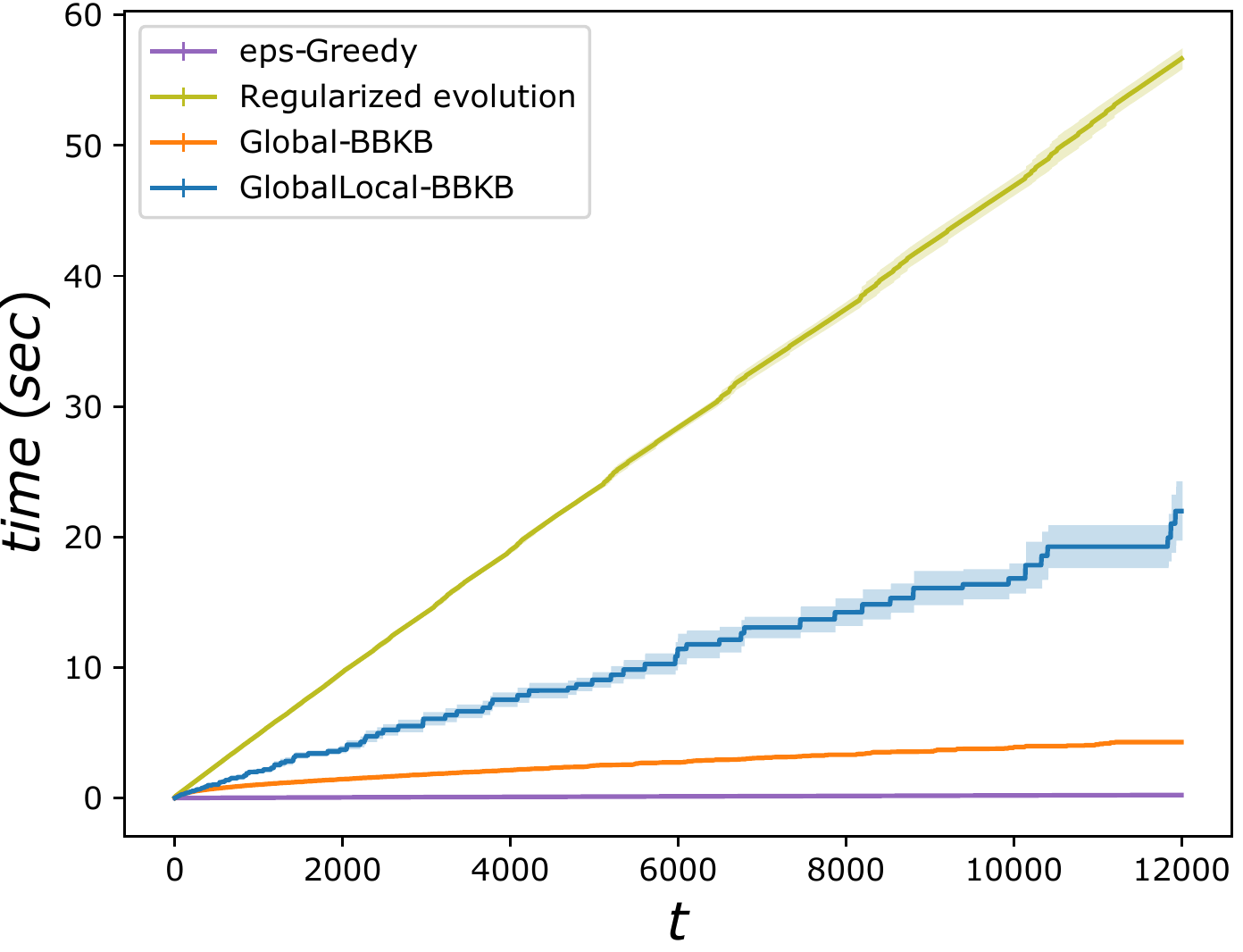}
\endminipage\hfill
\minipage{0.47\textwidth}
\includegraphics[width=\linewidth]{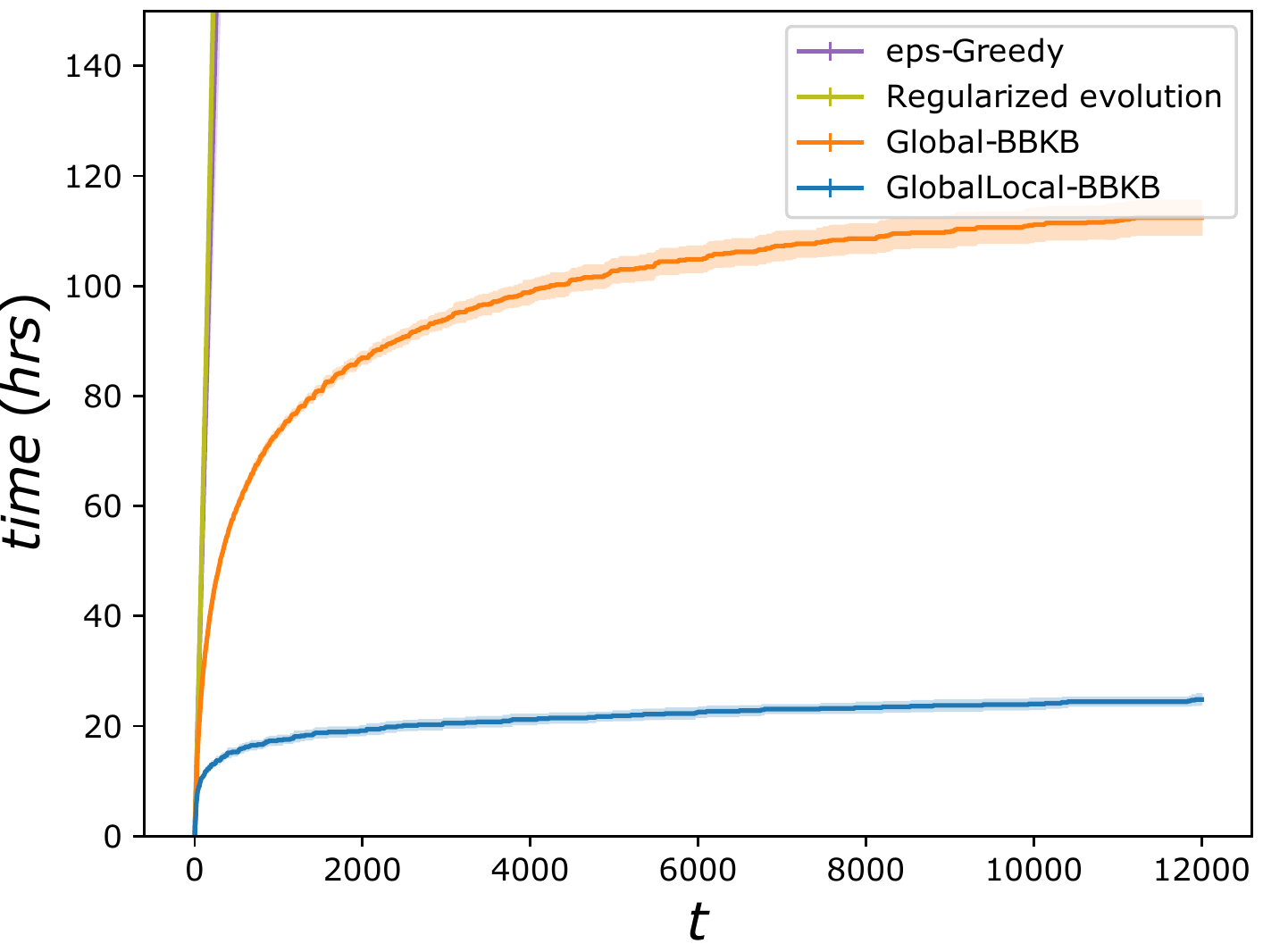}
\endminipage\hfill\caption{From left to right and from top to bottom: regret ratio, batch-size, time without experimental costs and total runtime on the NAS-bench-101 dataset, with Gaussian kernel with bandwith $100$}\label{fig:napp}
\end{figure}

\end{document}